\documentclass[10pt,fullpage,letterpaper]{article}

\usepackage[utf8]{inputenc} % allow utf-8 input
\usepackage[T1]{fontenc}    % use 8-bit T1 fonts
\usepackage{hyperref}       % hyperlinks
\usepackage{url}            % simple URL typesetting
\usepackage{booktabs}       % professional-quality tables
\usepackage{amsfonts}       % blackboard math symbols
\usepackage{nicefrac}       % compact symbols for 1/2, etc.
\usepackage{microtype}      % microtypography
\usepackage[table]{xcolor}         % colors
\usepackage{algorithm,algpseudocode}
\usepackage{enumerate}
\usepackage{amsmath, amsfonts, amssymb}
\usepackage{amsthm}
\usepackage{mathrsfs}
\usepackage{bm}
\usepackage{graphicx}
\usepackage{alltt}
\usepackage{tikz}
\usepackage{float}
\usepackage{booktabs}
\usepackage{hyperref}
\usepackage{array}
\usepackage{dsfont}
\usepackage{nicefrac}
\usepackage{comment}
\usepackage{anyfontsize}
\usepackage{longtable}
\usepackage{natbib}
\usepackage{caption}
\usepackage[symbol]{footmisc}
\allowdisplaybreaks
\usepackage{mathtools}
\usepackage{thmtools}
\usepackage{thm-restate}
\usepackage{makecell}
\usepackage{tablefootnote}

\DeclareMathOperator*{\argmin}{argmin}
\DeclareMathOperator*{\argmax}{argmax}
\DeclareMathOperator{\proj}{proj}

\newcommand{\Sp}[1]{\left(#1\right)}
\newcommand{\Mp}[1]{\left[#1\right]}
\newcommand{\Bp}[1]{\left\{#1\right\}}
\newcommand{\abs}[1]{\left|#1\right|}
\newcommand{\Norm}[1]{\left\|#1\right\|}

\newcommand{\inner}[1]{\left\langle#1\right\rangle}
\newcommand{\oV}{\overline{V}}
\newcommand{\uV}{\underline{V}}
\newcommand{\oQ}{\overline{Q}}
\newcommand{\uQ}{\underline{Q}}

\newcommand{\A}{\mathcal{A}}
\newcommand{\B}{\mathcal{ B}}
\newcommand{\D}{\mathcal{D}}
\newcommand{\E}{\mathbb{E}}
\newcommand{\F}{\mathcal{F}}
\newcommand{\G}{\mathcal{G}}
\renewcommand{\P}{\mathbb{P}}
\renewcommand{\S}{\mathcal{S}}

\renewcommand{\a}{\mathbf{a}}
\newcommand{\R}{\mathbb{R}}

\newcommand{\C}{\mathcal{C}}
\renewcommand{\r}{\mathbf{r}}

\newcommand{\ot}{\overline{\theta}}
\newcommand{\ut}{\underline{\theta}}
\newcommand{\ott}{\widetilde{\theta}}
\newcommand{\utt}{\widehat{\theta}}

\newcommand{\V}{\mathcal{V}}
\newcommand{\Q}{\mathcal{Q}}
\newcommand{\Trig}{\mathrm{Trig}}
\newcommand{\tot}{\mathrm{tot}}
\newcommand{\RNum}[1]{\uppercase\expandafter{\romannumeral #1\relax}}
\newcommand{\lin}{\mathrm{lin}}

\newcommand{\algoname}{\textbf{PReFI}}
\newcommand{\algonametab}{\textbf{PReBO}}
\newcommand{\algonamempg}{\textbf{Lin-Nash-CA}}

\newtheorem{theorem}{Theorem}
\newtheorem{lemma}{Lemma}
\newtheorem{definition}{Definition}
\newtheorem{assumption}{Assumption}
\newtheorem{proposition}{Proposition}
\newtheorem{corollary}{Corollary}

\newtheorem{example}{Example}

\newcounter{protocol}
\makeatletter
\newenvironment{protocol}[1][htb]{%
  \let\c@algorithm\c@protocol
  \renewcommand{\ALG@name}{Protocol}% Update algorithm name
  \begin{algorithm}[#1]%
  }{\end{algorithm}
}
\makeatother

\usepackage{color}
\definecolor{ForestGreen}{rgb}{0.1333,0.5451,0.1333}
\hypersetup{colorlinks,
	linkcolor=ForestGreen,
	citecolor=ForestGreen,
	urlcolor=black,
	linktocpage,
	plainpages=false}

\newcommand{\kz}[1]{\textcolor{red}{[Kaiqing: #1]}}

\date{}
\author{Qiwen Cui\footnote{University of Washington. Email: \url{qwcui@cs.washington.edu}} \and Kaiqing Zhang\footnote{University of Maryland, College Park. Email: \url{kaiqing@umd.edu}} \and Simon S. Du\footnote{University of Washington. Email: \url{ssdu@cs.washington.edu}}}

\title{
Breaking the Curse of Multiagents in a Large State Space: RL  in Markov Games with Independent  Linear Function Approximation
}

\begin{document}

\maketitle

\begin{abstract}
We propose a new model, \emph{independent linear Markov game}, for multi-agent reinforcement learning with a large state space and a large number of agents.
This is a class of Markov games with \emph{independent} linear function approximation, where each agent has its own function approximation for the state-action value functions that are {\it marginalized} by other players' policies. 
We design new algorithms for learning the Markov coarse correlated equilibria (CCE) and Markov correlated equilibria (CE) with sample complexity bounds that only scale polynomially with \emph{each agent's own function class complexity}, thus breaking the curse of multiagents. 
In contrast, existing works for Markov games with function approximation have sample complexity bounds scale with the size of the \emph{joint action space} when specialized to the canonical tabular Markov game setting, which is exponentially large in the number of agents. 
% Remarkably, the results exponentially improve all previous function approximation schemes for Markov games, which have polynomial dependence on the \emph{joint} action space (which can be exponential in the number of agents) when specialized to tabular Markov games.
% with provably efficient algorithms to 
% Importantly, when specialized to tabular Markov games, our sample complexity results exponentially improve all exist
Our algorithms rely on two key technical innovations: 
 % The key technical contribution is
 (1) utilizing policy replay to tackle {\it non-stationarity} incurred by multiple agents and the use of function approximation; (2) separating learning Markov equilibria and exploration in the Markov games, which allows us to use the full-information no-regret learning oracle instead of the stronger bandit-feedback no-regret learning oracle used in the tabular setting. 
%We further showcase these two technical ideas by applying them to two canonical settings: (1) We design a new algorithm to learn pure Markov Nash equilibria in linear Markov potential games, with applications in learning in congestion games; (2) For tabular Markov games, We provide a new algorithm with $\widetilde{O}(1/\epsilon^2)$ sample complexity to learn Markov CCE, which significantly \kz{would it look more impressive, and can really be called ``significantly improve'', if we include other parameter dependencies?} improves the state-of-the-result $\widetilde{O}(1/\epsilon^3)$~\citep{daskalakis2022complexity}\simon{hmmm the other parameters are far from optimal, so I only state $\epsilon$ here, which is optimal. I don't have a strong opinion though}\kz{yea not optimal indeed, but definitely much better than our previous one, in multiple aspects :)}, and the first algorithm to learn Markov CE that breaks the curse of multiagents.
%\simon{double check..} 
 Furthermore, we propose an iterative-best-response type algorithm that can learn pure Markov Nash  equilibria in independent linear Markov potential games, with applications in learning in congestion games.
In the tabular case, by adapting the policy replay mechanism for independent linear Markov games, we propose an algorithm with $\widetilde{O}(\epsilon^{-2})$ sample complexity to learn Markov CCE, 
which  improves the state-of-the-art result $\widetilde{O}(\epsilon^{-3})$ in \cite{daskalakis2022complexity}, where $\epsilon$ is the desired accuracy, and also significantly improves other problem parameters. Furthermore, we design  the first provably efficient  algorithm for learning Markov CE that breaks the curse of multiagents.\footnote{Accepted for presentation at the Conference on Learning Theory (COLT) 2023.}

\end{abstract}

\section{Introduction}
Decision-making under uncertainty in a multi-agent system has shown its potential to approach artificial intelligence, with superhuman performance in Go games \citep{silver2017mastering}, Poker \citep{brown2019superhuman}, and real-time strategy games \citep{vinyals2019grandmaster}, etc. All  these successes can be generally viewed as examples of 
multi-agent reinforcement learning (MARL), a generalization of single-agent reinforcement learning (RL) \citep{sutton2018reinforcement} where multiple RL agents interact and make sequential decisions in a common environment  \citep{zhang2021multi}. Despite the impressive empirical achievements of MARL, the theoretical understanding of MARL is still far from complete due to the complex interactions among agents.

One of the most prominent challenges in RL is the curse of {\it large state-action}  spaces. In real-world applications, the number of states and actions is exponentially large so that the  \emph{tabular} RL algorithms are not applicable. For example, there are $3^{361}$ potential states in Go games,  and it is impossible to enumerate all of them. In single-agent RL, plenty of works attempt to tackle this issue via function approximation so that the sample complexity only depends on the complexity of the function class, thus successfully breaking the curse of large state-action spaces  \citep{wen2017efficient,jiang2017contextual,yang2020reinforcement,du2019good,jin2020provably,weisz2021exponential,wang2020reinforcement,zanette2020learning,wang2021exponential,jin2021bellman,du2021bilinear,foster2021statistical}. 

% \kz{given bullet 1 in Section 1.1, shall we mention that ``non-stationarity'' is one key issue in this independent FA case?}
However, it is still unclear what is the proper function approximation model for multi-agent RL. The existing theoretical analyses in MARL exclusively focus  on a \emph{global}  function approximation paradigm, i.e., a function class capturing the state-joint-action value $Q_i(s,a_1,\cdots,a_m)$ where $s$ is the state and $a_i$ is the action of player $i\in[m]$ \citep{xie2020learning,huang2021towards,chen2021almost,jin2022power,chen2022unified,ni2022representation}. 
Unfortunately, these algorithms would \emph{suffer from the curse of multiagents} when specialized to tabular Markov games, one of the most canonical models in MARL. Specifically, the sample complexity depends on the number of joint actions $\prod_{i\in[m]}A_i$, where $A_i$ is the number of actions for player $i$, which is exponentially worse than the best algorithms specified to the tabular Markov game whose sample complexity only depends on $\max_{i\in[m]}A_i$ \citep{jin2021v,song2021can,mao2022improving,daskalakis2022complexity}. 

On the other hand, empirical algorithms with \emph{independent} function approximation such as Independent PPO have surprisingly good performance, where only the independent state-individual-action value function $Q_i(s,a_i)$ is modeled \citep{de2020independent,yu2021surprising}. %compared with the algorithms with global function approximation
% \kz{is this accurate? I thought it is just ``surprisingly working well'', but is it necessarily ``better''? lets double check. maybe just say it is performing very well (without comparing)}. 
This is very surprising due to the fact that the independent state-action value function $Q_i(s,a_i)$ does not reflect the change of other players' policies, a.k.a. the non-stationarity from multiple agents, which should fail to allow learning at first glance. In addition, single-agent RL with function approximation already suffers from the non-stationarity of applying function approximation \citep{baird1995residual}, making it even harder for MARL. 
This gap between theoretical and empirical research leads to the following question: 
\begin{center}
    \emph{Can we design provably efficient MARL algorithms for Markov games\\ with independent function approximation that can break the curse of multiagents?}
\end{center}

In this paper, we provide an affirmative answer to the above question. We highlight our contributions and technical novelties below. 

\subsection{Main Contributions and Technical Novelties}

\paragraph{1. Multi-player general-sum Markov games with independent linear function approximation.}
We propose independent linear Markov games, which is the first provably efficient model in MARL that allows each agent to have its own independent function approximation. 
We show that independent linear Markov games capture several important instances, namely tabular Markov games \citep{shapley1953stochastic}, linear Markov decision processes (MDP) \citep{jin2020provably}, and congestion games \citep{rosenthal1973class}. Then we provide the first provably efficient algorithm in MARL that breaks the curse of multiagents and the curse of large state and action spaces at the same time, i.e., the sample complexity only has polynomial dependence on the complexity of the independent  function class complexity. 
See Table~\ref{tab:gen_sum} for comparisons between our work and prior works.

Our algorithm design relies on two high-level technical ideas which we detail  here:
\begin{itemize}
    \item \textbf{Policy replay to tackle non-stationarity.}
Different from experience replay that incrementally adds new on-policy data to a dataset, \emph{policy replay}  maintains a policy set and completely renews the dataset at each episode by collecting fresh data using the policy set.
We propose a new policy replay mechanism for learning equilibria in independent linear Markov games, 
% \kz{did we really show it is ``insufficient'' (as some formal result)?}\kz{shall we introduce briefly  what it really means by ``policy replay'' before the first sentence? and we can then refer to later Section 4.1 for more details}, 
which allows efficient exploration while adapting to the non-stationarity 
% \kz{we seem to have not motivated/mentioned this before.} 
induced by both multiple agents and function approximation at the same time. 
%As a side product, policy replay allows learning Markov equilibria as in \citep{daskalakis2022complexity}\kz{it seems not clear why here.}, which is considerably harder than learning non-Markovian equilibria as in \citep{jin2021v,song2021can,mao2022improving}.
\item \textbf{Separating exploration and learning Markov equilibria.} States and actions in independent linear Markov games are correlated through the feature map, so we can no longer resort to adversarial bandit oracles as in algorithms for tabular Markov games~\citep{jin2021v,song2021can,mao2022improving,daskalakis2022complexity}. In particular, the adversarial contextual linear bandit oracles would be a potential substitute, while the existence of such oracles remains largely an open problem (see Section 29.4 in \cite{lattimore2020bandit}). 
To tackle this issue, we exploit the fact that under the self-play setting, other players are not adversarial but  {\it under control}, so we can sample multiple i.i.d. feedback to derive an accurate estimate instead of just a single bandit feedback.  We separate the exploration in Markov games from learning equilibria so that any no-regret algorithms with \emph{full-information feedback} are sufficient for our MARL algorithm, which is significantly weaker than the adversarial bandit oracle used in all the previous works that break the curse of multiagents in the tabular setting.
% \kz{``that break the curse of multiagents in the tabular setting''?}. 
% \simon{I think we can add some more details by moving/copying some parts from Sec4.2.}

\end{itemize}

\paragraph{2. Learning Nash equilibria in Linear Markov potential games.}
We provide an algorithm to learn Markov Nash equilibria (NE) when the underlying independent linear Markov game is also a Markov potential game. The algorithm is based on the reduction from learning NE in independent linear Markov potential games to learning the optimal policy in linear MDPs. In addition, the result directly implies a provable efficient decentralized algorithm for learning NE in congestion games, which has better sample complexity compared with the previous state-of-the-art result in \cite{cui2022provably}.

% We showcase the power of these two technical ideas and obtain improved or new results for canonical settings\kz{what do we mean by canonical here? tabular? did  we mean  the results in the following bullets?}.
% \simon{linear Markov potential games may not be a canonical setting.. we can discuss.}

\paragraph{3. Improved sample complexity for tabular multi-player general-sum Markov games.}
Aside from our contributions to Markov games with function approximation, we design an algorithm for tabular Markov games with improved sample complexity for learning Markov CCE by adapting the policy replay mechanism we proposed for the independent linear Markov games. 
% \kz{given we are referring to this term often, and it is important, we might want to emphasize it a bit more in the intro, and really explain what it is.}. 
Our sample complexity for learning Markov CCE is $\widetilde{O}(H^6S^2A_{\max}\epsilon^{-2})$, which significantly improves the prior state-of-the-art result $\widetilde{O}(H^{11}S^3A_{\max}\epsilon^{-3})$ in \cite{daskalakis2022complexity}, where $H$ is the time horizon, $S$ is the number of the states, $A_{\max}=\max_{i\in[m]}A_i$ is the maximum action space and $\epsilon$ is the desired accuracy.\footnote{We use $\widetilde{O}(\cdot)$ to omit logarithmic dependence on all the parameters.}
Furthermore, our analysis is simpler.
In addition, we provide the first provably efficient algorithm for learning Markov CE with sample complexity $\widetilde{O}(H^6S^2A_{\max}^2\epsilon^{-2})$.

\begin{table}[t]
    \centering
    \resizebox{\columnwidth}{!}{%
    \begin{tabular}{c|c|c|c|c|c}
    \hline
        Algorithms & Game & Equilibrium & Sample complexity & \makecell{Sample complexity \\ (tabular)} & BCM\\
    \hline
        \makecell{\citep{liu2021sharp}} & MG & NE/CE/CCE& $H^4S^2\prod_{i=1}^mA_i\epsilon^{-2}$ & - & $\times$\\
        \makecell{\citep{jin2021v}} & ZSMG & NE & $H^5SA_{\max}\epsilon^{-2}$ & - & -\\
        \makecell{\citep{jin2021v}} & MG & NM-CCE & $H^5SA_{\max}\epsilon^{-2}$ & - & \checkmark\\
        \makecell{\citep{jin2021v}} & MG & NM-CE & $H^5SA_{\max}^2\epsilon^{-2}$ & - & \checkmark\\
        \makecell{\citep{daskalakis2022complexity}} & MG & CCE & $H^{11}S^3A_{\max}\epsilon^{-3}$ & - & \checkmark\\
        \makecell{\citep{xie2020learning}} & ZSMG & NE & $H^4d^3\epsilon^{-2}$ & $d=SA_1A_2$ & -\\
        \makecell{\citep{chen2021almost}} & ZSMG & NE & $H^3d^2\epsilon^{-2}$ & $d=SA_1A_2$ & -\\
        \makecell{\citep{huang2021towards}} & ZSMG & NE & $H^3W^2A_{\max}\epsilon^{-2}$ & $W=SA_1A_2$ & -\\
        \makecell{\citep{jin2022power}} & ZSMG & NE & $H^2d^2\epsilon^{-2}$ & $d=SA_1A_2$ & -\\
        \makecell{\citep{chen2022unified}} & MG & NE/CE/CCE & $S^3(\prod_{i\in[m]}A_i)^2H^3\epsilon^{-2}$ & - & $\times$\\
        \makecell{\citep{ni2022representation}} & MG & NE/CE/CCE & $H^6d^4(\prod_{i=1}^mA_i)^2\log(|\Phi||\Psi|)\epsilon^{-2}$ & $d=S\prod_{i\in[m]}A_i$ & $\times$\\
        \makecell{\citep{ni2022representation}} & MG & NE/CE/CCE & $m^4H^6d^{2(L+1)^2}A_{\max}^{2(L+1)}\epsilon^{-2}$&$d=S\prod_{i\in[m]}A_i$ & $\times$\\
        \cellcolor{gray!50} Algorithm \ref{algo} (\algoname) & MG & CCE & $m^4H^{10}d_{\max}^4\epsilon^{-4}$ & $d_{\max}=SA_{\max}$ & \checkmark\\
        \cellcolor{gray!50} Algorithm \ref{algo} (\algoname) & MG & CE & $m^4H^{10}d_{\max}^4A_{\max}\epsilon^{-4}$ & $d_{\max}=SA_{\max}$ & \checkmark\\
        \cellcolor{gray!50} Algorithm \ref{algo:tabular} (\algonametab) & MG & CCE & $H^6S^2A_{\max}\epsilon^{-2}$ &- & \checkmark\\
        \cellcolor{gray!50} Algorithm \ref{algo:tabular} (\algonametab) & MG & CE & $H^6S^2A_{\max}^2\epsilon^{-2}$ &- & \checkmark\\
        \hline
    \end{tabular}
    }
    \caption{Comparison of the models and the most related sample complexity results for MARL  in Markov games.  $S$ is the number of states, $m$ is the number of players, $A_i$ is the number of actions for player $i$ with $A_{\max}=\max_{i\in[m]}A_i$, $\epsilon$ is the target accuracy, and $d$ or $W$ is the complexity of the corresponding function class. We use \textbf{MG} to denote multi-player general-sum Markov games, \textbf{ZSMG} to denote two-player zero-sum Markov games, \textbf{NE/CE/CCE} to denote Markov Nash equilibria, Markov correlated equilibria, and Markov coarse correlated equilibria, respectively. We use the prefix (NM-) to denote non-Markov equilibria. For algorithms with function approximation, we show the parameters when applied to the tabular setting and whether breaking the curse of multiagents (\textbf{BCM}) or not in the last two columns. 
  Polylog dependence on relevant parameters is omitted in the sample complexity results. 
    }
    \label{tab:gen_sum}
\end{table}

\subsection{Related Work}
\paragraph{Tabular Markov games.} Markov games, also known as stochastic games, are introduced in the seminal work \cite{shapley1953stochastic}. We first discuss works that consider bandit feedback as in our paper. \cite{bai2020provable} provide the first provably sample-efficient MARL algorithm  for two-player zero-sum Markov games, which is later improved in \cite{bai2020near}.
For multi-player general-sum Markov games, \cite{liu2021sharp} provide the first provably efficient algorithm with sample complexity depending on the size of joint action space $\prod_{i\in[m]}A_i$. \citet{jin2021v,song2021can,mao2022improving} utilize a decentralized algorithm to break the curse of multiagents. However, the output policy therein is non-Markov. Recently, \citet{daskalakis2022complexity} provide the first algorithm that can learn Markov CCE and break the curse of multiagents at the same time. Several other lines of research consider full-information feedback setting in Markov games and have attempted to prove convergence to NE/CE/CCE and/or sublinear individual regret \citep{sayin2021decentralized,zhang2022policy,cen2022faster,yang2022t,erez2022regret,ding2022independent}, and offline learning setting where a dataset is given and no further interaction with the environment is permitted \citep{cui2022offline,zhong2022pessimistic,yan2022model,xiong2022nearly,cui2022provably}.

\paragraph{Markov games with function approximation.} To tackle the curse of large state and action spaces, it is natural to incorporate existing function approximation frameworks for  single-agent RL into 
MARL algorithms. \citet{xie2020learning,chen2021almost} consider linear function approximation in two-player zero-sum Markov games, which originate from linear MDP and linear mixture MDP in single-agent RL,  respectively \citep{jin2020provably,yang2020reinforcement}. \citet{huang2021towards,jin2022power,chen2022unified,ni2022representation} consider different kinds of general function approximation, which also originate from single-agent RL literature~\citep{jiang2017contextual,du2019good,agarwal2020flambe,wang2020reinforcement,zanette2020learning,jin2021bellman,foster2021statistical,du2021bilinear}. It is notable that all of these frameworks are based on {\it global}  function approximation, which is centralized and suffers from the curse of multiagents when applied to tabular Markov games.

\paragraph{Markov potential games.} Markov potential games  incorporate Markovian state transition to potential games \citep{monderer1996potential}. Most existing results consider full-information feedback or well-explored setting and prove fast convergence of policy gradient methods to NE \citep{leonardos2021global,zhang2021gradient,ding2022independent}. \citet{song2021can} provide a best-response type algorithm that can {\it explore} in tabular Markov potential games. One important class of potential games is congestion games \citep{rosenthal1973class}. \cite{cui2022learning} give the first non-asymptotic analysis for general congestion games with bandit feedback. We refer the readers to  \citet{cui2022learning} for a more detailed background about learning in potential/congestion games. It is worth noting that for congestion games, each player is in a combinatorial bandit if other players' policies are fixed, which can be directly handled by our independent linear Markov games model, while applying potential game results lead to polynomial dependence on $A_{\max}$, which could be exponentially large in the number of  facilities in congestion games.

\paragraph{Comparison with  \cite{wang2023breaking}.} Shortly after we submitted our work to arXiv, we became aware of a concurrent and independent  work \cite{wang2023breaking}. The two works share quite a bit of results, 
% Our results share several 
% a surprising amount 
% of 
% similarity, 
e.g., the use of a similar function approximation model,  similar algorithm design and sample complexity results for learning Markov CCE in tabular Markov games,  similar discussions on the improved result by using  additional communication among agents, etc.  Here we highlight several differences in learning Markov CCE with linear function approximation. First, they 
% \simon{cleverly} 
utilize a novel second-order regret oracle and Bernstein-type concentration bounds, so that they can leverage the \emph{single-sample}  estimate instead of the \emph{batched} estimate in our algorithm, which results in better dependence on $d_{\max}$, $\epsilon$ and $H$ compared with our sample complexity. On the other hand, our result has no dependence on the number of actions, which is aligned with the single-agent linear MDP sample complexity, while theirs  has a polynomial dependence on $A_{\max}$.\footnote{In Theorem \ref{thm:CCE}, there is a $\log(A_{\max})$ factor, which can be replaced by $d_{\max}$ by using a covering argument as in adversarial linear bandits \citep{bubeck2012towards}.} This difference is because they use a uniform policy to sample at the last step while we always use the on-policy samples. In fact, neither of the sample complexity bounds is strictly better than the other one and is not directly comparable as the assumptions are not the same. Second, our algorithm can use arbitrary full-information no-regret learning oracles while their results are specialized to the Expected Follow-the-Perturbed-Leader (E-FTPL) oracle \citep{hazan2020faster}, which makes the policy class $\Pi^{\mathrm{estimate}}$ therein the linear argmax policy class. Our $\Pi^{\mathrm{estimate}}$ is induced by the full-information oracle being used, and the result is in this sense more agnostic. 
On the other hand, if we use E-FTPL, the induced $\Pi^{\mathrm{estimate}}$ has a more complicated form than the linear argmax policy class. This is because we use the optimistic estimation of the $Q$ function in our algorithm.
% \simon{whereas our $\Pi_{estimate}$ is induced by the full-information oracle used.}  
Third, our algorithm can work with agnostic model misspecification which is not considered in \cite{wang2023breaking}.  Besides the differences in linear function approximation results mentioned above and the similar algorithms and sample complexity for the tabular case, we also have results for learning NE  in Markov potential games, as well as learning Markov CE in general-sum Markov games, while they provide a policy mirror-descent-type algorithm for other function approximation settings, such as linear quadratic games and the settings with  low Eluder dimension, with a weaker version of CCE called policy-class-restricted CCE. 

\begin{table}[!]
    \centering
    \begin{tabular}{c|c|c}
    \hline
       Algorithms  & Game type & Sample complexity \\
       \hline
       \citep{leonardos2021global} & Markov potential game & $\mathrm{poly}(\kappa,m,A_{\max},S,H,\epsilon)$\\
        \citep{ding2022independent} & Markov potential game & $\mathrm{poly}(\kappa,m,A_{\max},d,H,\epsilon)$\\
        \citep{song2021can} & Markov potential game & $m^2H^4SA_{\max}\epsilon^{-3}$\\
        \citep{cui2022learning} (Centralized) & Congestion game & $m^2F\epsilon^{-2}$\\
        \citep{cui2022learning} (Decentralized) & Congestion game & $m^{12}F^{6}\epsilon^{-6}$\\
        \cellcolor{gray!50} Algorithm \ref{algo:mpg} (\algonamempg) & Linear Markov potential game & $m^2H^7d_{\max}^4\epsilon^{-3}$\\
        \cellcolor{gray!50} Algorithm \ref{algo:mpg} (\algonamempg) & Congestion game & $m^2F^2\epsilon^{-3}$\\
    \hline
    \end{tabular}
    \caption{Comparison of algorithms for learning NE in Markov potential games. $\kappa$ is the distribution mismatch coefficient, $S$ is the number of states, $m$ is the number of players, $A_i$ is the number of actions for player $i$, $A_{\max}=\max_{i\in[m]}A_i$, $F$ is the number of facilities in congestion games, $\epsilon$ is accuracy, and $d_{\max}$ is the complexity of the function class.  For \citet{leonardos2021global,ding2022independent}, $\kappa$ can be arbitrarily large as no exploration is considered.}
    \label{tab:potential}
\end{table}

\paragraph{Notation.} For a finite set $X$, we use $\Delta(X)$ to denote the space of distributions over $X$. For $n\in\mathbb{N}^{+}$, we use $[n]$ to denote $\{1,2,\cdots,n\}$. We use $\Norm{\cdot}$ to denote the Euclidean norm $\Norm{\cdot}_2$ and $\inner{\cdot,\cdot}$ to denote the Euclidean inner product. We define $\proj_{[a,b]}(x):=\min\{\max\{x,a\},b\}$ and $x\vee y:=\max\{x,y\}$. An arbitrary tie-breaking rule can be used for determining $\argmax_x f(x)$.

\section{Preliminaries}

Multi-player general-sum Markov games are defined by the tuple $(\S,\{A_i\}_{i=1}^m,H,\P,\{r_i\}_{i=1}^m)$, where $\S$ is the state space with $|\S|=S$, $m$ is the number of the players, $\A_i$ is the action space for player $i$ with $|\A_i|=A_i$, $H$ is the length of the horizon, $\P=\{\P_h\}_{h\in[H]}$ is the collection of the transition kernels such that $\P_h(\cdot\mid s,\a)$ gives the distribution of  the next state given the current state $s$ and joint action $\a=(a_1,a_2,\cdots,a_m)$ at step $h$, and $r_i=\{r_{h,i}\}_{h\in[H]}$ is the collection of random reward functions for each player such that $r_{h,i}(s,\a)\in[0,1]$ is the random reward with mean $R_{h,i}(s,\a)$ for player $i$ given the current state $s$ and the joint action $\a$ at step $h$. %\footnote{We assume deterministic reward as learning the transition is the main difficulty.} 
We use $\A=\A_1\times\A_2\times\cdots\times\A_m$ to denote the joint action space, $\r_h=(r_{h,1},r_{h,2},\cdots,r_{h,m})$ to denote the joint reward profile at step $h$, and $A_{\max}=\max_{i\in[m]}A_i$. In the rest of the paper, we will simplify ``multi-player general-sum Markov games'' to ``Markov games''  when it is clear from the context.

Markov games will start at a fixed initial state $s_1$ 
% \kz{as above, maybe we need to introduce this ``fixed'' $s_1$ before introducing MPG above, since in Ioannis, Lina, and my papers, it is defined for ``all'' $s$.} 
for each episode.\footnote{It is straightforward to generalize to stochastic initial state $s_1\sim p_1(\cdot)$ by adding a dummy state $s_0$ instead, which will transition to $s_1\sim p_1(\cdot)$ no matter what action is chosen.} At each step $h\in[H]$, each player $i$ will observe the current state $s_h$ and choose some action $a_{h,i}$ simultaneously, and receive their own reward realization $\widetilde{r}_{h,i}\sim r_{h,i}(s_h,\a_h)$ where $\a_h=(a_{h,1},a_{h,2},\cdots,a_{h,m})$. Then the state will transition according to $s_{h+1}\sim \P_h(\cdot\mid s_h,\a_h)$. The game will terminate when state $s_{H+1}$ is reached and the goal of each player is to maximize their own expected   total reward $\E\Mp{\sum_{h=1}^H\widetilde{r}_{h,i}}$. We consider the bandit-feedback setting where only the reward for the chosen action is revealed, and there is no simulator and thus  exploration is necessary.
% \kz{maybe also mention that we here also mean that ``there is no simulator and we need to do exploration''?}

\paragraph{Policy.}
A Markov joint policy is denoted by $\pi=\{\pi_h\}_{h=1}^H$ where each $\pi_h:\S\rightarrow\Delta(\A)$ is the joint policy at step $h$. We say that a Markov joint policy is a Markov product policy if there are policies $\{\pi_i\}_{i=1}^m$ such that $\pi_h(\a\mid s)=\prod_{i=1}^m\pi_{h,i}(a_i\mid s)$ for each $h\in[H]$, where $\pi_i=\{\pi_{h,i}\}_{h=1}^H$ is the collection of Markov policies $\pi_{h,i}:\S\rightarrow\Delta(\A_i)$ for player $i$. In other words, a Markov product policy means that the policies of each player are not correlated. For a Markov joint policy $\pi$, we use $\pi_{-i}$ to denote the Markov joint policy for all the players except player $i$. We will simplify the terminology by using ``policy'' instead of ``Markov joint policy'' when it is clear from the context as we will only focus on Markov policies.
% \kz{i guess we can say we will only be focusing on ``Markov''  policy in the paper?}

\paragraph{Value function.}
For a policy $\pi$, it can induce a random trajectory $(s_1,\a_1,\r_1,s_2,\cdots,s_H,\a_H,\r_H,s_{H+1})$ such that $\a_h\sim\pi_h(\cdot\mid s_h)$, $\r_h\sim\r_h(s_h,\a_h)$, and $s_{h+1}\sim\P_h(\cdot\mid s_h,\a_h)$ for all $h\in[H]$. For simplicity, we will denote $\E_\pi[\cdot]=\E_{(s_1,\a_1,\r_1,s_2,\cdots,s_H,\a_H,\r_H,s_{H+1})\sim\pi}[\cdot]$. We define the state value function under policy $\pi$ for each player $i\in[m]$ to be
$$V_{h,i}^\pi(s_h):=\E_\pi\Mp{\sum_{t=h}^Hr_{t,i}(s_t,\a_t)\ \middle| \ s_h},\forall s_h\in\S,$$ 
which is the expected total reward for player $i$ if all the players are following policy $\pi$ starting from state $s_h$ at step $h$. 

\paragraph{Best response and strategy modification.}
Suppose all the players except player $i$ are playing according to a fixed policy $\pi_{-i}$, then the best response of player $i$ is the policy that can achieve the highest total reward for player $i$. Concretely, $\pi_i$ is the best response to $\pi_{-i}$ if $\pi_i=\argmax_{\pi'_i\in\Pi_i}V_{1,i}^{\pi'_i,\pi_{-i}}(s_1)$, where $\Pi_i$ consists of all the possible policies for player $i$. We will use $V_{h,i}^{\dagger,\pi_{-i}}(s)$ to denote the best-response value $\max_{\pi'_i\in\Pi_i}V_{h,i}^{\pi'_i,\pi_{-i}}(s)$ for all $h\in[H]$, $i\in[m]$ and $s\in\S$ and $\E_{\dagger,\pi_{-i}}[\cdot]$ to be the expectation over the corresponding best-response policy. Note that if all the other players are playing a fixed policy, then player $i$ is in an MDP and the best response is the corresponding optimal policy, which can always be deterministic and achieve the optimal value $\max_{\pi'_i\in\Pi_i}V_{h,i}^{\pi'_i,\pi_{-i}}(s)$ for all $h\in[H]$ and $s\in\S$ simultaneously.

A strategy modification $\psi_i=\{\psi_{h,i}\}_{h=1}^H$ for player $i$ is a collection of maps $\psi_{h,i}:\S\times\A_i\rightarrow\A_i$, which will map the action chosen at any state to another action.\footnote{We only consider deterministic strategy modification as it is known that the optimal strategy modification can always be deterministic \citep{jin2021v}.} For a Markov joint policy $\pi$, we use $\psi_i\diamond\pi$ to denote the modified Markov joint policy such that 
$$(\psi_i\diamond\pi)_h(\a\mid s)=\sum_{\a':\psi_{h,i}(a'_{i}\mid s)=a_i,\a'_{-i}=\a_{-i}}\pi_h(\a'\mid s). $$
In words, if the policy $\pi_h$ assigns action $a_i$ to player $i$ at state $s$, it will be modified to action $\psi_{h,i}(a_i\mid s)$. We use $\Psi_i$ to denote all the possible strategy modifications for player $i$. As $\Psi_i$ contains all the constant modifications, we have
$$\max_{\psi_i\in\Psi_i}V_{1,i}^{\psi_i\diamond\pi}(s_1)\geq \max_{\pi'_i}V_{1,i}^{\pi'_i,\pi_{-i}}(s_1)=V^{\dagger,\pi_{-i}}_{1,i}(s_1), $$
which means that strategy modification is always stronger than the best response. 

\paragraph{Notions of equilibria.}

A Markov Nash equilibrium is a Markov product policy where no player can increase their total reward by changing their own policy. 
\begin{definition}
(Markov Nash equilibrium)
A Markov product policy $\pi$ is an $\epsilon$-approximate Nash equilibrium if 
$$\mathrm{NashGap}(\pi):=\max_{i\in[m]}\Sp{V_{1,i}^{\dagger,\pi_{-i}}(s_1)-V_{1,i}^\pi(s_1)}\leq\epsilon. $$
\end{definition}
In general, it is intractable to compute Nash equilibrium even in normal-form general-sum games, which are Markov games with $H=1$ and $S=1$ \citep{daskalakis2009complexity,chen2009settling}. In this paper, we will focus on the following two relaxed equilibrium notions, which allow {computationally} efficient learning. 

\begin{definition}
(Markov Coarse Correlated Equilibrium)
A Markov joint policy $\pi$ is a Markov coarse correlated equilibrium if
$$\mathrm{CCEGap}(\pi):=\max_{i\in[m]}\Sp{V_{1,i}^{\dagger,\pi_{-i}}(s_1)-V_{1,i}^\pi(s_1)}\leq\epsilon.$$
\end{definition}

\begin{definition}
(Markov Correlated Equilibrium)
A Markov joint policy $\pi$ is a Markov correlated equilibrium if
$$\mathrm{CEGap}(\pi):=\max_{i\in[m]}\Sp{\max_{\psi_i\in\Psi_i}V_{1,i}^{\psi_i\diamond\pi}(s_1)-V_{1,i}^\pi(s_1)}\leq\epsilon.$$
\end{definition}

It is known that every Markov NE is a Markov CE and every Markov CE is a Markov CCE, and in two-player zero-sum Markov games, these three notions are equivalent. In this work, we will focus on Markov equilibria, which are more refined  compared with non-Markov equilibria considered in \cite{jin2021v,song2021can,mao2022improving}. For a detailed discussion regarding the difference, we refer the readers to \cite{daskalakis2022complexity}. 

Two important special cases of Markov games are two-player zero-sum Markov games and Markov potential games, which have computationally efficient algorithms for learning Markov NE. Two-player zero-sum Markov games are Markov games with the number of players $m=2$ and reward function satisfying $r_{h,1}(s,\a)+r_{h,2}(s,\a)=0$ for all $(s,\a)\in\S\times\A$ and $h\in[H]$. Markov potential games are Markov games with a potential function $\Phi:\Pi\rightarrow[0,\Phi_{\max}]$, where $\Pi$ is the set of all possible Markov product policies $\pi_1\times\pi_2\cdots\times\pi_m$, such that for any player $i\in[m]$, two policies $\pi_i,\pi'_i$ of player $i$ and policy $\pi_{-i}$ for the other players, we have
 $$V_{1,i}^{\pi_i,\pi_{-i}}(s_1)-V_{1,i}^{\pi'_i,\pi_{-i}}(s_1)=\Phi(\pi_i,\pi_{-i})-\Phi(\pi'_i,\pi_{-i}). $$
 Immediately, we have $\Phi_{\max}\leq mH$ by varying $\pi_i$ for each player $i$ for one time. One special case of Markov potential games is Markov cooperative games, where all the players share the same reward function.

\section{MARL with Independent Linear Function Approximation}\label{sec:linear MG}
In this section, we will introduce the independent linear Markov game  model and demonstrate the advantage of this model over existing Markov games with function approximation. Intuitively, independent linear Markov games assume that if other players are following some fixed Markov product policies, then player $i$ is approximately in a linear MDP~\citep{jin2020provably}. This is fundamentally different from previous global function approximation formulations, which basically assume that the Markov game is a big linear MDP where the action is the joint action $\a=(a_1,a_2,\cdots,a_m)$.

\paragraph{Feature and independent linear function class. } 
For each player $i$, they have access to their own feature map $\phi_i:\S\times\A_i\rightarrow \R^{d_i}$ and we assume that 
$$\sup_{(s,a_i)\in\S\times\A_i}\Norm{\phi_i(s,a_i)}_2\leq1. $$
For player $i$, given parameters  $\theta_i=(\theta_{1,i},\cdots,\theta_{H,i})$, the corresponding linear state-action value function for player $i$ would be $f_i^{\theta}=(f_{1,i}^{\theta_{1,i}},f_{2,i}^{\theta_{2,i}},\cdots,f_{H,i}^{\theta_{H,i}})$ where $f_{h,i}^{\theta_{h,i}}(s,a_i)=\inner{\phi_i(s,a_i),\theta_{h,i}}$ for all $(s,a_i)\in\S\times\A_i$. We consider the following linear state-action value function class for player $i$:
$$\Q_i^{\lin}=\Bp{f_i^{\theta_i}\mid \Norm{\theta_{h,i}}_2\leq H\sqrt{d},\forall h\in[H]}. $$
We also define the state value function class 
$$\V=\Bp{(V_{1},\cdots,V_{H+1})\mid V_{h}(s)\in[0,H+1-h],\forall h\in[H+1],s\in\S}.$$
% \kz{can $h=H+1$? then what is $V_0$?}

Given the state value function $V\in\V$   and other players' policies $\pi_{-i}$, we can define the independent state-action value function for all $h\in[H]$ and $(s_h,a_{h,i})\in\S\times\A_i$ as:
$$Q_{h,i}^{\pi_{-i},V}(s_h,a_{h,i})=\E_{a_{h,-i}\sim\pi_{h,-i}(\cdot\mid s_h)}\Mp{r_{h,i}(s_h,a_{h,i},a_{h,-i})+V_{h+1}(s_{h+1})}.$$
Now we formally define Markov games with independent linear function approximation. This definition generalizes the misspecified MDPs with linear function approximation model proposed in \cite{zanette2022stabilizing} to the Markov games setting. 
% \kz{how may a ``definition of a game with FA'' be a generalization of some ``error''? maybe ``relies on''?}

\begin{definition}\label{def:transfer error}
For any player $i$, feature map $\phi_i$ is $\nu$-misspecified with policy set $\Pi^{\mathrm{estimate}}$ if for any rollout policy 
% \kz{not super clear what do we mean by this name here.} 
$\overline{\pi}$, target policy $\widetilde{\pi}$, we have for any $V\in\V$, 
% \kz{isn't this definition also  $\Pi^{\mathrm{estimate}}$-dependent?}, 
$$\max_{\pi\in\Pi^{\mathrm{estimate}}}\abs{\sum_{h=1}^H\E_{\widetilde{\pi}}\Mp{\proj_{[0,H+1-h]}\Sp{\inner{\phi_{i}(s_h,a_{h,i}),\theta_h^{\overline{\pi},\pi_{-i},V}}}-Q_{h,i}^{\pi_{-i},V}(s_h,a_{h,i})}}\leq\nu,$$
where $\Pi^{\mathrm{estimate}}$ is the collection of Markov product policies that need to be evaluated and 
\begin{equation}\label{eq:theta}
    \theta_h^{\overline{\pi},\pi_{-i},V}=\argmin_{\Norm{\theta}\leq H\sqrt{d}}\E_{\overline{\pi}}\Sp{\inner{\phi_{i}(s_h,a_{h,i}),\theta}-Q_{h,i}^{\pi_{-i},V}(s_h,a_{h,i})}^2
\end{equation}
is the parameter for the best linear function fit to $Q_{h,i}^{\pi_{-i},V}$ under rollout policy $\overline{\pi}$.
We say a multi-player general-sum Markov game with features $\{\phi_i\}_{i\in[m]}$ is a $\nu$-misspecified linear Markov game with $\Pi^{\mathrm{estimate}}$ if for any player $i$, the feature map $\phi_i$ is $\nu$-misspecified with $\Pi^{\mathrm{estimate}}$. In addition, we define $d_{\max}:=\max_{i\in[m]}d_i$ as the complexity measure of the linear Markov game. 
\end{definition}

  The policy estimation set $\Pi^{\mathrm{estimate}}$ consists of policies that need to be estimated in the algorithm, which reflects the inductive bias of the algorithm.  We emphasize that all of our algorithms do not require any knowledge of the policy estimation set $\Pi^{\mathrm{estimate}}$ or the misspecification error $\nu$, which is known as the {\it agnostic setting}~\citep{agarwal2020optimality,agarwal2020pc}. Here we give some concrete examples to serve as the special cases of the independent linear Markov game.

% \kz{shall we move these examples to appendix if we are short of space for the colt template? putting these examples here might in turn hurt us, if some picky reviewer challenges that the examples are 1) tabular; 2) single-agent MDP... I would rather suggest to not put them here. maybe even include the state-abstraction game instead (at least looks more different)?}

\begin{example}(Tabular Markov games)
Let $d_i=SA_i$ and set $\phi_i(s,a_i)=e_{(s,a_i)}$ be the canonical basis in $\R^{d_i}$ for all $i\in[m]$. Then we recover tabular Markov game with misspecification error $\nu=0$. 
\end{example}

\begin{example}\label{example:abstraction}(State abstraction Markov games)
Suppose we have an abstraction function $\psi:\S\rightarrow \mathcal{Z}$ for all $h\in[H]$, where $\mathcal{Z}$ is a finite set as the ``state abstractions'' such that states with the same images have similar properties. The model misspecification is defined as
$$\epsilon_h(z):=\max_{s,s':\psi(s)=\psi(s')=z;i\in[m],h\in[H],\a\in\A}\Bp{\abs{r_{h,i}(s,\a)-r_{h,i}(s',\a)},\Norm{\P_h(\cdot\mid s,\a)-\P_h(\cdot\mid s',\a)}_1},\forall z\in\mathcal{Z}.  $$
We define $\nu$-misspecified state abstraction Markov games to satisfy that for any policy $\pi$, we have
$$\abs{\sum_{h=1}^H\E_{\pi}\Mp{\epsilon_h(\psi(s_h))}}\leq\nu,$$
which means the misspecification error is small under any policy $\pi$. 
\end{example}

\begin{restatable}{proposition}{abstraction}\label{prop:abstraction}
$\nu$-misspecified state abstraction Markov games (Example \ref{example:abstraction}) are $H\nu$-misspecified independent linear Markov games with $\Pi^{\mathrm{abstraction}}=\Bp{\pi\mid \pi_h(\cdot\mid s)=\pi_h(\cdot\mid s'),\psi(s)=\psi(s')}$, $d_i=|\mathcal{Z}|A_i$ for all $i\in[m]$ and feature $\phi_i(s,a_i)=e_{\psi(s),a_i}$ to be the canonical basis in $\R^{d_i}$.
\end{restatable}

\begin{example}\label{example:congestion}
(Congestion games) Congestion games are normal-form general-sum games defined by the tuple $(\F,\{A_i\}_{i=1}^m,\{r^f\}_{f\in\F})$, where $\F$ is the facility set with $F=|\F|$, $\A_i\subseteq 2^{\F}$ is the action set for player $i\in[m]$, and $r^f(n)\in[0,1/F]$ is a random reward function with mean $R^f(n)$ for all $n\in[m]$. For a joint action $\a=(a_1,\cdots,a_m)$, $n^f(\a)=\sum_{i=1}^m\mathbf{1}\{f\in a_i\}$ is the number of players choosing facility $f$ and the reward collected for player $i$ is $r_i(\a)=\sum_{f\in a_i}r^f(n^f(\a))$, which is sum of the reward from the facilities they choose.
\end{example}

\begin{restatable}{proposition}{congestion}\label{prop:congestion}
Congestion games (Example \ref{example:congestion}) are independent linear Markov games with $S=1$, $H=1$ and $d_i=F$ for all $i\in[m]$ and misspecification error $\nu=0$.
\end{restatable}

The proofs for Proposition \ref{prop:abstraction} and Proposition \ref{prop:congestion} are deferred to Appendix \ref{apx:linear MG}. These examples demonstrate the generality of the linear Markov games we defined. We want to emphasize that the complexity of tabular Markov games would be $d=S\prod_{i\in[m]}A_i$ if we apply the global function approximation models in \cite{chen2022unified,ni2022representation}, which is exponentially larger than $d_{\max}=S\max_{i\in[m]}A_i$, as in the tabular setting when model-based approaches are used \citep{bai2020provable,zhang2020model,liu2021sharp}. 
See Table~\ref{tab:gen_sum} for a detailed comparison. 

% The independent function approximation model is crucial to designing decentralized \kz{make sure we still want to use the word ``decentralized''. i thought we decided to not use it.} algorithms and break the curse of multiagents, while requiring novel algorithm design and analysis techniques. In the following section, we will discuss how to design decentralized \kz{same here. we are dealing with a self-play/controlled  setting (would we be challenged if we call it ``decentralized''?). lets maybe check and change  throughout?} algorithms in linear Markov games with sample complexity that only has polynomial dependence on $d_{\max}$. 
 
\section{Algorithms and Analyses for Linear Markov Games}\label{sec:linear MG Alg}

\subsection{Experience Replay and Policy Replay}
Before getting into the details of our algorithm, we will first review two popular exploration paradigms in single-agent RL, namely {\it experience replay} and {\it policy replay}. Experience replay is utilized in most empirical and theoretical algorithms, which adds new on-policy data to a dataset and then uses the dataset to retrain a new policy \citep{mnih2013playing,azar2017minimax,jin2020provably}. By carefully designing how to train the new policy to strategically explore the underlying MDP, the dataset will contain more and more information about the MDP and thus we can learn the optimal policy {without any simulator}. 

Another popular approach is called policy replay, which is also known as policy cover. Instead of incrementally maintaining a dataset, the algorithm will maintain a policy set, and at each episode renew the dataset by drawing fresh samples using the policies in this policy set. As the dataset is completely refreshed at each episode, policy replay is able to tackle non-stationarity and enjoy better robustness in many different settings. In \cite{agarwal2020pc}, it is used to address the “catastrophic forgetting” problem in policy gradient methods while being robust to the so-called transfer error. In \cite{zanette2022stabilizing,daskalakis2022complexity}, it is used to tackle the non-stationarity in Q-learning with function approximation and non-stationarity of multiple agents in tabular Markov games, respectively.

In independent linear Markov games, non-stationarity comes from both multiple agents and function approximation.
In particular, the change in other players' policies will lead to a different independent state-action value function to estimate, and the change in the next-step value function estimate will lead to changing targets for regression. 
In our algorithm, we will show that policy replay can tackle both types of non-stationarity at the same time as we use it to create a stationary environment with fixed regression targets, which leads to provably efficient algorithms for independent linear Markov games. Policy replay also guarantees that if each player has a misspecified feature, the final guarantee will only have a linear dependence on the misspecification error. In addition, we will provide a carefully designed policy-replay-type algorithm for tabular Markov games which has significant improvement over \cite{daskalakis2022complexity} in Section \ref{sec:tabular}.

\subsection{Algorithm} 
One technical difficulty in designing algorithms for linear Markov games is that we can no longer resort to adversarial bandits oracles, which is utilized in all algorithms that can break the curse of multiagents \citep{jin2021v,song2021can,mao2022improving,daskalakis2020independent}. This is because adversarial contextual linear bandits oracle is necessary to avoid dependence on $S$ and $A_i$. However, to the best of our knowledge, the only relevant result considering i.i.d. context with known covariance is \cite{neu2020efficient}, which can not fit into Markov games. Indeed, adversarial linear bandits with changing action set is still an open problem (See Section 29.4 in \cite{lattimore2020bandit}).

Perhaps surprisingly, our algorithms only require no-regret learning with full-information feedback oracle (Protocol \ref{algo:no-regret}). This oracle is considerably easier than the previous (weighted) high-probability adversarial bandit with noisy bandit feedback oracles \citep{jin2021v,daskalakis2022complexity}. The intuition is that as all the players are using the same algorithm, the environment is not completely adversarial and we can take multiple i.i.d. samples so that the full-information feedback can be constructed with the batched data.

\paragraph{$\textsc{No\_Regret\_Update}$ subroutine.}
Consider the expert problem with $B$ experts~\citep{freund1997decision}. We use $\B$ to denote the action set with $|\B|=B$, and the policy $p\in\Delta(\B)$. At round $t$, the adversary chooses some loss $l_t$ (also known as the ``expert advice''). Then the learner observes the loss $l_t$ and updates the policy to $p_{t+1}$, which is denoted as $p_{t+1}\leftarrow\textsc{No\_Regret\_Update}(l_t)$. 

For learning CCE and CE, the no-regret learning oracle needs to satisfy the following no-external-regret and no-swap-regret properties, respectively. We will use the minimax optimal no-external-regret and no-swap-regret algorithms while any other no-regret algorithms are eligible. Assumption \ref{asp:no regret} and Assumption \ref{asp:no swap regret} can be achieved by EXP3 \citep{freund1997decision} and BM-EXP3 \citep{blum2007external}, respectively.

\begin{protocol}[tb]
    \caption{No-regret Learning Algorithm}
    \label{algo:no-regret}
    \begin{algorithmic}
        \State {\bfseries Initialize}: Action set $\mathcal{B}$, and $p_1$ to be the uniform distribution over $\mathcal{B}$.
        \For{$t=1,2,\dots,T$}
        \State Adversary chooses loss $l_t$. 
        \State Observe loss $l_t$.
        \State Update $p_{t+1}\leftarrow\textsc{No\_Regret\_Update}(l_t)$. 
        \EndFor
    \end{algorithmic}
\end{protocol}

\begin{assumption}\label{asp:no regret}(No-external-regret with full-information feedback)
For any loss sequence $l_1,\dots,l_T\in\R^B$ bounded between $[0,1]$, the no-regret learning oracle (Protocol \ref{algo:no-regret}) enjoys external-regret \citep{freund1997decision}: 
% \kz{should we switch $\leq $ and $:=$ below?}
$$\max_{b\in\B}\sum_{t=1}^T\Sp{\inner{p_t,l_t}-l_t(b)}\leq\mathrm{Reg}(T):=O(\sqrt{\log(B)T}). $$
\end{assumption}

\begin{assumption}\label{asp:no swap regret}(No-swap-regret with full-information feedback)
For any loss sequence $l_1,\dots,l_T\in\R^B$ bounded between $[0,1]$, the no-regret learning oracle (Protocol \ref{algo:no-regret}) enjoys swap-regret \citep{blum2007external,ito2020tight}:
$$\max_{\psi\in\Psi}\sum_{t=1}^T\Sp{\inner{p_t,l_t}-\inner{\psi\diamond p_t,l_t}}\leq\mathrm{SwapReg}(T):=O(\sqrt{B\log(B)T}),$$
where $\Psi$ denote the set $\{\psi:\B\rightarrow \B\}$ which consists of all possible strategy modifications. 
\end{assumption}

\begin{algorithm}[!]
    \caption{\textbf{P}olicy \textbf{Re}ply with \textbf{F}ull \textbf{I}nformation Oracle in Independent Linear Markov Games (\algoname) }
	\label{algo}
    \begin{algorithmic}[1]
        \State {\bfseries Input:} $\epsilon$, $\delta$, $d_{\max}$, $\lambda$, $\beta$, $T_{\Trig}$, $K_{\max}$, $T$, $N$
        \State {\bfseries Initialization:} Policy Cover $\Pi=\emptyset$. $n^\tot=0$. 
        \For{episode $k=1,2,\dots,K_{\max}$}
        \State Set $\oV_{H+1,i}^k(\cdot)=\uV_{H+1,i}^k(\cdot)=0$, $n^k=0$.  
        \For{$h=H,H-1,\dots,1$}\Comment{Retrain policy with the current policy cover}\label{line:h}
        \State Initialize $\pi_{h,i}^{k,1}$ to be uniform policy for all player $i$. Initialize $\oV_{h,i}^k(\cdot)=\uV_{h,i}^k(\cdot)=0$.  
        \State Each player $i$ initializes a no-regret learning instance (Protocol \ref{algo:no-regret}) at each state $s\in\S$ and step $h\in[H]$, for which we will use $\textsc{No\_Regret\_Update}_{h,i,s}(\cdot)$ to denote the update. 
        \For{$t=1,2,\dots,T$}\label{line:t}
        \For{$i\in[m]$}
        %\STATE Call Algorithm \ref{algo:streaming LS} with policy cover $\Pi$ and opponent strategy $\pi_{h,-i}^{k,t}$ and receive $\ot_{h,i}^{k,t}$, $\ut_{h,i}^{k,t}$ and $[\Sigma_{h,i}^{k,t}]^{-1}$.
        \State Set Dataset $\D_{h,i}^{k,t}=\emptyset$.
        \For{$l=1,2,\dots,\sum_{j=1}^{k-1}n^j$}\label{line:l}
        \State Sample $\pi^l\in\Pi=\{\pi^j\}_{j=1}^{k-1}$ with probability $n^l/\sum_{j=1}^{k-1}n^j$.  
        \State Draw a joint trajectory $(s_{1}^l,\a_1^l,r_{1,i}^l,\dots,s_h^l,\a_h^l,r_{h,i}^l,s_{h+1}^l)$ from $\pi^l_{1:h-1}\circ\Sp{\pi^l_{h,i},\pi_{h,-i}^{k,t}}$, which is the policy that follows $\pi^l$ for the first $h-1$ steps and follows $\pi^l_{h,i},\pi_{h,-i}^{k,t}$ for step $h$. 
        \State Add $(s_h^l,a_{h,i}^l,r_{h,i}^l,s_{h+1}^l)$ to $\D_{h,i}^{k,t}$. 
        \EndFor
        \State Set $\Sigma_{h,i}^{k,t}=\lambda I+\sum_{(s,a,r,s')\in\D_{h,i}^{k,t}}\phi_i(s,a)\phi_i(s,a)^\top$. 
        %\State Set $\ot_{h,i}^{k,t}=[\Sigma_{h,i}^{k,t}]^{-1}\sum_{(s_h^l,a_{h,i}^l,r_{h,i}^l,s_{h+1}^l)\in\D_{h,i}^{k,t}}\phi_i(s_h^l,a_{h,i}^l)(r_{h,i}^l+\oV^k_{h+1,i}(s_{h+1}^l))$.
        %\State Set $\ut_{h,i}^{k,t}=[\Sigma_{h,i}^{k,t}]^{-1}\sum_{(s_h^l,a_{h,i}^l,r_{h,i}^l,s_{h+1}^l)\in\D_{h,i}^{k,t}}\phi_i(s_h^l,a_{h,i}^l)(r_{h,i}^l+\uV^k_{h+1,i}(s_{h+1}^l))$.
        \State Set $\ot_{h,i}^{k,t}=\argmin_{\Norm{\theta}\leq H\sqrt{d_{\max}}}\sum_{(s,a,r,s')\in\D_{h,i}^{k,t}}\Sp{\inner{\phi_i(s,a),\theta}-r-\oV^k_{h+1,i}(s')}^2$.
        \State Set $\ut_{h,i}^{k,t}=\argmin_{\Norm{\theta}\leq H\sqrt{d_{\max}}}\sum_{(s,a,r,s')\in\D_{h,i}^{k,t}}\Sp{\inner{\phi_i(s,a),\theta}-r-\uV^k_{h+1,i}(s')}^2$.
        \State Set $\oQ_{h,i}^{k,t}(\cdot,\cdot)=\proj_{[0,H+1-h]}\Sp{\inner{\phi_i(\cdot,\cdot),\ot_{h,i}^{k,t}}+\beta\Norm{\phi_i(\cdot,\cdot)}_{[\Sigma_{h,i}^{k,t}]^{-1}}}$.\label{line:oq}
        \State Set $\uQ_{h,i}^{k,t}(\cdot,\cdot)=\proj_{[0,H+1-h]}\Sp{\inner{\phi_i(\cdot,\cdot),\ut_{h,i}^{k,t}}-\beta\Norm{\phi_i(\cdot,\cdot)}_{[\Sigma_{h,i}^{k,t}]^{-1}}}$.
        \State Update $\oV_{h,i}^k(s)\leftarrow\frac{t-1}{t}\oV^k_{h,i}(s)+\frac{1}{t}\sum_{a_i\in\A_i}\pi_{h,i}^{k,t}(a_i|s)\oQ_{h,i}^{k,t}(s,a)$ for all $s\in\S$.
        \State Update $\uV_{h,i}^k(s)\leftarrow\frac{t-1}{t}\uV^k_{h,i}(s)+\frac{1}{t}\sum_{a_i\in\A_i}\pi_{h,i}^{k,t}(a_i|s)\uQ_{h,i}^{k,t}(s,a)$ for all $s\in\S$.\label{line:uv}
        \State Update the no-regret learning instance for all state $s$ at step $h$: $\pi_{h,i}^{k,t+1}(\cdot\mid s)\leftarrow\textsc{No\_Regret\_Update}_{h,i,s}(1-\oQ_{h,i}^{k,t}(s,\cdot)/H)$.\label{line:no regret update}
        \EndFor
        \EndFor
        \State Set $\oV^k_{h,i}(s)\leftarrow \proj_{[0,H+1-h]}\Sp{\oV^k_{h,i}(s)+\frac{H}{T}\cdot\mathrm{(Swap)Reg}(T)}$ for all $i\in[m]$ and $s\in\S$. \label{line:ov}
        \EndFor
        \State Set $\pi^{k}$ to be the Markov joint policy such that $\pi^{k}_h(\a|s)=\frac{1}{T}\sum_{t=1}^T\prod_{i\in[m]}\pi_{h,i}^{k,t}(a_i|s)$.
        \If{$n^\tot=N$}\label{line:terminate}
        \State Output $\pi^{\mathrm{output}}=\pi^{k^{\mathrm{output}}}$, where $k^{\mathrm{output}}=\argmin_{k'\in[k]}\max_{i\in[m]}\oV_{1,i}^{k'}(s_1)-\uV_{1,i}^{k'}(s_1)$. \label{line:certificate}
        \EndIf
        \State Set $T_{h,i}=0$, for all $h\in[H],i\in[m]$.
        \Repeat\Comment{Update policy cover}\label{line:repeat}
        \State Reset to $s=s_1$, $n^k=n^k+1$, $n^\tot=n^\tot+1$.
        \For{$h=1,2,\dots,H$}
        \State Play $\a=\pi_h^k(\cdot|s)$.
        \For{$i\in[m]$}
        \State $T_{h,i}\rightarrow T_{h,i}+\Norm{\phi_i(s,a_i)}_{[\Sigma_{h,i}^{k,1}]^{-1}}^2$.\label{line:T}
        \EndFor
        \State Get next state $s'$, $s\rightarrow s'$.
        \EndFor
        \Until{$\exists h\in[H],i\in[m]$ such that $T_{h,i}\geq T_{\Trig}$ or $n^\tot=N$. } \label{line:trigger}
        \State Update $\Pi\leftarrow\Pi\bigcup\{(\pi^k,n^k)\}$. 
	    \EndFor

    \end{algorithmic}
\end{algorithm}

We will explain the algorithm for learning Markov CCE and the only difference in learning Markov CE is to use the no-swap-regret oracle to replace the no-external-regret one.  The algorithm has two main components: learning Markov CCE with policy cover and policy cover update. For the first part, given a policy cover $\Pi$, we will compute an approximate optimistic CCE under the distribution induced by the policy cover. Specifically, we use a value-iteration-type algorithm that computes the CCE and the corresponding value function from step $H$ to $1$ (Line \ref{line:h}). At each step $h$, each player will run a no-regret algorithm for $T$ steps (Line \ref{line:t}). In this inner loop, we will generate a dataset by using policies in the policy cover concatenated with the current policies from the no-regret oracle (Line \ref{line:l}). Then we compute an optimistic local Q function $\oQ_{h,i}^{k,t}$ via constrained least squares and feed it into the no-regret algorithm as the full-information feedback (Line \ref{line:oq} and Line \ref{line:no regret update}). At the end of the no-regret loop, we will compute the optimistic value function, which will be an upper bound of the best response value with high probability (Line \ref{line:ov}).

For the policy cover update part, we utilize a lazy update to ensure that the algorithm will end within $K\leq K_{\max}:=\widetilde{O}(mHd_{\max})$ episodes with high probability, which can significantly improve the final sample complexity bound, similar to the single-agent MDP case studied in \cite{zanette2022stabilizing}. We maintain a counter $T_{h,i}$ for each player $i$ at each step $h$, which estimates the information gained by adding the current policy $\pi^k$ to the existing policy cover (Line \ref{line:T}). Whenever there is a counter satisfying $T_{h,i}\geq T_\Trig$ for some carefully chosen parameter $T_\Trig$, we will add $(\pi^k,n^k)$ to the policy cover, where $n^k$ is the number of times that $\pi^k$ should be repeated in data collection. In addition, the algorithm will terminate when the dataset size reaches $N$ (Line \ref{line:trigger} and Line \ref{line:terminate}) so that the sample complexity is always upper bounded by $O(mHTK_{\max}N)$.

We also have a policy certification part, where similar ideas have been utilized in \cite{dann2019policy,liu2021sharp,ni2022representation} to convert regret-based analysis to sample complexity. Specifically, we maintain a pessimistic value estimate $\uV_{1,i}^k(s_1)$, which satisfies $\uV_{1,i}^k(s_1)\leq V_{1,i}^{\pi^k}(s_1)$ with high probability (Line \ref{line:uv}). Thus the output policy is the best approximation of Markov CCE in the policy cover. This technique can be applied to most no-regret algorithms in RL to transform regret bounds to sample complexity bounds with a better dependence on the failure probability $\delta$.\footnote{In \cite{jin2018q}, they show how to transform regret bounds to sample complexity bounds while the dependence on failure probability becomes $1/\delta$. This technique can improve it to $\log(1/\delta)$.}

\subsection{Decentralized Implementation}

Now we discuss the implementation details of the algorithm. Our algorithm can be implemented in a decentralized manner as specified below:

\begin{enumerate}
    \item All players know the input parameters of the algorithm. \label{1}
    \item Each player only knows their own features $\phi_i(\cdot,\cdot)$  
    % \kz{note that our algorithm requires the knowledge of $d_{\max}$. this needs to be ``shared'' also?}\qiwen{I set it as the input parameter now} 
    and observes the states, individual actions, and individual rewards in each sample trajectory. \label{2}
    \item All players have shared random seeds to sample from the output Markov joint policy $\pi^{\mathrm{output}}$. \label{3}
    \item All players have shared random seeds to sample from the Markov joint policy $\pi^k$, which is the policy learned at episode $k$. \label{4}
    \item All players can communicate $O(1)$ bit at each episode $k\in[K]$. \label{5}
\end{enumerate}

V-learning \citep{jin2021v,song2021can,mao2022improving} can be implemented with \eqref{1}, \eqref{2} and \eqref{3}, and SPoCMAR \citep{daskalakis2022complexity} can be implemented with \eqref{1}, \eqref{2}, \eqref{3} and \eqref{4}. Similar to the algorithm proposed in \cite{daskalakis2022complexity}, our algorithm can be implemented in a decentralized way with shared random seeds to enable sampling from the Markov joint policy $\pi^k$. In details, when the players want to sample $\a\sim\pi_h^k(\a\mid s)=\frac{1}{T}\sum_{t=1}^T\prod_{i\in[m]}\pi_{h,i}^{k,t}(a_i\mid s)$, each player samples $t\sim \mathrm{Unif}(T)$ with the shared random seed and then independently samples $a_i\sim\pi_{h,i}^{k,t}(a_i\mid s)$. Our algorithm also requires $O(1)$ communication for broadcasting the policy cover update (Line \ref{line:trigger}) and the output policy (Line \ref{line:certificate}) at each episode.\footnote{Line \ref{line:certificate} can be implemented with $O(1)$ communication at each episode by maintaining the best index and corresponding value up to the current episode $k$.} The total communication complexity is bounded by $O(K_{\max})=\widetilde{O}(mHd_{\max})$ with only polylog  dependence on the accuracy $\epsilon$. 

In Appendix \ref{apx:w/o com}, we present another algorithm for MARL in independent linear Markov games without communication, which can be implemented with \eqref{1}, \eqref{2}, \eqref{3} and \eqref{4}. To remove communication, we utilize agile   policy cover update and the number of episodes becomes $K=\widetilde{O}(m^2H^4d_{\max}^2\epsilon^{-2})$. As a result, the final sample complexity will be worse than Algorithm \ref{algo}. It would be an interesting future direction to study this tradeoff {between communication and sample complexity}. 

 %The memory complexity of the algorithm is $\widetilde{O}(N)$. However, for two-player zero-sum linear Markov games, we can improve the memory complexity to $\widetilde{O}(1)$ as we no longer need to store $\pi_{h,i}^{k,t}$ for all $t\in[T]$ to compute $\pi_h^k$.

\subsection{Guarantees}

Our algorithm, \algoname, has the following guarantees for learning Markov CCE and Markov CE in linear Markov games. The sample complexity only has polynomial dependence on $d_{\max}$, which exponentially improves all the previous results for Markov games with function approximation. Note that the $\widetilde{O}(\cdot)$ notation here only hide polylog dependence on $m,H,d_{\max},\epsilon,\delta$, and the $\log(A_{\max})$ factor in the bound can be replaced by $d_{\max}$ as in adversarial linear bandits \citep{bubeck2012towards}.

\begin{restatable}{theorem}{CCE}\label{thm:CCE}
Suppose Algorithm \ref{algo} is instantiated with no-regret learning oracles satisfying Assumption \ref{asp:no regret}. Then for $\nu$-misspecified independent linear Markov games with $\Pi^{\mathrm{estimate}}=\{\pi^{k,t}\}_{k,t=1,1}^{K,T}$, with probability at least $1-\delta$, Algorithm \ref{algo} will output an $(\epsilon+4\nu)$-approximate Markov CCE. The sample complexity is $O(mHTK_{\max}N)=\widetilde{O}(m^4H^{10}d_{\max}^4\log(A_{\max})\epsilon^{-4})$, where $d_{\max}=\max_{i\in[m]}d_i$ and $A_{\max}=\max_{i\in[m]}A_i$. 
\end{restatable} 

\begin{restatable}{theorem}{CE}\label{thm:CE}
Suppose Algorithm \ref{algo} is instantiated with no-regret learning oracles satisfying Assumption \ref{asp:no swap regret}. Then for $\nu$-misspecified independent linear Markov games with $\Pi^{\mathrm{estimate}}=\{\pi^{k,t}\}_{k,t=1,1}^{K,T}$, with probability at least $1-\delta$, Algorithm \ref{algo} will output an $(\epsilon+4\nu)$-approximate Markov CE. The sample complexity is $O(mHTK_{\max}N)=\widetilde{O}(m^4H^{10}d_{\max}^4A_{\max}\log(A_{\max})\epsilon^{-4})$.  
\end{restatable}

The choice of input parameters and the proofs are deferred to Appendix \ref{apx:linear MG proof}. As Markov CCE is equivalent to Markov NE in two-player zero-sum Markov games, we directly have the following Corollary. 

\begin{corollary}\label{crl:NE}
Suppose Algorithm \ref{algo} is instantiated with no-regret learning oracles satisfying Assumption \ref{asp:no regret}. Then for $\nu$-misspecified independent linear two-player zero-sum Markov games with $\Pi^{\mathrm{estimate}}=\{\pi^{k,t}\}_{k,t=1,1}^{K,T}$, with probability at least $1-\delta$, Algorithm \ref{algo} will output an $(\epsilon+4\nu)$-approximate Markov NE. The sample complexity is $O(mHTK_{\max}N)=\widetilde{O}(m^4H^{10}d_{\max}^4\log(A_{\max})\epsilon^{-4})$.  
\end{corollary}

By Proposition \ref{prop:abstraction}, we have the following corollary for state abstraction Markov games. Note that the feature is the same $\phi_i(s,a_i)=\phi_i(s',a_i)$ if $\psi(s)=\psi(s')$, so $\pi^{k,t}\in\Pi^{\mathrm{abstraction}}$ for all $(k,t)\in[K]\times[T]$ as the full-information feedback would be the same for $s$ and $s'$ mapped to the same abstraction and then the policy would be same as well.

\begin{corollary}\label{crl:abstraction}
Suppose Algorithm \ref{algo} is instantiated with no-regret learning oracles satisfying Assumption \ref{asp:no regret}. Then for $\nu$-misspecified state abstraction Markov games, with probability at least $1-\delta$, Algorithm \ref{algo} will output an $(\epsilon+4H\nu)$-approximate Markov NE. The sample complexity is $O(mHTK_{\max}N)=\widetilde{O}(m^4H^{10}|\mathcal{Z}|^4A_{\max}^4\log(A_{\max})\epsilon^{-4})$.  
\end{corollary}

\section{Learning Markov NE in Independent Linear Markov Potential Games}\label{sec:mpg}

In this section, we will focus on a special class of independent linear Markov games, namely {\it independent linear Markov potential games}. The existence of the potential function guarantees that the stationary points of the potential function are NE \citep{leonardos2021global}, which means the iterative best-response dynamic can converge to NE as it is similar to coordinate descent  \citep{durand2018analysis}. Specifically, we will provide an iterative best-response-type  algorithm that can learn pure Markov NE in independent linear Markov potential games, which generalizes the algorithm for tabular Markov potential games in \cite{song2021can}.

%For linear two-player zero-sum Markov games, Algorithm \ref{algo} can be directly applied as every Markov CCE in two-player zero-sum Markov games is a Markov NE \citep{liu2021sharp}. On the other hand, we provide a variant of Algorithm \ref{algo} to improve the memory complexity. Note that even for normal-form multi-player general-sum games, the no-regret dynamic requires $O(\epsilon^{-2})$ memory for each player to converge to an $\epsilon$-approximate CCE/CE. For two-player zero-sum Markov games, our algorithm will only require $\widetilde{O}(d_i^3H^2)$ which only has polylog dependence on $\epsilon$. The intuition is that for two-player zero-sum games, we only need to store the average policy instead of all the history policies.

As when the other players are fixed, player $i\in[m]$ will be in an approximate linear MDP, existing algorithms for misspecified linear MDP can all serve as the best-response oracle. The algorithm will use the following oracle $\textsc{LinearMDP\_Solver}$ that can solve misspecified linear MDPs. Here misspecified linear MDPs are the degenerated cases of misspecified independent linear Markov games with only one player and thus no $\Pi^{\mathrm{estimate}}$ is included, which is similar to the model in \cite{zanette2022stabilizing}.

\begin{definition}\label{def:transfer error mdp}
Feature map $\phi:\S\times\A\rightarrow\R^d$ is $\nu$-misspecified if for any rollout policy  $\overline{\pi}$, target policy $\widetilde{\pi}$, we have for any $V\in\V$, 
$$\abs{\sum_{h=1}^H\E_{\widetilde{\pi}}\Mp{\proj_{[0,H+1-h]}\Sp{\inner{\phi(s_h,a_{h}),\theta_h^{\overline{\pi},V}}}-Q_{h}^{V}(s_h,a_h)}}\leq\nu,$$
where
$$\theta_h^{\overline{\pi},V}=\argmin_{\Norm{\theta}\leq H\sqrt{d}}\E_{\overline{\pi}}\Sp{\inner{\phi(s_h,a_{h}),\theta}-Q_{h}^{V}(s_h,a_h)}^2,Q_{h}^{V}(s_h,a_h)=\E\Mp{r_{h}(s_h,a_{h})+V_{h+1}(s_{h+1})}.$$
We say a Markov decision process with feature $\phi$ is a $\nu$-misspecified linear MDP if the feature map $\phi$ is $\nu$-misspecified.
\end{definition}

\begin{assumption}\label{asp:linear MDP}
For any $\nu$-misspecified linear MDP 
% \kz{did we define this notion for MDPs? are we ``misspecifying'' in the same way as the game one defined before? lets maybe define first.} 
with feature $\phi(s,a)\in\R^d$ for all $(s,a)\in\S\times\A$,
% \kz{would this be confusing with joint action space in this paper $\A$?}
$\textsc{LinearMDP\_Solver}$ takes features $\phi(\cdot,\cdot)$ as input and can interact with the underlying linear MDP. Then it can output an $(\epsilon+O(\nu))$-approximate optimal policy with sample complexity $\mathrm{LinearMDP\_SC}(\epsilon,\delta,d)$ with probability at least $1-\delta$. Without loss of generality, we assume that $\mathrm{LinearMDP\_SC}(\epsilon,\delta,d)$ is non-decreasing w.r.t. $d$. 
\end{assumption}

In Appendix \ref{apx:linear MDP}, we will adapt Algorithm \ref{algo} to the single-agent case to serve as $\textsc{LinearMDP\_Solver}$ with $\mathrm{LinearMDP\_SC}(\epsilon,\delta,d)=\widetilde{O}(H^6d^4\epsilon^{-2})$ and output an $(\epsilon+4\nu)$-optimal policy (See Algorithm \ref{algo:mdp}). With the best-response oracle, we provide our MARL algorithm for linear Markov potential games (Algorithm \ref{algo:mpg}). It is easy to see that Algorithm \ref{algo:mpg} can be implemented in the same decentralized way as Algorithm \ref{algo}. Below we provide the sample complexity guarantees. 

\begin{algorithm}[t]
    \caption{\textbf{Nash} \textbf{C}oordinate \textbf{A}scent for Independent \textbf{Lin}ear Markov Potential Games (\algonamempg)}
	\label{algo:mpg}
    \begin{algorithmic}[1]
        \State {\bfseries Input:} $\epsilon$, $\delta$, $K=5mH\epsilon^{-1}$
        \State {\bfseries Initialization:} $\pi^1$  to be an arbitrary deterministic policy. 
        \For{episode $k=1,2,\dots,K$}
        \State Execute policy $\pi^k$ for $\widetilde{O}(H^2\epsilon^{-2})$ episodes and obtain $\widehat{V}_{1,i}^{\pi^k}(s_1)$ as the empirical average 
        % \kz{shall we introduce it formally, i.e., give the formula?} 
        of the total reward for all player $i\in[m]$. 
        \For{$i\in[m]$}
        \State Fix all the players except player $i$ to follow policy $\pi_{-i}^k$ and player $i$ runs  $\textsc{LinearMDP\_Solver}$
         with feature $\phi_i(\cdot,\cdot)$, accuracy $\epsilon/8$ and failure probability $\delta/(2mK)$. Set $\widehat{\pi}_{i}^{k+1}$ to be the output of $\textsc{LinearMDP\_Solver}$.
        \State Execute policy $(\widehat{\pi}_i^{k+1},\pi_{-i}^k)$ for $\widetilde{O}(H^2\epsilon^{-2})$ episodes and obtain $\widehat{V}_{1,i}^{\widehat{\pi}_i^{k+1},\pi_{-i}^k}(s_1)$ as the empirical average of the total reward. 
        \State Set $\Delta_i\leftarrow \widehat{V}_{1,i}^{\widehat{\pi}_i^{k+1},\pi_{-i}^k}(s_1)-\widehat{V}_{1,i}^{\pi^k}(s_1)$. 
        \EndFor
        \If{$\max_{i\in[m]}\Delta_i>\epsilon/2$}
        \State Set $\pi^{k+1}:\pi^{k+1}_i=\pi^k_i,\pi^{k+1}_j=\widehat{\pi}^k_j$ for $i\neq j$ and $j=\argmax_{i\in[m]}\Delta_i$. 
        \Else
        \State Output $\pi^{\mathrm{output}}=\pi^k$.
        \EndIf
        \EndFor
    \end{algorithmic}
\end{algorithm}

\begin{restatable}{theorem}{mpg}\label{thm:mpg}
For $\nu$-misspecified independent linear Markov potential games with $\Pi^{\mathrm{estimate}}=\{\pi^{k}\}_{k=1}^{K}$, with probability at least $1-\delta$, Algorithm \ref{algo:mpg} will output an $(\epsilon+O(\nu))$-approximate pure Markov NE. The sample complexity is $O(m^2H\epsilon^{-1}\cdot\mathrm{LinearMDP\_SC}(\epsilon/8,\delta/(10m^2H\epsilon^{-1}),d_{\max}))$. 
\end{restatable}

As the congestion game is a special case of linear Markov potential game (Proposition \ref{prop:congestion}), we have the following corollary if we replace the linear MDP solver with a linear bandit solver with $\mathrm{LinearBandit\_SC}(\epsilon,\delta,d))$ sample complexity.

\begin{corollary}
For congestion games, with probability at least $1-\delta$, Algorithm \ref{algo:mpg} will output an $\epsilon$-approximate pure NE. The sample complexity is $O(m^2\epsilon^{-1}\cdot\mathrm{LinearBandit\_SC}(\epsilon/8,\delta/(10m^2H\epsilon^{-1}),F))$. 
\end{corollary}

If we use Algorithm \ref{algo:mdp} as the oracle, the sample complexity for linear Markov potential games would be $\widetilde{O}(m^2H^7d_{\max}^4\epsilon^{-3})$. For linear bandits, it is easy to adapt the $\widetilde{O}(d\sqrt{K})$ algorithm in \cite{abbasi2011improved} to sample complexity $\widetilde{O}(d^2\epsilon^{-2})$, which leads to $\widetilde{O}(m^2F^2\epsilon^{-3})$ sample complexity for congestion games.\footnote{E.g., we can use policy certification as in Algorithm \ref{algo} to find the best policy among all the policies played with no additional sample complexity.} Our algorithm significantly improves the previous result for the decentralized algorithm, which has sample complexity $\widetilde{O}(m^{12}F^6\epsilon^{-6})$~\citep{cui2022learning}.

\section{Improved Sample Complexity in Tabular Case}\label{sec:tabular}

In this section, we will present an algorithm specialized to tabular Markov games based on the policy cover technique in Algorithm \ref{algo}. The sample complexity for learning an $\epsilon$-approximate Markov CCE is $\widetilde{O}(H^6S^2A_{\max}\epsilon^{-2})$, which significantly improves the previous state-of-the-art result $\widetilde{O}(H^{11}S^3A_{\max}\epsilon^{-3})$ \citep{daskalakis2022complexity}, and is only worse than learning an $\epsilon$-approximate {\it non-Markov}  CCE by a factor of $HS$ \citep{jin2021v}. In addition, our algorithm can learn an $\epsilon$-approximate Markov CE with $\widetilde{O}(H^6S^2A_{\max}^2\epsilon^{-2})$ sample complexity, which is the first provably efficient result for learning Markov CE in tabular Markov games.

\paragraph{$\textsc{Adv\_Bandit\_Update}$ subroutine. }Consider the adversarial multi-armed bandit problem with $B$ arms. At round $t$, the adversary chooses some loss $l_t$ and the learner chooses some action $b_t\sim p_t$, where $p_t\in\Delta(\B)$ is the policy at round $t$. Then the learner observes a noisy bandit-feedback $\widetilde{l}_t(b_t)\in[0,1]$ such that $\E[\widetilde{l}_t(b_t)\mid l_t,b_t]=l_t(b_t)$. The player will update the policy to $p_{t+1}$ for round $t+1$, which is denoted as $p_{t+1}\leftarrow\textsc{Adv\_Bandit\_Update}(b_t,\widetilde{l}_t(b_t))$.

For learning CCE and CE, the adversarial bandit algorithm (Protocol \ref{algo:adv bandits}) needs to satisfy the following no-external-regret and no-swap-regret properties,  respectively. The following two assumptions can be achieved by leveraging the results in \cite{neu2015explore} and \cite{blum2007external}, which is shown in \cite{jin2021v}.\footnote{They proved a stronger version for weighted regret while we only require the unweighted version.} 

\begin{protocol}[t]
    \caption{Adversarial Bandit  Algorithm}
    \label{algo:adv bandits}
    \begin{algorithmic}
        \State {\bfseries Initialize}: Action set $\mathcal{B}$, and $p_1$ to be the uniform distribution over $\mathcal{B}$.
        \For{$t=1,2,\dots,T$}
        \State Adversary chooses loss $l_t$. 
        \State Player take action $b_t\sim p_t$ and observe noisy bandit-feedback $\widetilde{l}_t(b_t)$. 
        \State Update $p_{t+1}\leftarrow\textsc{Adv\_Bandit\_Update}(b_t,\widetilde{l}_t(b_t))$. 
        \EndFor
    \end{algorithmic}
\end{protocol}

\begin{algorithm}[!]
    \caption{\textbf{P}olicy \textbf{Re}ply with \textbf{B}andit \textbf{O}racle in Tabular Markov Games (\algonametab)}
    \label{algo:tabular}
    \begin{algorithmic}[1]
        \State {\bfseries Input:} $\epsilon$, $\delta$, $\beta$, $T_{\Trig}$, $K_{\max}$, $N_{\max}$
        \State {\bfseries Initialization:} Policy Cover $\Pi=\emptyset$. $n^\tot=0$.
        \For{episode $k=1,2,\dots,K_{\max}$}
        \State Set $\oV_{H+1,i}^k(\cdot)=\uV_{H+1,i}^k(\cdot)=0$, $n^k=0$, $n^k_h(s)=0$ for all $h\in[H]$ and $s\in\S$. 
        \For{$h=H,H-1,\dots,1$}\Comment{Retrain policy with the current policy cover}
        \State Initialize $\pi_{h,i}^{k,1}$ to be uniform policy for all player $i$. Initialize $\oV_{h,i}^k(\cdot)=\uV_{h,i}^k(\cdot)=0$.  
        \State Each player $i$ initializes an adversarial bandit instance (Protocol \ref{algo:adv bandits}) at each state $s\in\S$ and step $h\in[H]$, for which we will use $\textsc{No\_Regret\_Update}_{h,i,s}(\cdot)$ to denote the update. 
        \For{$t=1,2,\dots,\sum_{j=1}^{k-1}n^j$}
        \State Sample $\pi^l\in\Pi$ with probability $n^l/\sum_{j=1}^{k-1}n^j$. 
        \State Draw a joint trajectory $(s_1,\a_1,\r_{1},\dots,s_h,\a_h,\r_{h},s_{h+1})$ from $\pi^l_{1:h-1}\circ\pi_{h}^{k,t}$,which is the policy that follows $\pi^l$ for the first $h-1$ steps and follows $\pi_{h}^{k,t}$ for step $h$.
        \State Update $n^k_h(s_h)\leftarrow n^k_h(s_h)+1$.
        \State Update the adversarial bandit instance for player $i$ at step $h$ and state $s_h$: $\pi_{h,i}^{k,t+1}(\cdot|s_h)\leftarrow\textsc{Adv\_Bandit\_Update}_{h,i,s_h}(a_{h,i},1-\Sp{r_{h,i}+\oV_{h+1,i}^k(s_{h+1})}/H)$.
        \State Update policy $\pi_{h,i}^{k,t+1}(\cdot|s)\leftarrow\pi_{h,i}^{k,t+1}(\cdot|s)$ for $s\neq s_h$.
        \State Update $\oV_{h,i}^k(s_h)\leftarrow\frac{n^k_h(s_h)-1}{n^k_h(s_h)}\oV^k_{h,i}(s_h)+\frac{1}{n^k_h(s_h)}(r_{h,i}+\oV_{h+1,i}^k(s_{h+1}))$.
        \State Update $\uV_{h,i}^k(s_h)\leftarrow\frac{n^k_h(s_h)-1}{n^k_h(s_h)}\uV^k_{h,i}(s_h)+\frac{1}{n^k_h(s_h)}(r_{h,i}+\uV_{h+1,i}^k(s_{h+1}))$.
        \EndFor
        \State Set $\oV^k_{h,i}(s)\leftarrow \proj_{[0,H+1-h]}\Sp{\oV^k_{h,i}(s)+\frac{H}{T}\cdot\mathrm{B(Swap)Reg}(n^k_h(s_h))+\beta_{n^k_h(s)}}$ for all $i\in[m]$ and $s\in\S$. 
        \State Set $\uV^k_{h,i}(s)\leftarrow \proj_{[0,H+1-h]}\Sp{\uV^k_{h,i}(s)-\beta_{n^k_h(s)}}$ for all $i\in[m]$ and $s\in\S$. 
        \EndFor
        \State Set $\pi^{k}$ to be the Markov joint policy such that $\pi^{k}_h(\a|s)=\frac{1}{n_h^k(s)}\sum_{j=1}^{n_h^k(s)}\prod_{i\in[m]}\pi_{h,i}^{k,t_h^k(j;s)}(a_i|s)$, where $t_h^k(j;s)$ is the time $t$ such that state $s$ is visited for the $j$-th time in episode $k$ at step $h$.
        \If{$\max_{i\in[m]}\oV_{1,i}^k(s_1)-\uV_{1,i}^k(s_1)\leq\epsilon$}\Comment{Policy certification}\label{line:terminate tabular}
        \State {\bfseries Output:} $\pi^{\mathrm{output}}=\pi^t$.
        \EndIf
        \State Set $T_h^k(s)=0$ for all $h\in[H]$, $s\in\S$.
        \Repeat\Comment{Update policy cover}
        \State Reset $s=s_1$, $n^k=n^k+1$, $n^\tot=n^\tot+1$.
        \For{$i\in[m]$}
        \For{$h=1,2,\dots,H$}
        \State Play $\a_h=\pi_h^k(\cdot|s)$.
        \State $T^k_h(s_h)\leftarrow T^k_h(s_h)+1$.
        \State Get next state $s'$, $s\rightarrow s'$.
        \EndFor
        \EndFor
        \Until{$\exists h\in[H]$ such that $T^k_h(s_h)=n^k_h(s_h)\vee T_\Trig$ or $n^\tot=N_{\max}$.} \label{line:trigger tabular}
        \State Update $\Pi\leftarrow\Pi\bigcup\{(\pi^k,n^k)\}$. 
	    \EndFor
    \end{algorithmic}
\end{algorithm}

\begin{assumption}\label{asp:no regret bandit}(No-external-regret with bandit-feedback)
For any loss sequence $l^1,\dots,l^T\in\R^B$ bounded between $[0,1]$, the adversarial bandit oracle satisfies that with probability at least $1-\delta$, for all $t\leq T$,
$$\max_{b\in\B}\sum_{i=1}^t\Sp{\inner{p_i,l_i}-l_i(b)}\leq \mathrm{BReg}(t):=O\Sp{\sqrt{Bt}\log(Bt/\delta)}. $$
\end{assumption}

\begin{assumption}\label{asp:no swap regret bandit}(No-swap-regret with bandit-feedback)
For any loss sequence $l^1,\dots,l^T\in\R^B$ bounded between $[0,1]$, the adversarial bandit oracle satisfies that with probability at least $1-\delta$, for all $t\leq T$, 
$$\max_{\psi\in\Psi}\sum_{i=1}^t\Sp{\inner{p_i,l_i}-\inner{\psi\diamond p_i,l_i}}\leq\mathrm{BSwapReg}(t):=O\Sp{B\sqrt{t}\log(Bt/\delta)}. $$
where $\Psi$ denotes the set $\{\psi:\B\rightarrow \B\}$ which consist of all possible strategy modifications. 
\end{assumption}

Here we emphasize several major differences between Algorithm \ref{algo:tabular} and Algorithm \ref{algo}. The choice of input parameters and the proofs are deferred to Appendix \ref{apx:tabular}.
\begin{enumerate}
    \item  
    % \kz{people can be confused: even in linear Markov games, it is ``fully observable'', right? so the state-space to be explored (and state-visitation-information)  should also be shared?}. 
    The states and actions are no longer entangled through the feature map as in independent linear Markov games. As a result, we can use the adversarial bandit oracle to explore {\it individual action space} while using policy cover to explore the {\it shared state space}. Then there will be no inner loop for estimating the full-information feedback and saving $\widetilde{O}(\epsilon^{-2})$ factors.
    \item For independent linear Markov games, each player has its own feature space so that the exploration progress is different and communication is required to synchronize. However, in tabular Markov games, all the players explore in the shared state space, which means the exploration progress is inherently synchronous and no communication is required. The triggering event is that whenever a state visitation is approximately doubled, the policy cover will update, which guarantees that with high probability, the number of episodes is bounded by $\widetilde{O}(HS)$.
\end{enumerate}

\begin{restatable}{theorem}{tabularCCE}\label{thm:tabular CCE}
Suppose Algorithm \ref{algo:tabular} is instantiated with adversarial multi-armed bandit oracles satisfying Assumption \ref{asp:no regret bandit}. Then for tabular Markov games, with probability at least $1-\delta$, Algorithm \ref{algo:tabular} will output an $\epsilon$-approximate Markov CCE. The sample complexity is $\widetilde{O}(HK_{\max}N_{\max})=\widetilde{O}(H^6S^2A_{\max}\epsilon^{-2})$. 
\end{restatable}

\begin{restatable}{theorem}{tabularCE}\label{thm:tabular CE}
Suppose Algorithm \ref{algo:tabular} is instantiated with adversarial multi-armed bandit oracles satisfying Assumption \ref{asp:no swap regret bandit}. Then for tabular Markov games, with probability at least $1-\delta$, Algorithm \ref{algo:tabular} will output an $\epsilon$-approximate Markov CE. The sample complexity is $\widetilde{O}(HK_{\max}N_{\max})=\widetilde{O}(H^6S^2A_{\max}^2\epsilon^{-2})$. 
\end{restatable}

\section{Conclusion}
In this paper, we propose the independent function approximation model for Markov games and provide algorithms for different types of Markov games that can break the curse of multiagents in a large state  space. We hope this work can serve as the first step towards understanding the empirical success of MARL with independent function approximation. Below we list some interesting open problems for future research. 

\begin{enumerate}
    \item Sharpen the sample complexity. The sample complexity for independent linear sample complexity is far from optimal. For example, it would be a significant improvement if the dependence on $\epsilon$ could be improved to the optimal rate of $\widetilde{O}(\epsilon^{-2})$. 
    \item Incorporate general function approximation. We study independent linear function approximation as an initial attempt. There is a huge body of general function approximation results for single-agent RL and it would be interesting to study them in the context of independent function approximation for Markov games. 
    \item Different data collection oracles. In this work, we study the online setting where exploration is necessary. It would be interesting to extend our results to other settings, such as the offline setting or the simulator setting where specific new challenges might occur or the tightest sample complexity is preferred. 
\end{enumerate}

\section*{Acknowledgements} 
We sincerely thank Yifang Chen, Kevin Jamieson and Andrew Wagenmaker for the discussion on adversarial linear bandits with changing action set.

\bibliography{reference}
\bibliographystyle{plainnat}

\newpage

\appendix

\section{Properties of Independent Linear Markov Games}\label{apx:linear MG}

\abstraction*

\begin{proof}
For all player $i$, we will let $d_i=|\mathcal{Z}|A_i$ and $\phi_i(s,a_i)=e_{(\psi(s),a_i)}$ be the canonical basis in $\R^{d_i}$. For any policy $\pi\in\Pi^{\mathrm{estimate}}$, by the definition of $\theta_h^{\overline{\pi},\pi_{-i},V}$ (See Equation \eqref{eq:theta}), we have
$$\theta_h^{\overline{\pi},\pi_{-i},V}(z,a_i)=\frac{\sum_{s:\psi(s)=z}d_h^{\overline{\pi}}(s)Q_{h,i}^{\pi_{-i},V}(s,a_i)}{\sum_{s:\psi(s)=z}d_h^{\overline{\pi}}(s)}\in[0,H+1-h],$$
where $d_h^{\overline{\pi}}(\cdot)$ is the distribution over $\S$ induced by following policy $\overline{\pi}$ till step $h$. 
Thus we have
\begin{align*}
    &\proj_{[0,H+1-h]}\Sp{\inner{\phi_{i}(s_h,a_{h,i}),\theta_h^{\overline{\pi},\pi_{-i},V}}}-Q_{h,i}^{\pi_{-i},V}(s_h,a_{h,i})\\
    =&\proj_{[0,H+1-h]}\Sp{\theta_h^{\overline{\pi},\pi_{-i},V}(\psi(s_h),a_{h,i})}-Q_{h,i}^{\pi_{-i},V}(s_h,a_{h,i})\\
    =&\theta_h^{\overline{\pi},\pi_{-i},V}(\psi(s_h),a_{h,i})-Q_{h,i}^{\pi_{-i},V}(s_h,a_{h,i})\\
    =&\frac{\sum_{s:\psi(s)=\psi(s_h)}d_h^{\overline{\pi}}(s)\Sp{Q_{h,i}^{\pi_{-i},V}(s,a_{h,i})-Q_{h,i}^{\pi_{-i},V}(s_h,a_{h,i})}}{\sum_{s:\psi(s)=\psi(s_h)}d_h^{\overline{\pi}}(s)}.
\end{align*}
On the other hand, for any $z=\psi(s_h)=\psi(s'_h)$, $i\in[m]$, $h\in[H]$, $V\in\mathcal{V}$ and $\pi\in\Pi^{\mathrm{estimate}}$, we have
\begin{align*}
    &\abs{Q_{h,i}^{\pi_{-i},V}(s_h,a_{h,i})-Q_{h,i}^{\pi_{-i},V}(s'_h,a_{h,i})}\\
    =&\abs{\E_{a_{h,-i}\sim\pi_{h,-i}(\cdot\mid s_h)}\Mp{r_{h,i}(s_h,a_{h,i},a_{h,-i})+V_{h+1}(s_{h+1})}-\E_{a_{h,-i}\sim\pi_{h,-i}(\cdot\mid s'_h)}\Mp{r_{h,i}(s'_h,a_{h,i},a_{h,-i})+V_{h+1}(s_{h+1})}}\\
    \leq&\E_{a_{h,-i}\sim\pi_{h,-i}(\cdot\mid s_h)}\Mp{\abs{r_{h,i}(s_h,\a_{h,i})-r_{h,i}(s'_h,\a_{h,i})}+\abs{\E_{s_{h+1}\sim \P_h(\cdot\mid s_h,\a_{h,i})}\Mp{V_{h+1}(s_{h+1})}-\E_{s_{h+1}\sim \P_h(\cdot\mid s'_h,\a_{h,i})}\Mp{V_{h+1}(s_{h+1})}}}\tag{For $\pi\in\Pi^{\mathrm{estimate}}$, we have $\pi_{h,-i}(\cdot\mid s_h)=\pi_{h,-i}(\cdot\mid s'_h)$}\\
    \leq&\E_{a_{h,-i}\sim\pi_{h,-i}(\cdot\mid s_h)}\Mp{\epsilon_h(z)+\abs{\sum_{s_{h+1}\in\S}\Sp{\P_h(s_{h+1}\mid s_h,\a_{h,i})-\P_h(s_{h+1}\mid s'_h,\a_{h,i})}V_{h+1}(s_{h+1})}}\\
    \leq&\E_{a_{h,-i}\sim\pi_{h,-i}(\cdot\mid s_h)}\Mp{\epsilon_h(z)+(H-h)\epsilon_h(z)}\\
    =&(H-h+1)\epsilon_h(z).
\end{align*}
Thus we have
\begin{align*}
    &\abs{\sum_{h=1}^H\E_{\widetilde{\pi}}\Mp{\proj_{[0,H+1-h]}\Sp{\inner{\phi_{i}(s_h,a_{h,i}),\theta_h^{\overline{\pi},\pi_{-i},V}}}-Q_{h,i}^{\pi_{-i},V}(s_h,a_{h,i})}}\\
    \leq& \sum_{h=1}^H\E_{\widetilde{\pi}}\Mp{\abs{\proj_{[0,H+1-h]}\Sp{\inner{\phi_{i}(s_h,a_{h,i}),\theta_h^{\overline{\pi},\pi_{-i},V}}}-Q_{h,i}^{\pi_{-i},V}(s_h,a_{h,i})}}\\
    =&\sum_{h=1}^H\E_{\widetilde{\pi}}\Mp{\abs{\frac{\sum_{s:\psi(s)=\psi(s_h)}d_h^{\overline{\pi}}(s)\Sp{Q_{h,i}^{\pi_{-i},V}(s,a_{h,i})-Q_{h,i}^{\pi_{-i},V}(s_h,a_{h,i})}}{\sum_{s:\psi(s)=z}d_h^{\overline{\pi}}(s)}}}\\
    \leq& \sum_{h=1}^H\E_{\widetilde{\pi}}\Mp{(H-h+1)\epsilon_h(\psi(s_h))}\\
    \leq& H\nu,
\end{align*}
where the last inequality is by the definition of $\nu$-misspecified state abstraction Markov games. 
\end{proof}

\congestion*

\begin{proof}
As $S=1$ and $H=1$, we will ignore $s$ and $h$ in the notation. For all player $i$ and action $a_i\in\A_i$, we set $\phi_i(a_i)\in\{0,1\}^F$ such that
$$[\phi_i(a_i)]_f=\begin{cases}
1,&\forall f\in a_i\\
0,&\forall f\notin a_i. 
\end{cases}$$
We only need to construct $\theta_{i}^{\pi_{-i}}$ such that $\Norm{\theta_{i}^{\pi_{-i}}}\leq \sqrt{F}$ and $\inner{\phi_i(a_i),\theta_{i}^{\pi_{-i}}}=\E_{a_{-i}\sim\pi_{-i}}\Mp{R_i(\a)}\in[0,1]$ for all policy $\pi$ and then we will have
$$\E_{a_i\sim\widetilde{\pi}_i}\Mp{\proj_{[0,1]}\inner{\phi_i(a_i),\theta_{i}^{\pi_{-i}}}-\E_{a_{-i}\sim\pi_{-i}}\Mp{R_i(\a)}}=0$$
for all $\widetilde{\pi}$. 

For any player $i$ and product policy $\pi_{-i}$, we can set
$$\Mp{\theta_{i}^{\pi_{-i}}}_f=\E_{a_{-i}\sim\pi_{-i}}\Mp{R^f(n^f(a_{-i})+1)},\forall f\in\F,$$
where we use $n^f(a_{-i})$ to denote the number of players except $i$ using facility $f$. As each element in $\theta_{i}^{\pi_{-i}}$ is bounded between $[0,1]$, we have $\Norm{\theta_{i}^{\pi_{-i}}}\leq \sqrt{F}$. In addition, we have
$$\inner{\phi_i(a_i),\theta_{i}^{\pi_{-i}}}=\E_{a_{-i}\sim\pi_{-i}}\Mp{\sum_{f\in a_i}(R^f(n^f(a_{-i})+1))}=\E_{a_{-i}\sim\pi_{-i}}\Mp{\sum_{f\in a_i}R^f(n^f(\a))}=\E_{a_{-i}\sim\pi_{-i}}\Mp{R_i(\a)},$$
which concludes the proof. 
\end{proof}

\begin{comment}
\begin{lemma}
There exists $\ott_{h,i}^{k,t}$ and $\utt_{h,i}^{k,t}$ such that for all $(s,a_i)$ we have
$$\inner{\phi_i(s,a_i),\ott_{h,i}^{k,t}}=\E_{a_{-i}\sim\pi^{k,t}_{h,-i}}\Mp{r_{h,i}(s,\a)+\oV_{h+1,i}^k(s')},$$
$$\inner{\phi_i(s,a_i),\utt_{h,i}^{k,t}}=\E_{a_{-i}\sim\pi^{k,t}_{h,-i}}\Mp{r_{h,i}(s,\a)+\uV_{h+1,i}^k(s')}.$$
\end{lemma}

\begin{proof}
We only prove the first argument and the second holds similarly. Note that $\pi_{h,-i}^{k,t}$ is a Markov product strategy. Suppose Assumption \ref{asp:linear MG} is satisfied for $\pi_{h,-i}^{k,t}$ with parameter $\theta_{h,i}$ and $\mu_{h,i}(\cdot)$. We have
\begin{align*}
    \E_{a_{-i}\sim\pi^{k,t}_{h,-i}}\Mp{r(s,\a)+\oV_{h+1,i}^k(s')}=&\E_{a_{-i}\sim\pi^{k,t}_{h,-i}}\Mp{r_{h,i}(s,\a)}+\E_{a_{-i}\sim\pi^{k,t}_{h,-i}}\Mp{\oV_{h+1,i}^k(s')}\\
    =&\inner{\phi_i(s,a_i),\theta_{h,i}}+\E_{a_{-i}\sim\pi^{k,t}_{h,-i}}\Mp{\sum_{s'\in\S}\P_h(s'|s,\a)\oV_{h+1,i}^k(s')}\\
    =&\inner{\phi_i(s,a_i),\theta_{h,i}}+\sum_{s'\in\S}\E_{a_{-i}\sim\pi^{k,t}_{h,-i}}\Mp{\P_h(s'|s,\a)}\oV_{h+1,i}^k(s')\\
    =&\inner{\phi_i(s,a_i),\theta_{h,i}}+\sum_{s'\in\S}\inner{\phi_i(s,a_i),\mu_{h,i}(s')}\oV_{h+1,i}^k(s')\\
    =&\inner{\phi_i(s,a_i),\theta_{h,i}+\sum_{s'\in\S}\mu_{h,i}(s')\oV_{h+1,i}^k(s')}. 
\end{align*}
Then we can prove the argument by setting $\ott_{h,i}^{k,t}=\theta_{h,i}+\sum_{s'\in\S}\mu_{h,i}(s')\oV_{h+1,i}^k(s')$. 
\end{proof}
\end{comment}

\section{Proofs for Section \ref{sec:linear MG Alg}}\label{apx:linear MG proof}

We will set the parameters for Algorithm \ref{algo} to be
\begin{itemize}
    \item $\lambda=\frac{2\log(16d_{\max}mNHT/\delta)}{\log(36/35)}$
    \item $W=H\sqrt{d_{\max}}$
    \item $\beta=16(W+H)\sqrt{\lambda+d_{\max}\log(32WN(W+H))+4\log(8mK_{\max}HT/\delta)}$
    \item $T_{\Trig}=64\log(8mHN^2/\delta)$
    \item $K_{\max}=\min\{\frac{2Hmd_{\max}\log(N+\lambda)}{\log(1+T_\Trig/4)},N\}$
    \item $T=\widetilde{O}(H^4\log(A_{\max})\epsilon^{-2})$ for Markov CCE and $T=\widetilde{O}(H^4A_{\max}\log(A_{\max})\epsilon^{-2})$ for Markov CE
    \item $N=\widetilde{O}(m^2H^4d_{\max}^2\epsilon^{-2})$.
\end{itemize}
We will use subscript $k,t$ to denote the variables in episode $k$ and inner loop $t$, and subscript $h,i$ to denote the variables at step $h$ and for player $i$. We will use $K$ to denote the episode that the Algorithm \ref{algo} ends ($n^\mathrm{tot}=N$ or $K=K_{\max}$) . Immediately we have $K\leq K_{\max}\leq N$. 

By the definition of the no-regret learning oracle (Assumption \ref{asp:no regret} and Assumption \ref{asp:no swap regret}), we have the following two lemmas. 

\begin{lemma}\label{lemma:no regret}
Suppose Algorithm \ref{algo} is instantiated with no-regret learning oracles satisfying Assumption \ref{asp:no regret}. For all $k\in[K]$, $t\in[T]$, $h\in[H]$, $i\in[m]$ and $s\in\S$ we have
$$\frac{1}{T}\sum_{t=1}^T\sum_{a_i\in\A_i}\pi_{h,i}^{k,t}(a_i\mid s)\oQ_{h,i}^{k,t}(s,a_i)\geq \max_{a_i\in\A_i}\frac{1}{T}\sum_{t=1}^T\oQ_{h,i}^{k,t}(s,a_i)-\frac{H}{T}\cdot\mathrm{Reg}(T).$$
\end{lemma}

\begin{lemma}\label{lemma:no swap regret}
Suppose Algorithm \ref{algo} is instantiated with no-regret learning oracles satisfying Assumption \ref{asp:no swap regret}. For all $k\in[K]$, $t\in[T]$, $h\in[H]$, $i\in[m]$ and $s\in\S$ we have
$$\frac{1}{T}\sum_{t=1}^T\sum_{a_i\in\A_i}\pi_{h,i}^{k,t}(a_i\mid s)\oQ_{h,i}^{k,t}(s,a_i)\geq \max_{\psi_i\in\Psi_i}\frac{1}{T}\sum_{t=1}^T\sum_{a_i\in\A_i}\pi_{h,i}^{k,t}(a_i\mid s)\oQ_{h,i}^{k,t}(s,\psi_h(a_i\mid s))-\frac{H}{T}\cdot\mathrm{SwapReg}(T).$$
\end{lemma}

\subsection{Concentration}

The population covariance matrix for episode $k$, inner loop $t$, 
% \kz{the population covariance below on the left does not depend on $t$, right? or maybe we shall say they are the same for all $t$.}\kz{similarly for the last-term $\sum_{l=1}^{k-1}n^l\Sigma_{h,i}^{\pi^l}$, there is no $t$}, 
step $h$ and player $i$ is defined as
$$\Sigma_{h,i}^{k}:=\E\Mp{\Sigma_{h,i}^{k,t}}=\lambda I+\sum_{l=1}^{k-1}n^l\Sigma_{h,i}^{\pi^l},$$
where   $\Sigma_{h,i}^{\pi^k}=\E_{\pi^k}\Mp{\phi_i(s_h,a_{h,i})\phi_i(s_h,a_{h,i})^\top}$. Note that $s_h^l,a_{h,i}^l$ is sampled following the same policy for each inner loop $t$, so the expected covariance is the same for different $t$. 

%By Proposition \ref{prop:linear completeness}, there exists $\ott_{h,i}^{k,t}$ and $\utt_{h,i}^{k,t}$ such that for all $k\in[K]$, $t\in[T]$, $h\in[H]$, $i\in[m]$, $s\in\S$ and $a_i\in\A_i$ we have 
%$$\inner{\phi_i(s,a_i),\ott_{h,i}^{k,t}}=\E_{a_{-i}\sim\pi_{h,-i}^{k,t}(\cdot\mid s)}\Mp{r_{h,i}(s,\a)+\oV_{h+1,i}^k(s')},$$
%$$\inner{\phi_i(s,a_i),\utt_{h,i}^{k,t}}=\E_{a_{-i}\sim\pi_{h,-i}^{k,t}(\cdot\mid s)}\Mp{r_{h,i}(s,\a)+\uV_{h+1,i}^k(s')}. $$

We define $\pi^{k,\mathrm{cov}}$ to be the mixture policy of the policy cover $\Pi^k$, where policy $\pi^l$ is given weight/probability $\frac{n^l}{\sum_{j=1}^{k-1}n^j}$. Then we define the on-policy population fit to be 
$$\ott_{h,i}^{k,t}:=\argmin_{\Norm{\theta}\leq W}\E_{(s_h,a_{h,i})\sim \pi^{k,\mathrm{cov}}}\Bp{\inner{\phi_i(s_h,a_{h,i}),\theta}-\E_{a_{h,-i}\sim\pi_{h,-i}^{k,t}(\cdot\mid s)}\Mp{r_{h,i}(s_h,\a_h)+\oV_{h+1,i}^k(s')}}^2,$$
$$\utt_{h,i}^{k,t}:=\argmin_{\Norm{\theta}\leq W}\E_{(s_h,a_{h,i})\sim \pi^{k,\mathrm{cov}}}\Bp{\inner{\phi_i(s_h,a_{h,i}),\theta}-\E_{a_{h,-i}\sim\pi_{h,-i}^{k,t}(\cdot\mid s)}\Mp{r_{h,i}(s_h,\a_h)+\uV_{h+1,i}^k(s')}}^2.$$

\begin{lemma}\label{lemma:concentration 1}
(Concentration) With probability at least $1-\delta/2$, for all $k\in[K]$, $h\in[H]$, $t\in[T]$, $i\in[m]$, we have
\begin{equation}\label{eq:opt concentration}
    \Norm{\ot_{h,i}^{k,t}-\ott_{h,i}^{k,t}}_{\Sigma_{h,i}^k}\leq 8(W+H)\sqrt{\lambda+d_i\log(32WN(W+H))+4\log(8mK_{\max}HT/\delta)}\leq\beta/2,
\end{equation}
\begin{equation}\label{eq:pes concentration}
    \Norm{\ut_{h,i}^{k,t}-\utt_{h,i}^{k,t}}_{\Sigma_{h,i}^k}\leq 8(W+H)\sqrt{\lambda+d_i\log(32WN(W+H))+4\log(8mK_{\max}HT/\delta)}\leq\beta/2,
\end{equation}
\begin{equation}\label{eq:cov concentration}
    \frac{1}{2}\Sigma_{h,i}^{k,t}\preceq \Sigma_{h,i}^k\preceq \frac{3}{2}\Sigma_{h,i}^{k,t}.  
\end{equation}
\end{lemma}

\begin{proof}
By applying Lemma \ref{lemma:constrained LS} with $Y_{\max}=H$ and union bound, \eqref{eq:opt concentration} 
% \kz{seems a repeated label. lets check all.} 
and \eqref{eq:pes concentration} holds with probability at least $1-\delta/4$. For \eqref{eq:cov concentration}, we can prove it holds with probability at least $1-\delta/4$ by applying Lemma \ref{lemma:cov concentration} with $\lambda>\frac{2\log(16d_imK_{\max}HT/\delta)}{\log(36/35)}$ and union bound. 
\end{proof}

\begin{lemma}\label{lemma:concentration 2}
With probability at least $1-\delta/2$, the following two events hold:
\begin{itemize}
    \item Suppose at episode 
    % \kz{episode or round? we seem to have used round for $t$ somewhere. lets double check all.} 
    $k$, Line \ref{line:trigger}: $T_{h,i}\geq T_\Trig$ is triggered, then we have
    $$\E_{\pi^{k}}\Norm{\phi_i(s_h,a_{h,i})}^2_{\Mp{\Sigma_{h,i}^{k,1}}^{-1}}\geq \frac{1}{2n^k}\sum_{j=1}^{n^k}\Norm{\phi_i(s_h^{k,j},a_{h,i}^{k,j})}^2_{\Mp{\Sigma_{h,i}^{k,1}}^{-1}}\geq \frac{T_\Trig}{2n^k},$$
    where $j$ denotes the $j$-th trajectory collected in the policy cover update (Line \ref{line:repeat}). 
    \item For any $k\in[K_{\max}]$, $h\in[H]$, $i\in[m]$, we have
    $$\E_{\pi^{k}}\Norm{\phi_i(s_h,a_{h,i})}^2_{\Mp{\Sigma_{h,i}^{k,1}}^{-1}}\leq \frac{2T_\Trig}{n^k}. $$
\end{itemize}
\end{lemma}

\begin{proof}
Note that if at episode $k$, $T_{h,i}\geq T_\Trig$ is triggered, we will have $n^k\leq N$ as otherwise $n^\tot=N$ will be triggered. By Lemma \ref{lemma:trig concentration} with $X_j=\Norm{\phi_i(s_h^{k,j},a_{h,i}^{k,j})}_{\Mp{\Sigma_{h,i}^{1,k}}^{-1}}$, $n_{\max}=N$ and $T_\Trig\geq 64\log(8mHK_{\max}N/\delta)$, we have that the argument holds with probability at least $1-\delta/(2mK_{\max}H)$ for any fixed $k\in[K_{\max}]$, $h\in[H]$ and $i\in[m]$.  Then we can prove the lemma by applying union bound.
\end{proof}

We denote $\G$ to be the good event where the arguments in Lemma \ref{lemma:concentration 1} and Lemma \ref{lemma:concentration 2} hold, which is with probability at least $1-\delta$ by Lemma \ref{lemma:concentration 1} and Lemma \ref{lemma:concentration 2}.

We define the misspecification error to be
$$\overline{\Delta}_{h,i}^{k,t}(s,a_i):=\E_{a_{-i}\sim\pi_{h,-i}^{k,t}(\cdot\mid s)}\Mp{r_{h,i}(s,\a)+\oV_{h+1,i}^k(s')}-\proj_{[0,H+1-h]}\Sp{\inner{\phi_i(s,a_i),\ott_{h,i}^{k,t}}},$$
$$\underline{\Delta}_{h,i}^{k,t}(s,a_i):=\E_{a_{-i}\sim\pi_{h,-i}^{k,t}(\cdot\mid s)}\Mp{r_{h,i}(s,\a)+\uV_{h+1,i}^k(s')}-\proj_{[0,H+1-h]}\Sp{\inner{\phi_i(s,a_i),\utt_{h,i}^{k,t}}}.$$
Then by the definition of $\nu$-misspecified linear Markov games, we have the following lemma. 

\begin{lemma}\label{lemma:misspecification}
For any policy $\pi$, we have
$$\abs{\sum_{h=1}^H\E_{\pi}\Mp{\overline{\Delta}_{h,i}^{k,t}(s,a_i)}}\leq\nu,\qquad \abs{\sum_{h=1}^H\E_{\pi}\Mp{\underline{\Delta}_{h,i}^{k,t}(s,a_i)}}\leq\nu.$$

\end{lemma}

\subsection{Proofs for Markov CCE}

\begin{lemma}\label{lemma:Q error}
 Under the good event $\G$, for all $k\in[K]$, $t\in[T]$, $h\in[H]$, $i\in[m]$, $s\in\S$ and $a_i\in\A_i$ we have 
$$-\overline{\Delta}_{h,i}^{k,t}(s,a_i)\leq \oQ_{h,i}^{k,t}(s,a_i)-\Mp{\E_{a_{-i}\sim\pi_{h,-i}^{k,t}(\cdot\mid s)}\Mp{r_{h,i}(s,\a)+\oV_{h+1,i}^k(s')}}\leq3\beta\Norm{\phi_i(s,a_i)}_{[\Sigma_{h,i}^k]^{-1}}-\overline{\Delta}_{h,i}^{k,t}(s,a_i),$$
$$-3\beta\Norm{\phi_i(s,a_i)}_{[\Sigma_{h,i}^k]^{-1}}-\underline{\Delta}_{h,i}^{k,t}(s,a_i)\leq \uQ_{h,i}^{k,t}(s,a_i)-\Mp{\E_{a_{-i}\sim\pi_{h,-i}^{k,t}(\cdot\mid s)}\Mp{r_{h,i}(s,\a)+\uV_{h+1,i}^k(s')}}\leq-\underline{\Delta}_{h,i}^{k,t}(s,a_i). $$
\end{lemma}

\begin{proof}
We only prove the first argument and the second one holds similarly. 

By Lemma \ref{lemma:concentration 1}, for any $s\in\S$, $a_i\in\A_i$, $h\in[H]$, $i\in[m]$, $k\in[K]$, we have
\begin{align*}
    \abs{\inner{\phi_i(s,a_i),\ot_{h,i}^{k,t}-\ott_{h,i}^{k,t}}}
    \leq\Norm{\phi_i(s,a_i)}_{[\Sigma_{h,i}^k]^{-1}}\Norm{\ot_{h,i}^{k,t}-\ott_{h,i}^{k,t}}_{\Sigma_{h,i}^k}
    \leq\beta/2\Norm{\phi_i(s,a_i)}_{[\Sigma_{h,i}^k]^{-1}},
\end{align*}
where the first inequality is from Cauchy-Schwarz inequality. 
As a result, we have 
\begin{align*}
    \oQ_{h,i}^{k,t}(s,a_i)=&\proj_{[0,H+1-h]}\Sp{\inner{\phi_i(s,a_i),\ot_{h,i}^{k,t}}+\beta\Norm{\phi_i(s,a_i)}_{[\Sigma_{h,i}^{k,t}]^{-1}}}\\
    \geq&\proj_{[0,H+1-h]}\Sp{\inner{\phi_i(s,a_i),\ot_{h,i}^{k,t}}+\frac{1}{2}\beta\Norm{\phi_i(s,a_i)}_{[\Sigma_{h,i}^{k}]^{-1}}}\tag{Lemma \ref{lemma:concentration 1}}\\
    \geq&\proj_{[0,H+1-h]}\Sp{\inner{\phi_i(s,a_i),\ott_{h,i}^{k,t}}}\\
    =&\E_{a_{-i}\sim\pi_{h,-i}^{k,t}(\cdot\mid s)}\Mp{r_{h,i}(s,\a)+\oV_{h+1,i}^k(s')}-\overline{\Delta}_{h,i}^{k,t}(s,a_i)
\end{align*}
and  
\begin{align*}
    \oQ_{h,i}^{k,t}(s,a_i)=&\proj_{[0,H+1-h]}\Sp{\inner{\phi_i(s,a_i),\ot_{h,i}^{k,t}}+\beta\Norm{\phi_i(s,a_i)}_{[\Sigma_{h,i}^{k,t}]^{-1}}}\\
    \leq&\proj_{[0,H+1-h]}\Sp{\inner{\phi_i(s,a_i),\ot_{h,i}^{k,t}}+2\beta\Norm{\phi_i(s,a_i)}_{[\Sigma_{h,i}^{k}]^{-1}}}\tag{Lemma \ref{lemma:concentration 1}}\\
    \leq&\proj_{[0,H+1-h]}\Sp{\inner{\phi_i(s,a_i),\ott_{h,i}^{k,t}}+3\beta\Norm{\phi_i(s,a_i)}_{[\Sigma_{h,i}^k]^{-1}}}\\
    \leq&\proj_{[0,H+1-h]}\Sp{\inner{\phi_i(s,a_i),\ott_{h,i}^{k,t}}}+3\beta\Norm{\phi_i(s,a_i)}_{[\Sigma_{h,i}^k]^{-1}}\\
    =&\E_{a_{-i}\sim\pi_{h,-i}^{k,t}(\cdot\mid s)}\Mp{r_{h,i}(s,\a)+\oV_{h+1,i}^k(s')}-\overline{\Delta}_{h,i}^{k,t}(s,a_i)+3\beta\Norm{\phi_i(s,a_i)}_{[\Sigma_{h,i}^k]^{-1
    }},
\end{align*}
which concludes the proof. 
\end{proof}

\begin{lemma}\label{lemma:optimism}
(Optimism) Under the good event $\G$, for all $k\in[K]$, $i\in[m]$, we have 
$$\oV_{1,i}^{k}(s_1)\geq V_{1,i}^{\dagger,\pi_{-i}^k}(s_1)-\sum_{h=1}^H\E_{\dagger,\pi_{-i}^k}\Mp{\frac{1}{T}\sum_{t=1}^T\overline{\Delta}_{h,i}^{k,t}(s_h,a_{h,i})}\geq V_{1,i}^{\dagger,\pi_{-i}^k}(s_1)-\nu. $$
\end{lemma}

\begin{proof}
For any $k\in[K]$, $i\in[m]$, under the good event $\G$, we have
\begin{align*}
    &\oV_{1,i}^{k}(s_1)-V_{1,i}^{\dagger,\pi_{-i}^k}(s_1)\\
    =&\proj_{[0,H]}\Sp{\frac{1}{T}\sum_{t=1}^T\sum_{a_i\in\A_i}\pi_{1,i}^{k,t}(a_{1,i}\mid s_1)\oQ_{1,i}^{k,t}(s_1,a_{1,i})+\frac{H}{T}\cdot\mathrm{Reg}(T)}-V_{1,i}^{\dagger,\pi_{-i}^k}(s_1)\\
    \geq&\proj_{[0,H]}\Sp{\max_{a_{1,i}\in\A_i}\frac{1}{T}\sum_{t=1}^T\oQ_{1,i}^{k,t}(s_1,a_{1,i})}-V_{1,i}^{\dagger,\pi_{-i}^k}(s_1)\tag{Lemma \ref{lemma:no regret}}\\
    \geq&\max_{a_{1,i}\in\A_i}\frac{1}{T}\sum_{t=1}^T\Bp{\E_{a_{-i}\sim\pi_{1,-i}^{k,t}(\cdot\mid s_1)}\Mp{r_{1,i}(s,\a)+\oV_{2,i}^k(s')}-\overline{\Delta}_{1,i}^{k,t}(s_1,a_{1,i})}-V_{1,i}^{\dagger,\pi_{-i}^k}(s_1)\tag{Lemma \ref{lemma:Q error}}\\
    \geq&\E_{\dagger,\pi_{-i}^k}\Mp{r_{1,i}(s_1,\a_1)+\oV_{2,i}^k(s')-\frac{1}{T}\sum_{t=1}^T\overline{\Delta}_{1,i}^{k,t}(s_1,a_{1,i})}-V_{1,i}^{\dagger,\pi_{-i}^k}(s_1)\\
    =&\E_{\dagger,\pi_{-i}^k}\Mp{\oV_{2,i}^{k}(s_2)-V_{2,i}^{\dagger,\pi_{-i}^k}(s_2)-\frac{1}{T}\sum_{t=1}^T\overline{\Delta}_{1,i}^{k,t}(s_1,a_{1,i})}\\
    \geq& -\E_{\dagger,\pi_{-i}^k}\Mp{\sum_{h=1}^H\frac{1}{T}\sum_{t=1}^T\overline{\Delta}_{h,i}^{k,t}(s_h,a_{h,i})}\\
    \geq&-\nu,\tag{Lemma \ref{lemma:misspecification}}
\end{align*}
where we use $\E_{\dagger,\pi_{-i}^k}$ to denote $\E_{\pi'_i,\pi_{-i}^k}$ such that $\pi'_i$ is a best response of $\pi_{-i}^k$. 
\end{proof}

\begin{comment}
\begin{proof}
We prove the lemma by mathematical induction on $h$. The argument holds for $h=H+1$ as both sides are 0. 
By the update rule of $\oV_{h,i}^k(s)$, we have
$$\oV_{h,i}^k(s)=\proj_{[0,H+1-h]}\Sp{\frac{1}{T}\sum_{t=1}^T\sum_{a\in\A_i}\pi_{h,i}^{k,t}(a_i\mid s)\oQ_{h,i}^{k,t}(s,a_i)+\frac{H}{T}\cdot\mathrm{Reg}(T)}. $$

By the definition of the no-regret oracle, we have
\begin{align*}
    \oV_{h,i}^k(s)=&\proj_{[0,H+1-h]}\Sp{\frac{1}{T}\sum_{t=1}^T\sum_{a_i\in\A_i}\pi_{h,i}^{k,t}(a_i\mid s)\oQ_{h,i}^{k,t}(s,a_i)+\frac{H}{T}\cdot\mathrm{Reg}(T)}\\
    \geq&\proj_{[0,H+1-h]}\Sp{\max_{a_i\in\A_i}\frac{1}{T}\sum_{t=1}^T\oQ_{h,i}^{k,t}(s,a_i)}\tag{Lemma \ref{lemma:no regret}}\\
    \geq&\proj_{[0,H+1-h]}\Sp{\max_{a_i\in\A_i}\frac{1}{T}\sum_{t=1}^T\E_{a_{-i}\sim\pi_{h,-i}^{k,t}(\cdot\mid s)}\Mp{r_{h,i}(s,\a)+\oV_{h+1,i}^k(s')}}\tag{Lemma \ref{lemma:Q error}}\\
    \geq&\proj_{[0,H+1-h]}\Sp{\max_{a_i\in\A_i}\frac{1}{T}\sum_{t=1}^T\E_{a_{-i}\sim\pi_{h,-i}^{k,t}(\cdot\mid s)}\Mp{r_{h,i}(s,\a)+V_{h+1,i}^{\dagger,\pi_{-i}^k}(s')}}\tag{Induction hypothesis}\\
    =&V_{h,i}^{\dagger,\pi_{-i}^k}(s). 
\end{align*}
Thus by mathematical induction, we can prove the lemma. 
\end{proof}

\end{comment}

\begin{lemma}\label{lemma:pessimism}
(Pessimism) Under the good event $\G$, for all $k\in[K]$, $i\in[m]$, we have
$$\uV_{1,i}^{k}(s_1)\leq V_{1,i}^{\pi^k}(s_1)-\sum_{h=1}^H\E_{\pi^k}\frac{1}{T}\Mp{\sum_{t=1}^T\underline{\Delta}_{h,i}^{k,t}(s_h,a_{h,i})}\leq V_{1,i}^{\pi^k}(s_1)+\nu. $$
\end{lemma}

\begin{proof}
For any $k\in[K]$, $i\in[m]$, under the good event $\G$, we have
\begin{align*}
    &\uV_{1,i}^{k}(s_1)-V_{1,i}^{\pi^k}(s_1)\\
    =&\frac{1}{T}\sum_{t=1}^T\sum_{a\in\A_i}\pi_{1,i}^{k,t}(a_{1,i}\mid s_1)\uQ_{1,i}^{k,t}(s_1,a_{1,i})-V_{1,i}^{\pi^k}(s_1)\\
    \leq&\frac{1}{T}\sum_{t=1}^T\sum_{a_{1,i}\in\A_i}\pi_{1,i}^{k,t}(a_{1,i}\mid s_1)\Mp{\E_{a_{1,-i}\sim\pi_{1,-i}^{k,t}(\cdot\mid s_1)}\Mp{r_{1,i}(s_1,\a_1)+\uV_{2,i}^k(s_2)}-\underline{\Delta}_{1,i}^{k,t}(s_1,a_i)}-V_{1,i}^{\pi^k}(s_1)\tag{Lemma \ref{lemma:Q error}}\\
    =&\E_{\a_1\sim\pi_1^k(\cdot\mid s_1)}\Mp{r_{1,i}(s_1,\a_1)+\uV_{2,i}^k(s_2)-\frac{1}{T}\sum_{t=1}^T\underline{\Delta}_{1,i}^{k,t}(s_1,a_i)}-V_{1,i}^{\pi^k}(s_1)\\
    =&\E_{\a_1\sim\pi_1^k(\cdot\mid s_1)}\Mp{\uV_{2,i}^k(s_2)-V_{2,i}^{\pi^k}(s_2)-\frac{1}{T}\sum_{t=1}^T\underline{\Delta}_{1,i}^{k,t}(s_1,a_i)}\\
    \leq&-\sum_{h=1}^H\E_{\pi^k}\Mp{\frac{1}{T}\sum_{t=1}^T\underline{\Delta}_{h,i}^{k,t}(s_h,a_{h,i})}\\
    \leq&\nu,\tag{Lemma \ref{lemma:misspecification}}
\end{align*}
which concludes the proof. 
\end{proof}

\begin{lemma}\label{lemma:CCE gap}
Under the good event $\G$, for all $k\in[K]$ and $i\in[m]$, we have
$$V^{\dagger,\pi_{-i}^k}_{1,i}(s_1)-V^{\pi^k}_{1,i}(s_1)-2\nu\leq \oV_{1,i}^k(s_1)-\uV_{1,i}^k(s_1)\leq6\beta\E_{\pi^{k}}\sum_{h=1}^H\Norm{\phi_i(s_h,a_{h,i})}_{[\Sigma_{h,i}^k]^{-1}}+\frac{H^2}{T}\cdot\mathrm{Reg}(T)+2\nu. $$
\end{lemma}

\begin{proof}
The first inequality is from Lemma \ref{lemma:optimism} and Lemma \ref{lemma:pessimism}. Now we prove the second argument. Under the good event $\G$, for all $k\in[K]$ and $i\in[m]$, we have
\begin{align*}
    &\oV_{1,i}^k(s_1)-\uV_{1,i}^k(s_1)\\
    \leq&\frac{1}{T}\sum_{t=1}^T\sum_{a_i\in\A_i}\pi_{1,i}^{k,t}(a_{1,i}\mid s_1)\oQ_{1,i}^{k,t}(s_1,a_{1,i})+\frac{H}{T}\cdot\mathrm{Reg}(T)-\frac{1}{T}\sum_{t=1}^T\sum_{a_{1,i}\in\A_i}\pi_{1,i}^{k,t}(a_{1,i}\mid s_1)\uQ_{1,i}^{k,t}(s_1,a_{1,i})\\
    \leq&\frac{1}{T}\sum_{t=1}^T\sum_{a_{1,i}\in\A_i}\pi_{1,i}^{k,t}(a_{1,i}\mid s_1)\Sp{\Mp{\E_{a_{1,-i}\sim\pi_{1,-i}^{k,t}(\cdot\mid s)}\Mp{r_{h,i}(s_1,\a_1)+\oV_{2,i}^k(s_2)}}+3\beta\Norm{\phi_i(s_1,a_{1,i})}_{[\Sigma_{1,i}^k]^{-1}}-\overline{\Delta}_{1,i}^{k,t}(s,a_{1,i})}\\
    &-\frac{1}{T}\sum_{t=1}^T\sum_{a_{1,i}\in\A_i}\pi_{1,i}^{k,t}(a_{1,i}\mid s_1)\Sp{\Mp{\E_{a_{1,-i}\sim\pi_{1,-i}^{k,t}(\cdot\mid s_1)}\Mp{r_{h,i}(s_1,\a_1)+\uV_{2,i}^k(s_2)}}-3\beta\Norm{\phi_i(s_1,a_{1,i})}_{[\Sigma_{1,i}^k]^{-1}}-\underline{\Delta}_{1,i}^{k,t}(s,a_{1,i})}\\
    &+\frac{H}{T}\cdot\mathrm{Reg}(T)\tag{Lemma \ref{lemma:Q error}}\\
    \leq&\frac{1}{T}\sum_{t=1}^T\Mp{\E_{\a_1\sim\pi_{1}^{k,t}(\cdot\mid s_1)}\Mp{\oV_{2,i}^k(s_2)-\uV_{2,i}^k(s_2)}}\\&+\E_{a_{1,i}\sim\pi^{k,t}_{1,i}(\cdot\mid s_1)}\Mp{6\beta\Norm{\phi_i(s_1,a_{1,i})}_{[\Sigma_{1,i}^k]^{-1}}-\frac{1}{T}\sum_{t=1}^T\overline{\Delta}_{1,i}^{k,t}(s_1,a_{1,i})-\frac{1}{T}\sum_{t=1}^T\underline{\Delta}_{1,i}^{k,t}(s_1,a_{1,i})}+\frac{H}{T}\cdot\mathrm{Reg}(T)\\
    =&\E_{\pi_{1}^{k}}\Mp{\oV_{2,i}^k(s_2)-\uV_{2,i}^k(s_2)}+\E_{a_{1,i}\sim\pi^{k,t}_{1,i}(\cdot\mid s_1)}\Mp{6\beta\Norm{\phi_i(s_1,a_{1,i})}_{[\Sigma_{1,i}^k]^{-1}}-\frac{1}{T}\sum_{t=1}^T\overline{\Delta}_{1,i}^{k,t}(s_1,a_{1,i})-\frac{1}{T}\sum_{t=1}^T\underline{\Delta}_{1,i}^{k,t}(s_1,a_{1,i})}\\&+\frac{H}{T}\cdot\mathrm{Reg}(T)\\
    \leq&6\beta\E_{\pi^{k}}\sum_{h=1}^H\Norm{\phi_i(s_h,a_{h,i})}_{[\Sigma_{h,i}^k]^{-1}}-\E_{\pi^k}\sum_{h=1}^H\frac{1}{T}\sum_{t=1}^T\Sp{\overline{\Delta}_{h,i}^{k,t}(s_h,a_{h,i})+\underline{\Delta}_{h,i}^{k,t}(s_h,a_{h,i})}+\frac{H^2}{T}\cdot\mathrm{Reg}(T)\\
    \leq&6\beta\E_{\pi^{k}}\sum_{h=1}^H\Norm{\phi_i(s_h,a_{h,i})}_{[\Sigma_{h,i}^k]^{-1}}+2\nu+\frac{H^2}{T}\cdot\mathrm{Reg}(T),
\end{align*}
which completes the proof. 
\end{proof}

\begin{lemma}\label{lemma:sum of uncertainty}
Under the good event $\G$, for all $i\in[m]$, we have
$$\sum_{k=1}^Kn^k\E_{\pi^{k}}\Norm{\phi_i(s_h,a_{h,i})}^2_{\Mp{\Sigma_{h,i}^k}^{-1}}\leq 4T_\Trig d_i\log\Sp{1+\frac{N}{d_i\lambda}}.$$
\end{lemma}

\begin{proof}
First, by the triggering condition, we have
$$\sum_{j=1}^{n^k}\Norm{\phi_i(s_h^j,a_{h,i}^j)}^2_{\Mp{\Sigma_{h,i}^{k,1}}^{-1}}=\sum_{j=1}^{n^k-1}\Norm{\phi_i(s_h^j,a_{h,i}^j)}^2_{\Mp{\Sigma_{h,i}^{k,1}}^{-1}}+\Norm{\phi_i(s_h^{n^k},a_{h,i}^{n^k})}^2_{\Mp{\Sigma_{h,i}^{k,1}}^{-1}}\leq T_\Trig+1,$$
where $j$ denotes the $j$-th trajectory collected in the policy cover update (Line \ref{line:repeat}). 
By Lemma \ref{lemma:concentration 2}, we have
\begin{align*}
    n^k\E_{\pi^{k}}\Norm{\phi_i(s_h,a_{h,i})}^2_{\Mp{\Sigma_{h,i}^k}^{-1}}\leq& 2n^k\E_{\pi^{k}}\Norm{\phi_i(s_h,a_{h,i})}^2_{\Mp{\Sigma_{h,i}^{k,1}}^{-1}}
    \leq4T_\Trig. 
\end{align*}
Then by Lemma \ref{lemma:information gain}, we have
$$n^k\E_{\pi^{k}}\Norm{\phi_i(s_h,a_{h,i})}^2_{\Mp{\Sigma_{h,i}^k}^{-1}}\leq 4T_\Trig\log\frac{\det(\Sigma_{h,i}^{k+1})}{\det(\Sigma_{h,i}^k)}. $$
Thus we have
\begin{align*}
    \sum_{k=1}^Kn^k\E_{\pi^{k}}\Norm{\phi_i(s_h,a_{h,i})}^2_{\Mp{\Sigma_{h,i}^k}^{-1}}\leq& \sum_{k=1}^K4T_\Trig\log\frac{\det(\Sigma_{h,i}^{k+1})}{\det(\Sigma_{h,i}^k)}\\
    =&4T_\Trig\log\frac{\det(\Sigma_{h,i}^{K+1})}{\det(\Sigma_{h,i}^1)}\\
    \leq& 4T_\Trig\Mp{d_i\log\Sp{\frac{d_i\lambda+N}{d_i}}-d_i\log(\lambda)}\\
    =&4T_\Trig d_i\log\Sp{1+\frac{N}{d_i\lambda}},
\end{align*}
where we utilized the fact that
$$\log\det(\Sigma_{h,i}^{K+1})\leq d_i\log\Sp{\frac{\mathrm{trace}(\Sigma_{h,i}^{K+1})}{d_i}}\leq d_i\log\Sp{\frac{d_i\lambda+N}{d_i}},$$
and complete the proof. 
\end{proof}

\begin{lemma}\label{lemma:CCE regret}
Under the good event $\G$, we have
$$\sum_{k=1}^Kn^k\max_{i\in[m]}\Sp{\oV_{1,i}^k(s_1)-\uV_{1,i}^k(s_1)}\leq6mH\beta\sqrt{4N(T_\Trig+1)d_{\max}\log\Sp{1+\frac{N}{\lambda}}}+\frac{H^2N}{T}\cdot\mathrm{Reg}(T)+2\nu N. $$

\end{lemma}

\begin{proof}
By Lemma \ref{lemma:episode bound}, under the good event $\G$, we have $\sum_{k=1}^Kn^k=n^{\tot}=N$. Thus we have
\begin{align*}
    &\sum_{k=1}^Kn^k\max_{i\in[m]}\Sp{\oV_{1,i}^k(s_1)-\uV_{1,i}^k(s_1)}\\
    \leq& \sum_{k=1}^Kn^k\max_{i\in[m]}\Mp{6\beta\E_{\pi^{k}}\sum_{h=1}^H\Norm{\phi_i(s_h,a_{h,i})}_{[\Sigma_{h,i}^k]^{-1}}}+\frac{H^2}{T}\sum_{k=1}^Kn^k\mathrm{Reg}(T)+2\nu N\tag{Lemma \ref{lemma:CCE gap}}\\
    \leq&6\beta\sum_{i\in[m]}\sum_{h=1}^H\sum_{k=1}^Kn^k\E_{\pi^{k}}\sqrt{\Norm{\phi_i(s_h,a_{h,i})}_{[\Sigma_{h,i}^k]^{-1}}^2}+\frac{H^2N}{T}\cdot\mathrm{Reg}(T)+2\nu N\\
    \leq&6\beta\sum_{i\in[m]}\sum_{h=1}^H\sum_{k=1}^Kn^k\sqrt{\E_{\pi^{k}}\Norm{\phi_i(s_h,a_{h,i})}_{[\Sigma_{h,i}^k]^{-1}}^2}+\frac{H^2N}{T}\cdot\mathrm{Reg}(T)+2\nu N\tag{Concavity of $f(x)=\sqrt{x}$}\\
    \leq&6\beta\sum_{i\in[m]}\sum_{h=1}^H\sqrt{\sum_{k=1}^Kn^k}\sqrt{\sum_{k=1}^Kn^k\E_{\pi^{k}}\Norm{\phi_i(s_h,a_{h,i})}_{[\Sigma_{h,i}^k]^{-1}}^2}+\frac{H^2N}{T}\cdot\mathrm{Reg}(T)+2\nu N \tag{Cauchy–Schwarz inequality}\\
    \leq&6\beta\sum_{i\in[m]}\sum_{h=1}^H\sqrt{N4(T_\Trig+1)d_i\log\Sp{1+\frac{N}{d_i\lambda}}}+\frac{H^2N}{T}\cdot\mathrm{Reg}(T)+2\nu N\tag{Lemma \ref{lemma:sum of uncertainty}}\\
    \leq&6\beta mH\sqrt{N4(T_\Trig+1)d_{\max}\log\Sp{1+\frac{N}{\lambda}}}+\frac{H^2N}{T}\cdot\mathrm{Reg}(T)+2\nu N.
\end{align*}
\end{proof}

\begin{lemma}\label{lemma:episode bound}
Under the good event $\G$, we have
$$K\leq \frac{2Hmd_{\max}\log(N+\lambda)}{\log(1+T_\Trig/4)},$$
which means $K<K_{\max}$ and Algorithm \ref{algo} ends due to Line \ref{line:trigger} ($n^{\mathrm{tot}}=N_{\max}$). 
\end{lemma}

\begin{proof}
By Lemma \ref{lemma:concentration 2}, for any player $i$ and $h\in[H]$, whenever $T_{h,i}^k\geq T_\Trig$ is triggered, with probability at least $1-\delta$ we have
\begin{align*}
    n^k\E_{\pi^{k}}\Norm{\phi_i(s_h,a_{h,i})}^2_{\Mp{\Sigma_{h,i}^k}^{-1}}\geq& \frac{1}{2}n^k\E_{\pi^{k}}\Norm{\phi_i(s_h,a_{h,i})}^2_{\Mp{\Sigma_{h,i}^{k,1}}^{-1}}\tag{Lemma \ref{lemma:concentration 1}}\\
    \geq& \frac{1}{4}\sum_{j=1}^{n^k}\Norm{\phi_i(s_h^j,a_{h,i}^j)}^2_{\Mp{\Sigma_{h,i}^{k,1}}^{-1}} \tag{Lemma \ref{lemma:concentration 2}}\\
    \geq& \frac{T_\Trig}{4}. 
\end{align*}
Then by Lemma \ref{lemma:information gain}, we have
$$\frac{\det(\Sigma_{h,i}^{k+1})}{\det(\Sigma_{h,i}^k)}\geq 1+n^k\E_{\pi^{k}}\Norm{\phi_i(s_h,a_{h,i})}^2_{\Mp{\Sigma_{h,i}^k}^{-1}}\geq 1+\frac{T_\Trig}{4}. $$
Suppose $s_{h,i}$ is the number of triggering $T_{h,i}^k\geq T_\Trig$ at level $h$ and player $i$, then we have
$$\frac{\det(\Sigma_{h,i}^{K+1})}{\det(\Sigma_{h,i}^1)}\geq\Sp{1+\frac{T_\Trig}{4}}^{s_{h,i}}. $$
In addition, we have
$$\log(\det(\Sigma_{h,i}^1))=d_i\log(\lambda),\log\det((\Sigma_{h,i}^{K+1}))\leq d_i\log\Sp{\frac{\mathrm{trace}(\Sigma_{h,i}^{K+1})}{d_i}}\leq d_i\log\Sp{\frac{d_i\lambda+N}{d_i}},$$ 
which gives
$$s_{h,i}\leq\frac{d_i\log(N/d_i+\lambda)}{\log(1+T_\Trig/4)}. $$
Thus, the total number of triggering is bounded by 
$$\sum_{i\in[m]}\sum_{h\in[H]}s_{h,i}+1\leq \frac{2mHd_{\max}\log(N+\lambda)}{\log(1+T_\Trig/4)},$$
where the additional $1$ is from the event $n^{\tot}=N$. 
\end{proof}

\CCE*

\begin{proof}
Under the good event $\G$, by Lemma $\ref{lemma:episode bound}$, the algorithm ends by $n^{\tot}=N$. By Lemma \ref{lemma:CCE regret}, under the good event $\G$, which happens with probability at least $1-\delta$ (Lemma \ref{lemma:concentration 1} and Lemma \ref{lemma:concentration 2}), we have
\begin{align*}
&\min_{k\in[K]}\max_{i\in[m]}\Sp{\oV_{1,i}^k(s_1)-\uV_{1,i}^k(s_1)}\\
\leq&\frac{1}{N}\sum_{k=1}^Kn^k\max_{i\in[m]}\Sp{\oV_{1,i}^k(s_1)-\uV_{1,i}^k(s_1)}\\
    \leq&6mH\beta\sqrt{4(T_\Trig+1)d_{\max}\log\Sp{1+\frac{N}{\lambda}}/N}+\frac{H^2}{T}\cdot\mathrm{Reg}(T)+2\nu. 
\end{align*}
By setting $N=\widetilde{O}(m^2H^4d_{\max}^3\epsilon^{-2})$ and $T=\widetilde{O}(H^4\log(A_{\max})\epsilon^{-2})$, we can have
$$\min_{k\in[K]}\max_{i\in[m]}\Sp{\oV_{1,i}^k(s_1)-\uV_{1,i}^k(s_1)}\leq\epsilon+2\nu.$$
Then by Lemma \ref{lemma:CCE gap} we have 
\begin{align*}
    \max_{i\in[m]}\Sp{V^{\dagger,\pi_{-i}^{\mathrm{output}}}_{1,i}(s_1)-V^{\pi^{\mathrm{output}}}_{1,i}(s_1)}\leq&\max_{i\in[m]}\Sp{\oV_{1,i}^{k^{\mathrm{output}}}(s_1)-\uV_{1,i}^{k^{\mathrm{output}}}(s_1)}+2\nu\\
    =&\min_{k\in[K]}\max_{i\in[m]}\Sp{\oV_{1,i}^k(s_1)-\uV_{1,i}^k(s_1)}+2\nu\\
    \leq&\epsilon+4\nu,
\end{align*}
which completes the proof. 
\end{proof}

\subsection{Proofs for Markov CE}

\begin{lemma}\label{lemma:optimism swap}
(Optimism) Let $\psi_i^k=\argmax_{\psi_i}V_{1,i}^{\psi\diamond\pi^k}(s_1)$ for all $k\in[K]$ and $i\in[m]$. Under the good event $\G$, for all $k\in[K]$ and $i\in[m]$, we have
$$\oV_{1,i}^{k}(s_1)\geq \max_{\psi_i}V_{1,i}^{\psi_i\diamond\pi^k}(s_1)-\sum_{h=1}^H\E_{\psi_i^k\diamond\pi^k}\Mp{\frac{1}{T}\sum_{t=1}^T\overline{\Delta}_{h,i}^{k,t}(s_h,a_{h,i})}\geq \max_{\psi_i}V_{1,i}^{\psi_i\diamond\pi^k}(s_1)-\nu. $$
\end{lemma}

\begin{proof}
Under the good event $\G$, for all $k\in[K]$, $h\in[H]$, $i\in[m]$, $s_1\in\S$, we have
\begin{align*}
    &\oV_{1,i}^{k}(s_1)-\max_{\psi_i}V_{1,i}^{\psi_i\diamond\pi^k}(s_1)\\
    =&\proj_{[0,H]}\Sp{\frac{1}{T}\sum_{t=1}^T\sum_{a_i\in\A_i}\pi_{1,i}^{k,t}(a_{1,i}\mid s_1)\oQ_{1,i}^{k,t}(s_1,a_{1,i})+\frac{H}{T}\cdot\mathrm{SwapReg}(T)}-V_{1,i}^{\dagger,\pi_{-i}^k}(s_1)\\
    \geq&\proj_{[0,H]}\Sp{\max_{\psi_{1,i}}\frac{1}{T}\sum_{t=1}^T\sum_{a_i\in\A_i}\pi_{1,i}^{k,t}(a_{1,i}\mid s_1)\oQ_{1,i}^{k,t}(s_1,\psi_1(a_{1,i}\mid s_1))}-V_{1,i}^{\dagger,\pi_{-i}^k}(s_1)\tag{Lemma \ref{lemma:no regret}}\\
    \geq&\max_{\psi_{1,i}}\frac{1}{T}\sum_{t=1}^T\E_{\a_1\sim\psi_{1,i}\diamond\pi_{1}^{k,t}(\cdot\mid s_1)}\Mp{r_{1,i}(s_1,\a_1)+\uV_{2,i}^k(s_2)-\overline{\Delta}_{1,i}^{k,t}(s_1,a_{1,i})}-V_{1,i}^{\dagger,\pi_{-i}^k}(s_1)\tag{Lemma \ref{lemma:Q error}}\\
    \geq&\E_{\psi_{1,i}^k\diamond\pi_{1}^k}\Mp{r_{1,i}(s_1,\a_1)+\oV_{2,i}^k(s')-\frac{1}{T}\sum_{t=1}^T\overline{\Delta}_{1,i}^{k,t}(s_1,a_{1,i})}-V_{1,i}^{\dagger,\pi_{-i}^k}(s_1)\\
    =&\E_{\psi_{1,i}^k\diamond\pi_{1}^k}\Mp{\oV_{2,i}^{k}(s_2)-V_{2,i}^{\dagger,\pi_{-i}^k}(s_2)-\frac{1}{T}\sum_{t=1}^T\overline{\Delta}_{1,i}^{k,t}(s_1,a_{1,i})}\\
    \geq& -\E_{\psi_{i}^k\diamond\pi^k}\Mp{\sum_{h=1}^H\frac{1}{T}\sum_{t=1}^T\overline{\Delta}_{h,i}^{k,t}(s_h,a_{h,i})}\\
    \geq&-\nu,\tag{Lemma \ref{lemma:misspecification}}
\end{align*}
which concludes the proof. 
\end{proof}

\begin{lemma}\label{lemma:CE gap}
Under the good event $\G$, for all $k\in[K]$ and $i\in[m]$, we have
$$\max_{\psi_i}V_{1,i}^{\psi_i\diamond\pi^k}(s_1)-V^{\pi^k}_{1,i}(s_1)-2\nu\leq \oV_{1,i}^k(s_1)-\uV_{1,i}^k(s_1)\leq6\beta\E_{\pi^{k}}\sum_{h=1}^H\Norm{\phi_i(s_h,a_{h,i})}_{[\Sigma_{h,i}^k]^{-1}}+\frac{H^2}{T}\cdot\mathrm{SwapReg}(T)+2\nu. $$
\end{lemma}

\begin{proof}
The first inequality is from Lemma \ref{lemma:optimism swap} and Lemma \ref{lemma:pessimism}. Now we prove the second inequality. Under the good event $\G$, for all $k\in[K]$ and $i\in[m]$, we have
\begin{align*}
    &\oV_{1,i}^k(s_1)-\uV_{1,i}^k(s_1)\\
    \leq&\frac{1}{T}\sum_{t=1}^T\sum_{a_{1,i}\in\A_i}\pi_{1,i}^{k,t}(a_{1,i}\mid s)\oQ_{1,i}^{k,t}(s_1,a_{1,i})+\frac{H}{T}\cdot\mathrm{SwapReg}(T)-\frac{1}{T}\sum_{t=1}^T\sum_{a_{1,i}\in\A_i}\pi_{1,i}^{k,t}(a_{1,i}\mid s_1)\uQ_{1,i}^{k,t}(s_1,a_{1,i})\\
    \leq&\frac{1}{T}\sum_{t=1}^T\sum_{a_{1,i}\in\A_i}\pi_{1,i}^{k,t}(a_{1,i}\mid s_1)\Sp{\Mp{\E_{a_{1,-i}\sim\pi_{1,-i}^{k,t}(\cdot\mid s)}\Mp{r_{1,i}(s_1,\a_1)+\oV_{2,i}^k(s_2)}}+3\beta\Norm{\phi_i(s_1,a_{1,i})}_{[\Sigma_{1,i}^k]^{-1}}-\overline{\Delta}_{1,i}^{k,t}(s_1,a_{1,i})}\\
    &-\frac{1}{T}\sum_{t=1}^T\sum_{a_{1,i}\in\A_i}\pi_{1,i}^{k,t}(a_{1,i}\mid s_1)\Sp{\Mp{\E_{a_{1,-i}\sim\pi_{1,-i}^{k,t}(\cdot\mid s_1)}\Mp{r_{1,i}(s_1,\a_1)+\uV_{2,i}^k(s_2)}}-3\beta\Norm{\phi_i(s_1,a_{1,i})}_{[\Sigma_{1,i}^k]^{-1}}-\underline{\Delta}_{1,i}^{k,t}(s_1,a_{1,i})}\\
    &+\frac{H}{T}\cdot\mathrm{SwapReg}(T)\tag{Lemma \ref{lemma:Q error}}\\
    =&\frac{1}{T}\sum_{t=1}^T\Sp{\E_{\a_1\sim\pi_{1}^{k,t}(\cdot\mid s_1)}\Mp{\oV_{2,i}^k(s_2)-\uV_{2,i}^k(s_2)}+\E_{a_{1,i}\sim\pi^{k,t}_{1,i}(\cdot\mid s_1)}\Mp{6\beta\Norm{\phi_i(s_1,a_{1,i})}_{[\Sigma_{1,i}^k]^{-1}}-\overline{\Delta}_{1,i}^{k,t}(s_1,a_{1,i})-\underline{\Delta}_{1,i}^{k,t}(s_1,a_{1,i})}}\\
    &+\frac{H}{T}\cdot\mathrm{SwapReg}(T)\\
    =&\E_{\pi_{1}^{k}}\Mp{\oV_{2,i}^k(s_2)-\uV_{2,i}^k(s_2)}+\E_{a_{1,i}\sim\pi^{k}_{1,i}}\Mp{6\beta\Norm{\phi_i(s_1,a_{1,i})}_{[\Sigma_{1,i}^k]^{-1}}-\overline{\Delta}_{1,i}^{k,t}(s_1,a_{1,i})-\underline{\Delta}_{1,i}^{k,t}(s_1,a_{1,i})}+\frac{H}{T}\cdot\mathrm{SwapReg}(T)\\
    \leq&6\beta\E_{\pi^{k}}\sum_{h=1}^H\Norm{\phi_i(s_h,a_{h,i})}_{[\Sigma_{h,i}^k]^{-1}}-\E_{\pi^k}\sum_{h=1}^H\Sp{\overline{\Delta}_{h,i}^{k,t}(s_h,a_{h,i})+\underline{\Delta}_{h,i}^{k,t}(s_h,a_{h,i})}+\frac{H^2}{T}\cdot\mathrm{SwapReg}(T)\\
    \leq&6\beta\E_{\pi^{k}}\sum_{h=1}^H\Norm{\phi_i(s_h,a_{h,i})}_{[\Sigma_{h,i}^k]^{-1}}+\frac{H^2}{T}\cdot\mathrm{SwapReg}(T)+2\nu. 
\end{align*}
\end{proof}

\begin{lemma}\label{lemma:CE regret}
Under the good event $\G$, we have
$$\sum_{k=1}^Kn^k\max_{i\in[m]}\Sp{\oV_{1,i}^k(s_1)-\uV_{1,i}^k(s_1)}\leq6mH\beta\sqrt{4N(T_\Trig+1)d_{\max}\log\Sp{1+\frac{N}{\lambda}}}+\frac{H^2N}{T}\cdot\mathrm{SwapReg}(T)+2\nu. $$

\end{lemma}

\begin{proof}
The proof is similar to the proof for Lemma \ref{lemma:CCE regret}, where the only difference is that we replace Lemma \ref{lemma:CCE regret} with Lemma \ref{lemma:CE gap} in the proof. 
\end{proof}

\CE*

\begin{proof}
Under the good event $\G$, by Lemma $\ref{lemma:episode bound}$, the algorithm ends by $n^\tot=N$. By Lemma \ref{lemma:CE regret}, under the good event $\G$, which happens with probability at least $1-\delta$ (Lemma \ref{lemma:concentration 1} and Lemma \ref{lemma:concentration 2}), we have
\begin{align*}
&\min_{k\in[K]}\max_{i\in[m]}\Sp{\oV_{1,i}^k(s_1)-\uV_{1,i}^k(s_1)}\\
\leq&\frac{1}{N}\sum_{k=1}^Kn^k\sum_{i\in[m]}\Sp{\oV_{1,i}^k(s_1)-\uV_{1,i}^k(s_1)}\\
    \leq&6mH\beta\sqrt{4(T_\Trig+1)d_{\max}\log\Sp{1+\frac{N}{\lambda}}/N}+\frac{H^2}{T}\cdot\mathrm{SwapReg}(T)+2\nu. 
\end{align*}
By setting $N=\widetilde{O}(m^2H^4d_{\max}^3\epsilon^{-2})$ and $T=\widetilde{O}(H^4A_{\max}\log(A_{\max})\epsilon^{-2})$, we can have
$$\min_{k\in[K]}\max_{i\in[m]}\Sp{\oV_{1,i}^k(s_1)-\uV_{1,i}^k(s_1)}\leq\epsilon+2\nu.$$
Then by Lemma \ref{lemma:CE gap}, we have
\begin{align*}
    \max_{i\in[m]}\Sp{\max_{\psi_i}V^{\psi_i\diamond\pi^{\mathrm{output}}}_{1,i}(s_1)-V^{\pi^{\mathrm{output}}}_{1,i}(s_1)}\leq& \max_{i\in[m]}\Sp{\oV_{1,i}^{k^{\mathrm{output}}}(s_1)-\uV_{1,i}^{k^{\mathrm{output}}}(s_1)}+2\nu\\
    =&\min_{k\in[K]}\max_{i\in[m]}\Sp{\oV_{1,i}^k(s_1)-\uV_{1,i}^k(s_1)}+2\nu\\
    \leq&\epsilon+4\nu,
\end{align*}
which thus completes the proof. 
\end{proof}

\section{Algorithms for Learning Markov CCE/CE without Communication}\label{apx:w/o com}

% \kz{shall we emphasize a bit more about this section in main text? it seems a bit hidden and unnoticeable there.}\qiwen{Not sure..I feel like this part is not too important}
In this section, we present a communication-free algorithm for independent linear Markov games. The key difference is that we leverage an agile policy cover update scheme, i.e., the policy cover is updated whenever a new $\pi^k$ is learned (Line \ref{line:pc update}), and the policy certification is replaced by a uniform sampling procedure  (Line \ref{line:unif output}). 

\begin{algorithm}[!]
    \caption{Communication-free \textbf{P}olicy \textbf{Re}ply with \textbf{F}ull \textbf{I}nformation Oracle in Independent Linear Markov Games (Communication-free \algoname)}
	\label{algo:no com}
    \begin{algorithmic}[1]
        \State {\bfseries Input:} $\lambda$, $\beta$, $K$, $T$
        \State {\bfseries Initialization:} Policy Cover $\Pi=\emptyset$.
        \For{episode $k=1,2,\dots,K$}
        \State Set $\oV_{H+1,i}^k(\cdot)=\uV_{H+1,i}^k(\cdot)=0$.  
        \For{$h=H,H-1,\dots,1$}\Comment{Retrain policy with the current policy cover}
        \State Initialize $\pi_{h,i}^{1,k}$ to be uniform policy for all player $i$. Initialize $\oV_{h,i}^k(\cdot)=\uV_{h,i}^k(\cdot)=0$.  
        \State Each player $i$ initializes a no-regret learning instance (Protocol \ref{algo:no-regret}) at each state $s\in\S$ and step $h\in[H]$, for which we will use $\textsc{No\_Regret\_Update}_{h,i,s}(\cdot)$ to denote the update. 
        \For{$t=1,2,\dots,T$}
        \For{$i\in[m]$}
        %\STATE Call Algorithm \ref{algo:streaming LS} with policy cover $\Pi$ and opponent strategy $\pi_{h,-i}^{k,t}$ and receive $\ot_{h,i}^{k,t}$, $\ut_{h,i}^{k,t}$ and $[\Sigma_{h,i}^{k,t}]^{-1}$.
        \State Set Dataset $\D_{h,i}^{k,t}=\emptyset$
        \For{$l=1,2,\dots,k-1$}
        \State Draw a joint trajectory $(s_{1}^l,\a_1^l,r_{1,i}^l,\dots,s_h^l,\a_h^l,r_{h,i}^l,s_{h+1}^l)$ from $\pi^l_{1:h-1}\circ\Sp{\pi^l_{h,i},\pi_{h,-i}^{k,t}}$, where $\pi^l$ is the policy learned at episode $l$ stored in policy cover $\Pi$. 
        \State Add $(s_h^l,a_{h,i}^l,r_{h,i}^l,s_{h+1}^l)$ to $\D_{h,i}^{k,t}$. 
        \EndFor
        \State Set $\Sigma_{h,i}^{k,t}=\lambda I+\sum_{(s,a,r,s')\in\D_{h,i}^{k,t}}\phi_i(s,a)\phi_i(s,a)^\top$. 
        \State Set $\ot_{h,i}^{k,t}=\argmin_{\Norm{\theta}\leq H\sqrt{d}}\sum_{(s,a,r,s')\in\D_{h,i}^{k,t}}\Sp{\inner{\phi_i(s,a),\theta}-r-\oV^k_{h+1,i}(s')}^2$.
        %\State Set $\ut_{h,i}^{k,t}=[\Sigma_{h,i}^{k,t}]^{-1}\sum_{(s_h^l,a_{h,i}^l,r_{h,i}^l,s_{h+1}^l)\in\D_{h,i}^{k,t}}\phi_i(s_h^l,a_{h,i}^l)(r_{h,i}^l+\uV^k_{h+1,i}(s_{h+1}^l))$
        \State Set $\oQ_{h,i}^{k,t}(\cdot,\cdot)=\proj_{[0,H+1-h]}\Sp{\inner{\phi_i(\cdot,\cdot),\ot_{h,i}^{k,t}}+\beta\Norm{\phi_i(\cdot,\cdot)}_{[\Sigma_{h,i}^{k,t}]^{-1}}}$.
        %\State Set $\uQ_{h,i}^{k,t}(\cdot,\cdot)=\inner{\phi_i(\cdot,\cdot),\ut_{h,i}^{k,t}}-C_d\Norm{\phi_i(\cdot,\cdot)}_{[\Sigma_{h,i}^{k,t}]^{-1}}$.
        \State Update $\oV_{h,i}^k(s)\leftarrow\frac{t-1}{t}\oV_{h,i}(s)+\frac{1}{t}\sum_{a_i\in\A_i}\pi_{h,i}^{k,t}(a_i|s)\oQ_{h,i}^{k,t}(s,a)$ for all $s\in\S$.
        %\State Update $\uV_{h,i}^k(s)\leftarrow\frac{t-1}{t}\uV_{h,i}(s)+\frac{1}{t}\sum_{a_i\in\A_i}\pi_{h,i}^{k,t}(a_i|s)\uQ_{h,i}^{k,t}(s,a)$.
        \State Update the no-regret learning instance at step $h$ and state $s$: $\pi_{h,i}^{k,t+1}(\cdot\mid s)\leftarrow\textsc{No\_Regret\_Update}_{h,i,s}(1-\oQ_{h,i}^{k,t}(s,\cdot)/H)$ for all $s\in\S$.
        %\STATE Update policy $\pi_{h,i}^{t+1}:\pi_{h,i}^{t+1}(a|s)=\frac{\pi_{h,i}^t(a|s)\exp{\Sp{Q_{h,i}^t(s,a)}}}{\sum_{a\in\A_i}\pi_{h,i}^t(a|s)\exp{\Sp{Q_{h,i}^t(s,a)}}}$
        \EndFor
        \EndFor
        \State Set $\oV^k_{h,i}(s)\leftarrow \proj_{[0,H+1-h]}\Sp{\oV^k_{h,i}(s)+\frac{H}{T}\cdot\mathrm{(Swap)Reg}(T)}$ for all $i\in[m]$ and $s\in\S$. 
        %\State Set $\uV^k_{h,i}(s)\leftarrow \max\{\uV^k_{h,i}(s),0\}$ for all $i\in[m]$ and $s\in\S$. 
        \EndFor
        \State Set $\pi^{k}$ to be the Markov joint policy such that $\pi^{k}_h(\a|s)=\frac{1}{T}\sum_{t=1}^T\prod_{i\in[m]}\pi_{h,i}^{k,t}(a_i|s)$.
        \State Update $\Pi\leftarrow\Pi\bigcup\{\pi^k\}$. \Comment{Policy cover update} \label{line:pc update}
	    \EndFor
	    \State Sample $k\sim\mathrm{Unif}(K)$ and output $\pi^{\mathrm{output}}=\pi^k$. \label{line:unif output}
    \end{algorithmic}
\end{algorithm}

We will set the parameters for Algorithm \ref{algo:no com} to be
\begin{itemize}
    \item $\lambda=\frac{2\log(16d_{\max}mKHT/\delta)}{\log(36/35)}$
    \item $W=H\sqrt{d_{\max}}$
    \item $\beta=16(W+H)\sqrt{\lambda+d_{\max}\log(32WN(W+H))+4\log(8mK_{\max}HT/\delta)}$
    \item $T=\widetilde{O}(H^4\log(A_{\max})\epsilon^{-2})$ for Markov CCE and $T=\widetilde{O}(H^4A_{\max}\log(A_{\max})\epsilon^{-2})$ for Markov CE
    \item $K=\widetilde{O}(m^2H^4d_{\max}^2\epsilon^{-2})$.
\end{itemize}

\subsection{Concentration}
The population covariance matrix for episode $k$, inner loop $t$, step $h$ and player $i$ is defined as
$$\Sigma_{h,i}^{k}:=\E\Mp{\widehat{\Sigma}_{h,i}^{k,t}}=\lambda I+\sum_{l=1}^{k-1}\Sigma_{h,i}^{\pi^l},$$
where $\Sigma_{h,i}^{\pi^k}=\E_{\pi^k}\Mp{\phi_i(s_h,a_{h,i})\phi_i(s_h,a_{h,i})^\top}$. Note that $s_h^l,a_{h,i}^l$ is sampled following the same policy for each inner loop $t$, so the expected covariance is the same for different $t$. 

We define $\pi^{k,\mathrm{cov}}$ to be the mixture policy in $\Pi^k=\{\pi^l\}_{l=1}^{k-1}$, where policy $\pi^l$ is given weight/probability $\frac{1}{k-1}$, and also define 
$$\ott_{h,i}^{k,t}:=\argmin_{\Norm{\theta}\leq W}\E_{(s_h,a_{h,i})\sim \pi^{k,\mathrm{cov}}}\Bp{\inner{\phi_i(s_h,a_{h,i}),\theta}-\E_{a_{h,-i}\sim\pi_{h,-i}^{k,t}(\cdot\mid s)}\Mp{r_{h,i}(s_h,\a_h)+\oV_{h+1,i}^k(s')}}^2,$$
$$\utt_{h,i}^{k,t}:=\argmin_{\Norm{\theta}\leq W}\E_{(s_h,a_{h,i})\sim \pi^{k,\mathrm{cov}}}\Bp{\inner{\phi_i(s_h,a_{h,i}),\theta}-\E_{a_{h,-i}\sim\pi_{h,-i}^{k,t}(\cdot\mid s)}\Mp{r_{h,i}(s_h,\a_h)+\uV_{h+1,i}^k(s')}}^2.$$

\begin{lemma}\label{lemma:concentration w/o com}
(Concentration) With probability at least $1-\delta/2$, for all $k\in[K]$, $h\in[H]$, $t\in[T]$, $i\in[m]$, we have
\begin{equation}\label{eq:opt concentration w/o}
    \Norm{\ot_{h,i}^{k,t}-\ott_{h,i}^{k,t}}_{\Sigma_{h,i}^k}\leq 8(W+H)\sqrt{\lambda+d_i\log(32WK(W+H))+4\log(8mKHT/\delta)}\leq\beta/2,
\end{equation}
\begin{equation}\label{eq:pes concentration w/o}
    \Norm{\ut_{h,i}^{k,t}-\utt_{h,i}^{k,t}}_{\Sigma_{h,i}^k}\leq 8(W+H)\sqrt{\lambda+d_i\log(32WK(W+H))+4\log(8mKHT/\delta)}\leq\beta/2,
\end{equation}
\begin{equation}\label{eq:cov concentration w/o}
    \frac{1}{2}\Sigma_{h,i}^{k,t}\preceq \Sigma_{h,i}^k\preceq \frac{3}{2}\Sigma_{h,i}^{k,t}.  
\end{equation}
\end{lemma}

\begin{proof}
The proof is the same as the proof for Lemma \ref{lemma:concentration 1}. 
\end{proof}

With a slight abuse of the notation, we will still denote the high probability event in Lemma \ref{lemma:concentration w/o com} as $\G$ 
% \kz{shall we define $\G$ not inline, but more explicit? it is a bit hidden now.}\qiwen{I defined it this way to be consistent with how I defined it for Lemma \ref{lemma:concentration 1} and Lemma \ref{lemma:concentration 2}. We can change all of them but I am not sure how to define it concisely there.}. 
Now we define 
$$\overline{\Delta}_{h,i}^{k,t}(s,a_i)=\E_{a_{-i}\sim\pi_{h,-i}^{k,t}(\cdot\mid s)}\Mp{r_{h,i}(s,\a)+\oV_{h+1,i}^k(s')}-\proj_{[0,H+1-h]}\Sp{\inner{\phi_i(s,a_i),\ott_{h,i}^{k,t}}},$$
$$\underline{\Delta}_{h,i}^{k,t}(s,a_i)=\E_{a_{-i}\sim\pi_{h,-i}^{k,t}(\cdot\mid s)}\Mp{r_{h,i}(s,\a)+\uV_{h+1,i}^k(s')}-\proj_{[0,H+1-h]}\Sp{\inner{\phi_i(s,a_i),\utt_{h,i}^{k,t}}}.$$

\begin{lemma}\label{lemma:Q error w/o com}
 Under good event $\G$, for all $k\in[K]$, $t\in[T]$, $h\in[H]$, $i\in[m]$, $s\in\S$ and $a_i\in\A_i$ we have
$$-\overline{\Delta}_{h,i}^{k,t}(s,a_i)\leq \oQ_{h,i}^{k,t}(s,a_i)-\Mp{\E_{a_{-i}\sim\pi_{h,-i}^{k,t}(\cdot\mid s)}\Mp{r_{h,i}(s,\a)+\oV_{h+1,i}^k(s')}}\leq3\beta\Norm{\phi_i(s,a_i)}_{[\Sigma_{h,i}^k]^{-1}}-\overline{\Delta}_{h,i}^{k,t}(s,a_i),$$
$$-3\beta\Norm{\phi_i(s,a_i)}_{[\Sigma_{h,i}^k]^{-1}}-\underline{\Delta}_{h,i}^{k,t}(s,a_i)\leq \uQ_{h,i}^{k,t}(s,a_i)-\Mp{\E_{a_{-i}\sim\pi_{h,-i}^{k,t}(\cdot\mid s)}\Mp{r_{h,i}(s,\a)+\uV_{h+1,i}^k(s')}}\leq-\underline{\Delta}_{h,i}^{k,t}(s,a_i). $$
\end{lemma}

\begin{proof}
The proof is the same as the proof for Lemma \ref{lemma:Q error}. 
\end{proof}

\subsection{Proofs for Learning Markov CCE with Algorithm \ref{algo:no com}}

\begin{lemma}\label{lemma:CCE gap w/o com}
Under the good event $\G$, for all $k\in[K]$ and $i\in[m]$, we have
$$V^{\dagger,\pi_{-i}^k}_{1,i}(s_1)-V^{\pi^k}_{1,i}(s_1)-\nu\leq \oV_{1,i}^k(s_1)-V^{\pi^k}_{1,i}(s_1)\leq3\beta\E_{\pi^{k}}\sum_{h=1}^H\Norm{\phi_i(s_h,a_{h,i})}_{[\Sigma_{h,i}^k]^{-1}}+\frac{H}{T}\cdot\mathrm{Reg}(T)+\nu. $$
\end{lemma}

\begin{proof}
The first inequality is from Lemma \ref{lemma:optimism}. Now we prove the second argument:  
\begin{align*}
    &\oV_{1,i}^k(s_1)-V^{\pi^k}_{1,i}(s_1)\\
    \leq&\frac{1}{T}\sum_{t=1}^T\sum_{a_{1,i}\in\A_i}\pi_{1,i}^{k,t}(a_{1,i}\mid s_1)\oQ_{1,i}^{k,t}(s_1,a_{1,i})+\frac{H}{T}\cdot\mathrm{Reg}(T)-V^{\pi^k}_{1,i}(s_1)\\
    \leq&\frac{1}{T}\sum_{t=1}^T\sum_{a_{1,i}\in\A_i}\pi_{1,i}^{k,t}(a_{1,i}\mid s_1)\Sp{\Mp{\E_{a_{1,-i}\sim\pi_{1,-i}^{k,t}(\cdot\mid s_1)}\Mp{r_{h,i}(s_1,\a_1)+\oV_{2,i}^k(s_2)}}+3\beta\Norm{\phi_i(s_1,a_{1,i})}_{[\Sigma_{1,i}^k]^{-1}}-\overline{\Delta}_{1,i}^{k,t}(s_1,a_{1,i})}\\
    &+\frac{H}{T}\cdot\mathrm{Reg}(T)-V^{\pi^k}_{1,i}(s_1)\tag{Lemma \ref{lemma:Q error w/o com}}\\
    \leq&\frac{1}{T}\sum_{t=1}^T\Sp{\Mp{\E_{\a_1\sim\pi_{1}^{k,t}(\cdot\mid s_1)}\Mp{\oV_{2,i}^k(s_2)-V_{2,i}^{\pi^k}(s_2)}}+\E_{a_{1,i}\sim\pi^{k,t}_{1,i}(\cdot\mid s_1)}\Mp{3\beta\Norm{\phi_i(s_1,a_{1,i})}_{[\Sigma_{1,i}^k]^{-1}}-\overline{\Delta}_{1,i}^{k,t}(s_1,a_{1,i})}}\\&+\frac{H}{T}\cdot\mathrm{Reg}(T)\\
    \leq&\E_{\pi_{1}^{k}}\Mp{\oV_{2,i}^k(s_2)-V_{2,i}^{\pi^k}(s_2)}+\E_{a_{1,i}\sim\pi^{k}_{1,i}(\cdot\mid s_1)}\Mp{3\beta\Norm{\phi_i(s_1,a_{1,i})}_{[\Sigma_{1,i}^k]^{-1}}-\frac{1}{T}\sum_{t=1}^T\overline{\Delta}_{1,i}^{k,t}(s_1,a_{1,i})}+\frac{H}{T}\cdot\mathrm{Reg}(T)\\
    \leq&3\beta\E_{\pi^{k}}\sum_{h=1}^H\Norm{\phi_i(s_h,a_{h,i})}_{[\Sigma_{h,i}^k]^{-1}}-\E_{\pi^k}\sum_{h=1}^H\frac{1}{T}\sum_{t=1}^T\overline{\Delta}_{h,i}^{k,t}(s_h,a_{h,i})+\frac{H^2}{T}\cdot\mathrm{Reg}(T) \\
    \leq&3\beta\E_{\pi^{k}}\sum_{h=1}^H\Norm{\phi_i(s_h,a_{h,i})}_{[\Sigma_{h,i}^k]^{-1}}+\frac{H^2}{T}\cdot\mathrm{Reg}(T)+\nu\tag{Lemma \ref{lemma:misspecification}}.
\end{align*}
\end{proof}

\begin{lemma}\label{lemma:sum of bonus w/o com}
Under the good event $\G$, we have
$$\sum_{k=1}^K\E_{\pi^k}\Norm{\phi_i(s_h,a_{h,i})}^2_{[\Sigma_{h,i}^k]^{-1}}\leq  d_i\log(1+\frac{K}{d_i\lambda}). $$
\end{lemma}

\begin{proof}
As $\Norm{\phi_i(s_h,a_{h,i})}^2_{[\Sigma_{h,i}^k]^{-1}}\leq 1$, by Lemma \ref{lemma:information gain} we have
$$\E_{\pi^k}\Norm{\phi_i(s_h,a_{h,i})}^2_{[\Sigma_{h,i}^k]^{-2}}\leq\log\frac{\det(\Sigma_{h,i}^{k+1})}{\det(\Sigma_{h,i}^k)}. $$
Thus we have
\begin{align*}
    \sum_{k=1}^K\E_{\pi^k}\Norm{\phi_i(s_h,a_{h,i})}^2_{[\Sigma_{h,i}^k]^{-1}}\leq& \sum_{k=1}^K\log\frac{\det(\Sigma_{h,i}^{k+1})}{\det(\Sigma_{h,i}^k)}\\
    =&\log\frac{\det(\Sigma_{h,i}^{K+1})}{\det(\Sigma_{h,i}^1)}\\
    \leq& d_i\log(1+\frac{K}{d_i\lambda}),
\end{align*}
where we utilized the fact that
$$\log\det(\Sigma_{h,i}^{K+1})\leq d_i\log\Sp{\frac{\mathrm{trace}(\Sigma_{h,i}^{K+1})}{d_i}}\leq d_i\log\Sp{\frac{d_i\lambda+K}{d_i}}. $$
\end{proof}

\begin{lemma}\label{lemma:CCE regret w/o com}
Under the good event $\G$, we have
$$\sum_{k=1}^K\max_{i\in[m]}\Sp{\oV_{1,i}^k(s_1)-V_{1,i}^{\pi^k}(s_1)}\leq3mH\beta\sqrt{Kd_{\max}\log\Sp{1+\frac{K}{\lambda}}}+\frac{H^2K}{T}\cdot\mathrm{Reg}(T)+\nu K. $$

\end{lemma}

\begin{proof}
\begin{align*}
    &\sum_{k=1}^K\max_{i\in[m]}\Sp{\oV_{1,i}^k(s_1)-V_{1,i}^{\pi^k}(s_1)}\\
    \leq& 3\beta\sum_{k=1}^K\max_{i\in[m]}\E_{\pi^{k}}\sum_{h=1}^H\Norm{\phi_i(s_h,a_{h,i})}_{[\Sigma_{h,i}^k]^{-1}}+\frac{H^2}{T}\sum_{k=1}^K\mathrm{Reg}(T)+\nu K\tag{Lemma \ref{lemma:CCE gap w/o com}}\\
    =&3\beta\sum_{i\in[m]}\sum_{h=1}^H\sum_{k=1}^K\E_{\pi^{k}}\sqrt{\Norm{\phi_i(s_h,a_{h,i})}_{[\Sigma_{h,i}^k]^{-1}}^2}+\frac{H^2K}{T}\cdot\mathrm{Reg}(T)+\nu K\\
    \leq&3\beta\sum_{i\in[m]}\sum_{h=1}^H\sum_{k=1}^K\sqrt{\E_{\pi^{k}}\Norm{\phi_i(s_h,a_{h,i})}_{[\Sigma_{h,i}^k]^{-1}}^2}+\frac{H^2K}{T}\cdot\mathrm{Reg}(T)+\nu K\tag{Concavity of $f(x)=\sqrt{x}$}\\
    \leq&3\beta\sum_{i\in[m]}\sum_{h=1}^H\sqrt{K\sum_{k=1}^K\E_{\pi^{k}}\Norm{\phi_i(s_h,a_{h,i})}_{[\Sigma_{h,i}^k]^{-1}}^2}+\frac{H^2K}{T}\cdot\mathrm{Reg}(T)+\nu K \tag{Cauchy–Schwarz inequality}\\
    \leq&3\beta\sum_{i\in[m]}\sum_{h=1}^H\sqrt{Kd_i\log\Sp{1+\frac{K}{d_i\lambda}}}+\frac{H^2K}{T}\cdot\mathrm{Reg}(T)+\nu K\tag{Lemma \ref{lemma:sum of uncertainty}}\\
    \leq&3mH\beta\sqrt{Kd_{\max}\log\Sp{1+\frac{K}{\lambda}}}+\frac{H^2K}{T}\cdot\mathrm{Reg}(T)+\nu K.
\end{align*}
\end{proof}

\begin{theorem}\label{thm:CCE w/o com}
Suppose Algorithm \ref{algo:no com} is instantiated with no-regret learning oracles satisfying Assumption \ref{asp:no regret}. Then for $\nu$-misspecified linear Markov games, with probability $0.9$, Algorithm \ref{algo:no com} will output an $(\epsilon+2\nu)$-approximate Markov CCE. The sample complexity is $O(mHTK^2)=\widetilde{O}(m^5H^{13}d_{\max}^6\log(A_{\max})\epsilon^{-6})$, where $d_{\max}=\max_{i\in[m]}d_i$ and $A_{\max}=\max_{i\in[m]}A_i$. 
\end{theorem}

\begin{proof}
By Lemma \ref{lemma:CCE regret w/o com}, under the good event $\G$, which happens with probability at least $1-\delta$ (Lemma \ref{lemma:concentration 1}), we have
\begin{align*}
\frac{1}{K}\sum_{k=1}^K\max_{i\in[m]}\Sp{V^{\dagger,\pi_{-i}^k}_{1,i}(s_1)-V_{1,i}^{\pi^k}(s_1)}
\leq&\frac{1}{K}\sum_{k=1}^K\max_{i\in[m]}\Sp{\oV_{1,i}^k(s_1)-V_{1,i}^{\pi^k}(s_1)}+\nu\tag{Lemma \ref{lemma:optimism}}\\
    \leq&3mH\beta\sqrt{d_{\max}\log\Sp{1+\frac{K}{\lambda}}/K}+\frac{H^2}{T}\cdot\mathrm{Reg}(T)+2\nu. \tag{Lemma \ref{lemma:CCE regret w/o com}}
\end{align*}
By Markov's inequality, we set $K=\widetilde{O}(m^2H^4d_{\max}^3\epsilon^{-2})$ and $T=\widetilde{O}(H^4\log(A_{\max})\epsilon^{-2})$, with probability 0.9 we have
$$\max_{i\in[m]}\Sp{V^{\dagger,\pi_{-i}^{\mathrm{output}}}_{1,i}(s_1)-V_{1,i}^{\pi^{\mathrm{output}}}(s_1)}\leq\epsilon+2\nu.$$
\end{proof}

\subsection{Proofs for Learning Markov CE with Algorithm \ref{algo:no com}}

\begin{lemma}\label{lemma:CE gap w/o com}
Under the good event $\G$, for all $k\in[K]$ and $i\in[m]$, we have
$$\max_{\psi_i}V_{1,i}^{\psi_i\diamond\pi^k}(s_1)-V^{\pi^k}_{1,i}(s_1)-\nu\leq \oV_{1,i}^k(s_1)-V^{\pi^k}_{1,i}(s_1)\leq3\beta\E_{\pi^{k}}\sum_{h=1}^H\Norm{\phi_i(s_h,a_{h,i})}_{[\Sigma_{h,i}^k]^{-1}}+\frac{H}{T}\cdot\mathrm{SwapReg}(T)+\nu. $$
\end{lemma}

\begin{proof}
The first inequality is from Lemma \ref{lemma:optimism swap}. Now we prove the second argument. 
\begin{align*}
    &\oV_{1,i}^k(s_1)-V^{\pi^k}_{1,i}(s_1)\\
    \leq&\frac{1}{T}\sum_{t=1}^T\sum_{a_{1,i}\in\A_i}\pi_{1,i}^{k,t}(a_{1,i}\mid s)\oQ_{1,i}^{k,t}(s,a_{1,i})+\frac{H}{T}\cdot\mathrm{SwapReg}(T)-V^{\pi^k}_{1,i}(s_1)\\
    \leq&\frac{1}{T}\sum_{t=1}^T\sum_{a_{1,i}\in\A_i}\pi_{1,i}^{k,t}(a_{1,i}\mid s_1)\Sp{\Mp{\E_{a_{1,-i}\sim\pi_{1,-i}^{k,t}(\cdot\mid s_1)}\Mp{r_{1,i}(s_1,\a_1)+\oV_{2,i}^k(s_2)}}+3\beta\Norm{\phi_i(s_1,a_{1,i})}_{[\Sigma_{1,i}^k]^{-1}}-\overline{\Delta}_{1,i}^{k,t}(s_1,a_{1,i})}\\
    &+\frac{H}{T}\cdot\mathrm{SwapReg}(T)-V^{\pi^k}_{1,i}(s_1)\tag{Lemma \ref{lemma:Q error w/o com}}\\
    \leq&\frac{1}{T}\sum_{t=1}^T\Sp{\Mp{\E_{\a_1\sim\pi_{1}^{k,t}(\cdot\mid s_1)}\Mp{\oV_{2,i}^k(s_2)-V_{2,i}^{\pi^k}(s_2)}}+3\beta\E_{a_{1,i}\sim\pi^{k,t}_{1,i}(\cdot\mid s_1)}\Norm{\phi_i(s,a_{1,i})}_{[\Sigma_{1,i}^k]^{-1}}-\overline{\Delta}_{1,i}^{k,t}(s_1,a_{1,i})}\\&+\frac{H}{T}\cdot\mathrm{SwapReg}(T)\\
    \leq&\E_{\pi_{1}^{k}}\Mp{\oV_{2,i}^k(s_2)-V_{2,i}^{\pi^k}(s_2)}+\E_{a_{1,i}\sim\pi^{k}_{1,i}(\cdot\mid s_1)}\Mp{3\beta\Norm{\phi_i(s_1,a_{1,i})}_{[\Sigma_{1,i}^k]^{-1}}-\frac{1}{T}\sum_{t=1}^T\overline{\Delta}_{1,i}^{k,t}(s_1,a_{1,i})}+\frac{H}{T}\cdot\mathrm{SwapReg}(T)\\
    \leq&3\beta\E_{\pi^{k}}\sum_{h=1}^H\Norm{\phi_i(s_h,a_{h,i})}_{[\Sigma_{h,i}^k]^{-1}}-\E_{\pi^{k}}\sum_{h=1}^H\frac{1}{T}\sum_{t=1}^T\overline{\Delta}_{h,i}^{k,t}(s_h,a_{h,i})+\frac{H^2}{T}\cdot\mathrm{SwapReg}(T)\\
    \leq&3\beta\E_{\pi^{k}}\sum_{h=1}^H\Norm{\phi_i(s_h,a_{h,i})}_{[\Sigma_{h,i}^k]^{-1}}+\frac{H^2}{T}\cdot\mathrm{SwapReg}(T)+\nu\tag{Lemma \ref{lemma:misspecification}},
\end{align*}
which completes the proof. 
\end{proof}

\begin{lemma}\label{lemma:CE regret w/o com}
Under the good event $\G$, we have
$$\sum_{k=1}^K\max_{i\in[m]}\Sp{\oV_{1,i}^k(s_1)-V_{1,i}^{\pi^k}(s_1)}\leq3mH\beta\sqrt{Kd_{\max}\log\Sp{1+\frac{K}{\lambda}}}+\frac{H^2K}{T}\cdot\mathrm{SwapReg}(T)+\nu. $$
\end{lemma}

\begin{proof}
The proof is the same as the proof for Lemma \ref{lemma:CCE regret w/o com} where we replace Lemma \ref{lemma:CCE gap w/o com} with Lemma \ref{lemma:CE gap w/o com} in the proof. 
\end{proof}

\begin{theorem}\label{thm:CE w/o com}
Suppose Algorithm \ref{algo:no com} is instantiated with no-regret learning oracles satisfying Assumption \ref{asp:no swap regret}. Then for $\nu$-misspecified linear Markov games, with probability $0.9$, Algorithm \ref{algo:no com} will output an $(\epsilon+2\nu)$-approximate Markov CCE. The sample complexity is $O(mHTK^2)=\widetilde{O}(m^5H^{13}d_{\max}^6A_{\max}\log(A_{\max})\epsilon^{-6})$, where $d_{\max}=\max_{i\in[m]}d_i$ and $A_{\max}=\max_{i\in[m]}A_i$.
\end{theorem}

\begin{proof}
By Lemma \ref{lemma:CE regret w/o com}, under the good event $\G$, which happens with probability at least $1-\delta$ (Lemma \ref{lemma:concentration w/o com}), we have
\begin{align*}
\frac{1}{K}\sum_{k=1}^K\max_{i\in[m]}\Sp{\max_{\psi_i}V_{1,i}^{\psi_i\diamond\pi^k}(s_1)-V_{1,i}^{\pi^k}(s_1)}
\leq&\frac{1}{K}\sum_{k=1}^K\max_{i\in[m]}\Sp{\oV_{1,i}^k(s_1)-V_{1,i}^{\pi^k}(s_1)}+\nu\tag{Lemma \ref{lemma:optimism swap}}\\
    \leq&3mH\beta\sqrt{d_{\max}\log\Sp{1+\frac{K}{\lambda}}/K}+\frac{mH^2}{T}\cdot\mathrm{SwapReg}(T)+2\nu.\tag{Lemma \ref{lemma:CE regret w/o com}}
\end{align*}
By Markov's inequality, we set $K=\widetilde{O}(m^2H^4d_{\max}^3\epsilon^{-2})$ and $T=\widetilde{O}(H^4A_{\max}\log(A_{\max})\epsilon^{-2})$, with probability 0.9, we have
$$\max_{i\in[m]}\max_{\psi_i}\Sp{V_{1,i}^{\psi_i\diamond\pi^{\mathrm{output}}}(s_1)-V_{1,i}^{\pi^{\mathrm{output}}}(s_1)}\leq\epsilon,$$
which completes the proof. 
\end{proof}

\section{Algorithms for Learning Optimal Policies in Misspecified Linear MDP}\label{apx:linear MDP}

In this section, we adapt Algorithm \ref{algo} to the linear MDP setting. As the single-agent degeneration of independent linear Markov games, we can remove the no-regret learning loop in Algorithm \ref{algo} and achieve better sample complexity. The analysis is almost the same as the analysis for Algorithm \ref{algo} in Appendix \ref{apx:linear MG proof} with $T=1$ and $m=1$.

\begin{algorithm}[!]
    \caption{Policy Replay for Misspecified MDP with linear function approximation}
	\label{algo:mdp}
    \begin{algorithmic}[1]
        \State {\bfseries Input:} $\epsilon$, $\delta$, $\lambda$, $\beta$, $T_{\Trig}$, $K_{\max}$, $N$
        \State {\bfseries Initialization:} Policy Cover $\Pi=\emptyset$. $n^\tot=0$. 
        \For{episode $k=1,2,\dots,K_{\max}$}
        \State Set $\oV_{H+1}^k(\cdot)=\uV_{H+1}^k(\cdot)=0$, $n^k=0$.  
        \For{$h=H,H-1,\dots,1$}\Comment{Retrain policy with the current policy cover}
        \State Initialize $\oV_{h}^k(\cdot)=\uV_{h}^k(\cdot)=0$.  
        %\STATE Call Algorithm \ref{algo:streaming LS} with policy cover $\Pi$ and opponent strategy $\pi_{h,-i}^{k,t}$ and receive $\ot_{h,i}^{k,t}$, $\ut_{h,i}^{k,t}$ and $[\Sigma_{h,i}^{k,t}]^{-1}$.
        \State Set Dataset $\D_{h}^{k}=\emptyset$.
        \For{$l=1,2,\dots,\sum_{j=1}^{k-1}n^j$}
        \State Sample $\pi^l$ with probability $n^l/\sum_{j=1}^{k-1}n^j$. 
        \State Draw a joint trajectory $(s_{1}^l,a_1^l,r_{1}^l,\dots,s_H^l,a_H^l,r_{H}^l,s_{H+1}^l)$ from $\pi^l$.
        \State Add $(s_h^l,a_{h,i}^l,r_{h,i}^l,s_{h+1}^l)$ to $\D_{h}^{k}$. 
        \EndFor
        \State Set $\widehat{\Sigma}_{h}^{k}=\lambda I+\sum_{(s,a,r,s')\in\D_{h}^{k}}\phi(s,a)\phi(s,a)^\top$. 
        %\State Set $\ot_{h,i}^{k,t}=[\Sigma_{h,i}^{k,t}]^{-1}\sum_{(s_h^l,a_{h,i}^l,r_{h,i}^l,s_{h+1}^l)\in\D_{h,i}^{k,t}}\phi_i(s_h^l,a_{h,i}^l)(r_{h,i}^l+\oV^k_{h+1,i}(s_{h+1}^l))$.
        %\State Set $\ut_{h,i}^{k,t}=[\Sigma_{h,i}^{k,t}]^{-1}\sum_{(s_h^l,a_{h,i}^l,r_{h,i}^l,s_{h+1}^l)\in\D_{h,i}^{k,t}}\phi_i(s_h^l,a_{h,i}^l)(r_{h,i}^l+\uV^k_{h+1,i}(s_{h+1}^l))$.
        \State Set $\ot_{h}^{k}=\argmin_{\Norm{\theta}\leq H\sqrt{d}}\sum_{(s,a,r,s')\in\D_{h}^{k}}\Sp{\inner{\phi(s,a),\theta}-r-\oV^k_{h+1}(s')}^2$.
        \State Set $\ut_{h}^{k}=\argmin_{\Norm{\theta}\leq H\sqrt{d}}\sum_{(s,a,r,s')\in\D_{h}^{k}}\Sp{\inner{\phi(s,a),\theta}-r-\uV^k_{h+1}(s')}^2$.
        \State Set $\oQ_{h}^{k}(\cdot,\cdot)=\proj_{[0,H+1-h]}\Sp{\inner{\phi(\cdot,\cdot),\ot_{h}^{k}}+\beta\Norm{\phi(\cdot,\cdot)}_{[\widehat{\Sigma}_{h}^{k}]^{-1}}}$.
        \State Set $\uQ_{h}^{k}(\cdot,\cdot)=\proj_{[0,H+1-h]}\Sp{\inner{\phi(\cdot,\cdot),\ut_{h}^{k}}-\beta\Norm{\phi(\cdot,\cdot)}_{[\widehat{\Sigma}_{h}^{k}]^{-1}}}$.
        \State Set $\oV_{h}^k(\cdot)=\max_{a\in\A}\oQ_h^k(\cdot,a)$.
        \State Set $\oV_{h}^k(\cdot)=\uQ_h^k(\cdot,\argmax_{a\in\A}\oQ_h^k(\cdot,a))$
        \EndFor
        \State Set $\pi^{k}$ to be the policy such that $\pi^{k}_h(s)=\argmax_{a\in\A}\oQ_h^k(s,a)$ for all $(h,s)\in[H]\times\S$.
        \If{$n^\tot=N$}
        \State Set $k^{\mathrm{output}}=\argmin_k\Sp{\oV_{1}^k(s_1)-\uV_{1}^k(s_1)}$. 
        \State Output $\pi^{\mathrm{output}}=\pi^{k^{\mathrm{output}}}$. 
        \EndIf
        \State Set $T_{h,i}=0$, for all $h\in[H],i\in[m]$.
        \Repeat\Comment{Update policy cover}\label{line:repeat mdp}
        \State Reset to $s=s_1$, $n^k=n^k+1$, $n^\tot=n^\tot+1$.
        \For{$h=1,2,\dots,H$}
        \State Play $a=\pi_h^k(\cdot|s)$.
        \State $T_{h}\rightarrow T_{h}+\Norm{\phi(s,a)}_{[\widehat{\Sigma}_{h}^{k}]^{-1}}^2$.
        \State Get next state $s'$, $s\rightarrow s'$.
        \EndFor
        \Until{$\exists h\in[H]$ such that $T_{h}\geq T_{\Trig}$ or $n^\tot=N$. } \label{line:trigger mdp}
        \State Update $\Pi\leftarrow\Pi\bigcup\{(\pi^k,n^k)\}$. 
	    \EndFor

    \end{algorithmic}
\end{algorithm}

We will set the parameters for Algorithm \ref{algo} to be
\begin{itemize}
    \item $\lambda=\frac{2\log(16dNH/\delta)}{\log(36/35)}$ 
    \item $W=H\sqrt{d}$
    \item $\beta=16(W+H)\sqrt{\lambda+d\log(32W(W+H))+4\log(8K_{\max}H/\delta)}$
    \item $T_{\Trig}=64\log(8HN^2/\delta)$
    \item $K_{\max}=\min\{\frac{2Hd\log(N+\lambda)}{\log(1+T_\Trig/4)},N\}$ 
    \item $N=\widetilde{O}(H^4d^2\epsilon^{-2})$.
\end{itemize}
We will use $K$ to denote the episode that Algorithm \ref{algo:mdp} ends ($n^\mathrm{tot}=N$ or $K=K_{\max}$). Immediately we have $K\leq K_{\max}\leq N$. 

The population covariance matrix for episode $k$, step $h$ is defined as
$$\Sigma_{h}^{k}:=\E\Mp{\widehat{\Sigma}_{h}^{k}}=\lambda I+\sum_{l=1}^{k-1}n^l\Sigma_{h}^{\pi^l},$$
where $\Sigma_{h}^{\pi^k}=\E_{\pi^k}\Mp{\phi(s_h,a_{h})\phi(s_h,a_{h})^\top}$. 

We define $\pi^{k,\mathrm{cov}}$ to be the mixture policy in $\Pi^k=\{(\pi^l,n^l)\}_{l=1}^{k-1}$, where policy $\pi^l$   is given weight/probability $\frac{n^l}{\sum_{j=1}^{k-1}n^j}$.  Then we define the on-policy population fit to be
$$\ott_{h}^{k}:=\argmin_{\Norm{\theta}\leq W}\E_{(s_h,a_{h})\sim \pi^{k,\mathrm{cov}}}\Bp{\inner{\phi(s_h,a_{h}),\theta}-\E\Mp{r_{h}(s_h,a_h)+\oV_{h+1}^k(s')}}^2,$$
$$\utt_{h}^{k}:=\argmin_{\Norm{\theta}\leq W}\E_{(s_h,a_{h})\sim \pi^{k,\mathrm{cov}}}\Bp{\inner{\phi(s_h,a_{h}),\theta}-\E\Mp{r_{h}(s_h,a_h)+\uV_{h+1}^k(s')}}^2.$$

We define the misspecification error to be 
$$\overline{\Delta}_{h}^{k}(s,a):=\E\Mp{r_{h}(s,a)+\oV_{h+1}^k(s')}-\proj_{[0,H+1-h]}\Sp{\inner{\phi(s,a),\ott_{h}^{k}}},$$ 
$$\underline{\Delta}_{h}^{k}(s,a):=\E\Mp{r_{h}(s,a)+\uV_{h+1}^k(s')}-\proj_{[0,H+1-h]}\Sp{\inner{\phi(s,a),\utt_{h}^{k}}}.$$

\begin{lemma}\label{lemma:mdp concentration 1}
(Concentration) With probability at least $1-\delta/2$, for all $k\in[K]$, $h\in[H]$, we have
\begin{equation}
    \Norm{\ot_{h}^{k}-\ott_{h}^{k}}_{\Sigma_{h}^k}\leq 8(W+H)\sqrt{\lambda+d\log(32WN(W+H))+4\log(8K_{\max}H/\delta)}\leq\beta/2,
\end{equation}
\begin{equation}
    \Norm{\ut_{h}^{k,t}-\utt_{h}^{k}}_{\Sigma_{h}^k}\leq 8(W+H)\sqrt{\lambda+d\log(32WN(W+H))+4\log(8K_{\max}H/\delta)}\leq\beta/2,
\end{equation}
\begin{equation}
    \frac{1}{2}\widehat{\Sigma}_{h}^{k}\preceq \Sigma_{h}^k\preceq \frac{3}{2}\widehat{\Sigma}_{h}^{k}.  
\end{equation}
\end{lemma}

\begin{proof}
The proof is the same as the proof for Lemma \ref{lemma:concentration 1}. 
\end{proof}

\begin{lemma}\label{lemma:mdp concentration 2}
With probability at least $1-\delta/2$, the following two events hold:
\begin{itemize}
    \item Suppose at episode   $k$, Line \ref{line:trigger mdp}: $T_{h}\geq T_\Trig$ is triggered, then we have
    $$\E_{\pi^{k}}\Norm{\phi(s_h,a_{h})}^2_{\Mp{\widehat{\Sigma}_{h}^{k}}^{-1}}\geq \frac{1}{2n^k}\sum_{j=1}^{n^k}\Norm{\phi(s_h^{k,j},a_{h}^{k,j})}^2_{\Mp{\widehat{\Sigma}_{h}^{k}}^{-1}}\geq \frac{T_\Trig}{2n^k},$$
    where $j$ denotes the $j$-th trajectory collected in the policy cover update (Line \ref{line:repeat mdp}). 
    \item For any $k\in[K_{\max}]$, $h\in[H]$, we have
    $$\E_{\pi^{k}}\Norm{\phi(s_h,a_{h})}^2_{\Mp{\widehat{\Sigma}_{h}^{k}}^{-1}}\leq \frac{2T_\Trig}{n^k}. $$
\end{itemize}
\end{lemma}

\begin{proof}
The proof is the same as the proof for Lemma \ref{lemma:concentration 2}. 
\end{proof}

We denote $\G$ to be the good event where the arguments in Lemma \ref{lemma:mdp concentration 1} and Lemma \ref{lemma:mdp concentration 2} hold, which holds with probability at least $1-\delta$ by Lemma \ref{lemma:mdp concentration 1} and Lemma \ref{lemma:mdp concentration 2}. 

\begin{lemma}\label{lemma:mdp Q error}
 Under good event $\G$, for all $k\in[K]$, $h\in[H]$, $s\in\S$ and $a\in\A$, we have
$$-\overline{\Delta}_{h}^{k}(s,a)\leq \oQ_{h}^{k}(s,a)-\Mp{\E\Mp{r_{h}(s,a)+\oV_{h+1}^k(s')}}\leq3\beta\Norm{\phi(s,a)}_{[\Sigma_{h}^k]^{-1}}-\overline{\Delta}_{h}^{k}(s,a),$$
$$-3\beta\Norm{\phi(s,a)}_{[\Sigma_{h}^k]^{-1}}-\underline{\Delta}_{h}^{k}(s,a)\leq \uQ_{h}^{k}(s,a)-\Mp{\E\Mp{r_{h}(s,a)+\uV_{h+1}^k(s')}}\leq-\underline{\Delta}_{h}^{k}(s,a). $$
\end{lemma}

\begin{proof}
The proof is the same as the proof for Lemma \ref{lemma:Q error}.  
\end{proof}

\begin{lemma}\label{lemma:mdp optimism}
(Optimism) Under the good event $\G$, for all $k\in[K]$, we have
$$\oV_{1}^{k}(s_1)\geq V_{1}^{*}(s_1)-\sum_{h=1}^H\E_{\pi^*}\Mp{\overline{\Delta}_{h}^{k}(s_h,a_{h})}\geq V_{1}^{*}(s_1)-\nu. $$
\end{lemma}

\begin{proof}
Under the good event $\G$, for all $k\in[K]$, we have
\begin{align*}
    &\oV_{1}^{k}(s_1)-V_{1}^{*}(s_1)\\
    =&\max_{a_1\in\A}\oQ_1^k(s_1,a_1)-V_1^*(s_1)\\
    \geq& \oQ_1^k(s_1,\pi^*_1(s_1))-Q_1^*(s_1,\pi^*_1(s_1))\\
    \geq& \E\Mp{r_1(s_1,\pi^*_1(s_1))+\oV_2^k(s_2)}-\overline{\Delta}_1^k(s_1,\pi^*_1(s_1))-Q_1^*(s_1,\pi^*_1(s_1))\tag{Lemma \ref{lemma:mdp Q error}}\\
    =& \E_{\pi^*}\Mp{\oV_2^k(s_2)-V_2^*(s_2)}-\overline{\Delta}_1^k(s_1,\pi^*_1(s_1))\\
    \geq& -\E_{\pi^*}\Mp{\sum_{h=1}^H\overline{\Delta}_h^k(s_h,a_h)}\\
    \geq& -\nu.\tag{Lemma \ref{lemma:misspecification}} 
\end{align*}
\end{proof}

\begin{lemma}\label{lemma:mdp pessimism}
(Pessimism) Under the good event $\G$, for all $k\in[K]$, we have
$$\uV_{1}^{k}(s_1)\leq V_{1}^{\pi^k}(s_1)-\sum_{h=1}^H\E_{\pi^k}\Mp{\underline{\Delta}_{h}^{k}(s_h,a_{h})}\leq V_{1}^{\pi^k}(s_1)+\nu. $$
\end{lemma}

\begin{proof}
Under the good event $\G$, for all $k\in[K]$, we have
\begin{align*}
    &\uV_{1}^{k}(s_1)-V_{1}^{\pi^k}(s_1)\\
    =&\uQ_1^k(s_1,\pi_1^k(s_1))-V_{1}^{\pi^k}(s_1)\\
    \leq& \E_{a_1=\pi_1^k(s_1)}\Mp{r_1(s_1,a_1)+\uV_2^k(s_2)-\underline{\Delta}_1^k(s_1,a_1)}-V_{1}^{\pi^k}(s_1)\tag{Lemma \ref{lemma:mdp Q error}}\\
    =&\E_{a_1=\pi_1^k(s_1)}\Mp{\uV_2^k(s_2)-V_{2}^{\pi^k}(s_2)-\underline{\Delta}_1^k(s_1,a_1)}\\
    \leq&-\E_{\pi^k}\Mp{\sum_{h=1}^H\underline{\Delta}_h^k(s_h,a_h)}\\
    \leq&\nu.\tag{Lemma \ref{lemma:misspecification}} 
\end{align*}
\end{proof}

\begin{lemma}\label{lemma:mdp gap}
Under the good event $\G$, for all $k\in[K]$, we have
$$V^{*}_{1}(s_1)-V^{\pi^k}_{1}(s_1)-2\nu\leq \oV_{1}^k(s_1)-\uV_{1}^k(s_1)\leq6\beta\E_{\pi^{k}}\sum_{h=1}^H\Norm{\phi(s_h,a_{h})}_{[\Sigma_{h}^k]^{-1}}+2\nu. $$
\end{lemma}

\begin{proof}
The proof is the same as the proof for Lemma \ref{lemma:CCE gap}. 
\end{proof}

\begin{lemma}\label{lemma:mdp sum of uncertainty}
Under the good event $\G$, we have
$$\sum_{k=1}^Kn^k\E_{\pi^{k}}\Norm{\phi(s_h,a_{h})}^2_{\Mp{\Sigma_{h}^k}^{-1}}\leq 4T_\Trig d\log\Sp{1+\frac{N}{d\lambda}}.$$
\end{lemma}

\begin{proof}
The proof is the same as the proof for Lemma \ref{lemma:sum of uncertainty}. 
\end{proof}

\begin{lemma}\label{lemma:mdp regret}
Under the good event $\G$, we have
$$\sum_{k=1}^Kn^k\Sp{\oV_{1}^k(s_1)-\uV_{1}^k(s_1)}\leq6H\beta\sqrt{4N(T_\Trig+1)d\log\Sp{1+\frac{N}{\lambda}}}+2\nu N. $$
\end{lemma}

\begin{proof}
The proof is the same as the proof for Lemma \ref{lemma:CCE regret}. 
\end{proof}

\begin{lemma}\label{lemma:mdp episode bound}
Under the good event $\G$, we have
$$K\leq \frac{2Hd\log(N+\lambda)}{\log(1+T_\Trig/4)},$$
which means $K<K_{\max}$ and Algorithm \ref{algo:mdp} ends due to $n^{\mathrm{tot}}=N_{\max}$. 
\end{lemma}

\begin{proof}
The proof is the same as the proof for Lemma \ref{lemma:episode bound}. 
\end{proof}

\begin{theorem}\label{thm:mdp}
For $\nu$-misspecified linear MDP, with probability at least $1-\delta$, Algorithm \ref{algo:mdp} will output an $(\epsilon+4\nu)$-approximate optimal policy. The sample complexity is $O(HK_{\max}N)=\widetilde{O}(H^6d^4\epsilon^{-2})$. 
\end{theorem}

\begin{proof}
Under the good event $\G$, by Lemma \ref{lemma:mdp episode bound}, the algorithm ends by $n^{\tot}=N$. By Lemma \ref{lemma:mdp regret}, we have
\begin{align*}
    \min_{k\in[K]}\Sp{\oV_1^k(s_1)-\uV_1^k(s_1)}
    \leq \frac{1}{N}\sum_{k=1}^Kn^k\Sp{\oV_1^k(s_1)-\uV_1^k(s_1)}
    \leq 6H\beta\sqrt{4(T_\Trig+1)d\log\Sp{1+\frac{N}{\lambda}}/N}+2\nu. 
\end{align*}
By setting $N=\widetilde{O}(H^4d^3\epsilon^{-2})$, we have
$$\min_{k\in[K]}\oV_1^k(s_1)-\uV_1^k(s_1)\leq\epsilon+2\nu. $$
Then by Lemma \ref{lemma:mdp gap}, we have
\begin{align*}
    V^*_1(s_1)-V^{\pi^{\mathrm{output}}}_1(s_1)\leq& \oV_{1}^{k^{\mathrm{output}}}(s_1)-\uV_{1}^{k^{\mathrm{output}}}(s_1)+2\nu
    =\min_{k\in[K]}\Sp{\oV_1^k(s_1)-\uV_1^k(s_1)}+2\nu
    \leq\epsilon+4\nu. 
\end{align*}
\end{proof}

\section{Proofs for Learning in Markov Potential Games}\label{apx:mpg}

\subsection{Proofs for Learning Markov NE with Algorithm \ref{algo:mpg}}

We will set the parameter for Algorithm \ref{algo:mpg} to be  
\begin{itemize}
    \item $K=5mH\epsilon^{-1}$
\end{itemize}

\begin{lemma}\label{lemma:mpg concentration 1}
With probability at least  $1-\delta/2$, for all $k\in[K]$   and $i\in[m]$, $\widehat{\pi}_i^{k+1}$ is an $(\epsilon/8+O(\nu))$-approximate optimal policy in the $\nu$-misspecified linear MDP induced by all the players except player $i$ following policy $\pi_{-i}^k$. 
\end{lemma}

\begin{proof}
The argument follows from the property of $\textsc{LinearMDP\_Solver}$ (Assumption \ref{asp:linear MDP}) and a union bound. 
\end{proof}

\begin{lemma}\label{lemma:mpg concentration 2}
Suppose for all $k\in[K]$ and $i\in[m]$, we execute policy $\pi^k$ and $(\widehat{\pi}_i^{k+1},\pi_{-i}^k)$ for $\widetilde{O}(H^2\epsilon^{-2})$ episodes, With probability at least $1-\delta/2$, for all $k\in[K]$ and $i\in[m]$, we have
$$\abs{\widehat{V}_{1,i}^{\pi^k}(s_1)-V_{1,i}^{\pi^k}(s_1)}\leq\frac{\epsilon}{8},\qquad\qquad\abs{\widehat{V}_{1,i}^{\widehat{\pi}_i^{k+1},\pi_{-i}^t}(s_1)-V_{1,i}^{\widehat{\pi}_i^{k+1},\pi_{-i}^k}(s_1)}\leq\frac{\epsilon}{8}. $$
\end{lemma}

\begin{proof}
The argument follows directly by Hoeffding's inequality and a union bound. 
\end{proof}

We will denote the event in Lemma \ref{lemma:mpg concentration 1} and Lemma \ref{lemma:mpg concentration 2} to be the good event $\G$.

\begin{lemma}\label{lemma:mpg trigger}
Under the good event $\G$, for any $k\in[K]$, if $\max_{i\in[m]}\Delta^k_i>\epsilon/2$ and $j=\argmax_{i\in[m]}\Delta^k_i$, we have
$$V_{1,j}^{\pi^{k+1}}(s_1)-V_{1,j}^{\pi^k}(s_1)\geq \epsilon/4. $$
And if $\max_{i\in[m]}\Delta^k_i\leq\epsilon/2$, we have
$$\max_{i\in[m]}\Sp{V_{1,i}^{\dagger,\pi^k_{-i}}(s_1)-V_{1,i}^{\pi^k}(s_1)}\leq \epsilon. $$
\end{lemma}

\begin{proof}
Under the good event $\G$, if $\max_{i\in[m]}\Delta^k_i>\epsilon/2$ and $j=\argmax_{i\in[m]}\Delta^k_i$, we have
\begin{align*}
    V_{1,j}^{\pi^{k+1}}(s_1)-V_{1,j}^{\pi^k}(s_1)\geq& \widehat{V}_{1,j}^{\widehat{\pi}_j^{k+1},\pi_{-i}^k}(s_1)-\epsilon/8-\widehat{V}_{1,j}^{\pi^k}(s_1)-\epsilon/8\tag{Lemma \ref{lemma:mpg concentration 2}}\\
    \geq&\epsilon/4. 
\end{align*}
On the other hand, if $\max_{i\in[m]}\Delta^k_i=\max_{i\in[m]}\Sp{\widehat{V}_{1,i}^{\widehat{\pi}_i^{k+1},\pi_{-i}^k}(s_1)-\widehat{V}_{1,i}^{\pi^k}(s_1)}\leq\epsilon/2$, for all $i\in[m]$ we have
\begin{align*}
    V_{1,i}^{\dagger,\pi^k_{-i}}(s_1)-V_{1,i}^{\pi^k}(s_1)\leq&V_{1,i}^{\widehat{\pi}_i^{k+1},\pi_{-i}^t}(s_1)+\frac{\epsilon}{8}+O(\nu)-V_{1,i}^{\pi^k}(s_1)\\
    \leq& \widehat{V}_{1,i}^{\widehat{\pi}_i^{k+1},\pi_{-i}^k}(s_1)+\frac{\epsilon}{8}+O(\nu)+\epsilon/8-\widehat{V}_{1,i}^{\pi^k}(s_1)+\epsilon/8\tag{Lemma \ref{lemma:mpg concentration 1} and Lemma \ref{lemma:mpg concentration 2}}\\
    \leq&\epsilon+O(\nu),
\end{align*}
completing the proof. 
\end{proof}

\mpg*

\begin{proof}
Suppose Algorithm \ref{algo:mpg} does not output a policy, then it ends due to $k=K$. Then under the good event $\G$, by the first argument of Lemma \ref{lemma:mpg trigger}, for all $k\in[K]$, and $j^k=\argmax_{i\in[m]}\Delta^k_i$, we have
$$\Phi(\pi^{k+1})-\Phi(\pi^k)=V_{1,j^t}^{\pi^{k+1}}(s_1)-V_{1,j^k}^{\pi^k}(s_1)\geq\epsilon/4. $$
As we set $K=5mH/\epsilon$, we have $\Phi(\pi^{K+1})>mH\geq\Phi_{\max}$, which is a contradiction. So Algorithm \ref{algo:mpg} will output a policy $\pi^{\mathrm{output}}$. As the $\textsc{LinearMDP\_Solver}$ always outputs a deterministic policy, $\pi^{\mathrm{output}}$ is a deterministic policy. Then by the second argument of Lemma \ref{lemma:mpg trigger}, when Algorithm \ref{algo:mpg} terminates, it will output an $\epsilon$-approximate pure NE $\pi^{\mathrm{output}}$. 
\end{proof}

\section{Proofs for Section \ref{sec:tabular}}\label{apx:tabular}

We will set the parameters for Algorithm \ref{algo:tabular} to be
\begin{itemize}
    \item $T_{\Trig}=12\log(8K_{\max}HS/\delta)$
    \item $K_{\max}=9HS\log(N_{\max})$
    \item $N_{\max}=\widetilde{O}(H^4SA_{\max}\epsilon^{-2})$ for Markov CCE and $N_{\max}=\widetilde{O}(H^4SA_{\max}^2\epsilon^{-2})$ for Markov CE
    \item $\beta_{n}=\sqrt{\frac{8H^2T_\Trig\log(2mK_{\max}HS/\delta)}{n\vee T_\Trig}}$.
\end{itemize}
We will use subscript $k,t$ to denote the variables in episode $k$ and inner loop $t$, and subscript $h,i$ to denote the variables at step $h$ and for player $i$. We will use $K$ to denote the episode that the Algorithm \ref{algo:tabular}  ends (Line \ref{line:certificate} is triggered or $n^\mathrm{tot}=N_{\max}$ or $K=K_{\max}$) and $N$ to denote $n^\mathrm{tot}$ when Algorithm \ref{algo:tabular}  ends. Immediately we have $K\leq K_{\max}\leq N_{\max}$. 

By the definition of the adversarial multi-armed bandit oracles (Assumption \ref{asp:no regret bandit} and Assumption \ref{asp:no swap regret bandit}), we have the following two lemmas. 

\begin{lemma}\label{lemma:no regret bandit}
For all $k\in[K]$, $h\in[H]$, $i\in[m]$ and $s\in\S$ we have
\begin{align*}
    &\frac{1}{n_h^k(s)}\sum_{j=1}^{n_h^k(s)}\E_{\a\sim\pi_h^{k,t_h^k(j;s)}(\cdot\mid s)}(r_{h,i}(s,\a)+\oV_{h+1,i}^k(s'))\\\geq& \max_{a_i\in\A_i}\frac{1}{n_h^k(s)}\sum_{j=1}^{n_h^k(s)}\E_{\a_{-i}\sim\pi_{h,-i}^{k,t_h^k(j;s)}(\cdot\mid s)}(r_{h,i}(s,\a)+\oV_{h+1,i}^k(s'))-\frac{n_h^k(s)}{H}\cdot\mathrm{BReg}(n_h^k(s)).
\end{align*}
\end{lemma}

\begin{lemma}\label{lemma:no swap regret bandit}
For all $k\in[K]$, $h\in[H]$, $i\in[m]$ and $s\in\S$ we have
\begin{align*}
    &\frac{1}{n_h^k(s)}\sum_{j=1}^{n_h^k(s)}\E_{\a\sim\pi_h^{k,t_h^k(j;s)}(\cdot\mid s)}(r_{h,i}(s,\a)+\oV_{h+1,i}^k(s'))\\
    \geq& \max_{\psi_{h,i}}\frac{1}{n_h^k(s)}\sum_{j=1}^{n_h^k(s)}\E_{\a\sim\psi_{h,i}\diamond\pi_{h}^{k,t_h^k(j;s)}(\cdot\mid s)}(r_{h,i}(s,\a)+\oV_{h+1,i}^k(s'))-\frac{n_h^k(s)}{H}\cdot\mathrm{BSwapReg}(n_h^k(s)).
\end{align*}
\end{lemma}

\subsection{Concentration}

\begin{lemma}\label{lemma:tabular concentration 1}
With probability at least $1-\delta/2$, for all $k\in[K]$, $h\in[H]$, $i\in[m],s\in\S$, we have
$$\abs{\frac{1}{n_h^k(s)}\sum_{j=1}^{n_h^k(s)}(r_{h,i}^{k,t_h^k(j;s)}+\oV_{h+1,i}^k(s_{h+1}^{k,t_h^k(j;s)}))-\frac{1}{n_h^k(s)}\sum_{j=1}^{n_h^k(s)}\E_{\a\sim\pi_h^{k,t_h^k(j;s)}(\cdot\mid s)}(r_{h,i}(s,\a)+\oV_{h+1,i}^k(s'))}\leq \beta_{n_h^k(s)},$$
$$\abs{\frac{1}{n_h^k(s)}\sum_{j=1}^{n_h^k(s)}(r_{h,i}^{k,t_h^k(j;s)}+\uV_{h+1,i}^k(s_{h+1}^{k,t_h^k(j;s)}))-\frac{1}{n_h^k(s)}\sum_{j=1}^{n_h^k(s)}\E_{\a\sim\pi_h^{k,t_h^k(j;s)}(\cdot\mid s)}(r_{h,i}(s,\a)+\uV_{h+1,i}^k(s'))}\leq \beta_{n_h^k(s)},$$
where 
$$\beta_{n_h^k(s)}=\sqrt{\frac{8H^2T_\Trig\log(2mK_{\max}HS/\delta)}{n_h^k(s)\vee T_\Trig}}.$$
\end{lemma}

\begin{proof}
If $n_h^k(s)\leq T_\Trig$, we have $\beta_{n_h^k(s)}\geq H$ and the arguments hold directly. If $n_h^k(s)\geq T_\Trig$, we have
$$\beta_{n_h^k(s)}=\sqrt{\frac{8H^2T_\Trig\log(2mK_{\max}HS/\delta)}{n_h^k(s)\vee T_\Trig}}\geq\sqrt{\frac{8H^2\log(2mK_{\max}HS/\delta)}{n_h^k(s)}},$$
and by Hoeffding's inequality and union bound, we can prove that the arguments hold with probability at least $1-\delta/2$. 
\end{proof}

\begin{lemma}\label{lemma:tabular concentration 2}
With probability at least $1-\delta/2$, for all $k\in[K_{\max}]$, $h\in[H]$, $i\in[m],s\in\S$, we have
$$n_h^k(s)\vee T_{\Trig}\geq \frac{1}{2}\Sp{\sum_{l=1}^{k-1}n^ld_h^{\pi^l}(s)}\vee T_{\Trig},n^kd_h^{\pi^k}(s)\leq2\Sp{n_h^k(s)\vee T_\Trig}. $$
In addition, if $T_h^k(s)=n_h^k(s)\vee T_\Trig$ is triggered, we have
$$n^kd_h^{\pi^k}(s)\geq\frac{1}{2}\Sp{n_h^k(s)\vee T_\Trig}.  $$
\end{lemma}

\begin{proof}
$n_h^k(s)$ is the sum of $\sum_{l=1}^{k-1}n^l$ independent Bernoulli random variables such that there are $n^l$ random variables with mean $d_h^{\pi^l}(s)$ for $l\in[k-1]$. By Lemma \ref{lemma:clipped bernoulli} and union bound, with probability at least $1-\delta/4$, for all $k\in[K_{\max}]$, $h\in[H]$, $s\in\S$, we have
$$n_h^k(s)\vee T_{\Trig}\geq \frac{1}{2}\Sp{\sum_{l=1}^{k-1}n^ld_h^{\pi^l}(s)}\vee T_{\Trig},$$
where $T_\Trig\geq12\log(8K_{\max}HS/\delta)$. 

$T_h^k(s)$ is the sum of $n^k$ i.i.d. Bernoulli random variables with mean $n_h^k$. For the second argument, by Lemma \ref{lemma:trig concentration} and union bound, with probability at least $1-\delta/4$, for all $k\in[K_{\max}]$, $h\in[H]$, $s\in\S$, we have
$$n^kd_h^{\pi^k}(s)\leq2(n_h^k(s)\vee T_\Trig),$$
and if $T_h^k(s)=n_h^k(s)\vee T_\Trig$ is triggered, we have
$$n^kd_h^{\pi^k}(s)\geq\frac{1}{2}T_h^k(s)=\frac{1}{2}\Sp{n_h^k(s)\vee T_\Trig}.  $$
\end{proof}

We denote $\G$ to be the good event where the arguments in Lemma \ref{lemma:tabular concentration 1} and Lemma \ref{lemma:tabular concentration 2} hold, which holds with probability at least $1-\delta$.

\subsection{Proofs for Learning Markov CCE with Algorithm \ref{algo:tabular}}

\begin{lemma}\label{lemma:optimism tabular}
Under the good event $\G$, for all $k\in[K]$, $h\in[H]$, $i\in[m]$, $s\in\S$, we have
$$\oV_{h,i}^k(s)\geq V_{h,i}^{\dagger,\pi_{-i}^k}(s).$$
\end{lemma}

\begin{proof}
Under the good event $\G$, for all $k\in[K]$, $h\in[H]$, $i\in[m]$, $s\in\S$, we have
\begin{align*}
    \oV_{h,i}^k(s)=&\proj_{[0,H+1-h]}\Sp{\frac{1}{n_h^k(s)}\sum_{j=1}^{n_h^k(s)}(r_{h,i}^{k,t_h^k(j;s)}+\oV_{h+1,i}^k(s_{h+1}^{k,t_h^k(j;s)}))+\frac{H}{n_h^k(s)}\cdot\mathrm{BReg}(n_h^k(s))+\beta_{n_h^k(s)}}\\
    \geq&\proj_{[0,H+1-h]}\Sp{\frac{1}{n_h^k(s)}\sum_{j=1}^{n_h^k(s)}\E_{\a\sim\pi_h^{k,t_h^k(j;s)}(\cdot\mid s)}(r_{h,i}(s,\a)+\oV_{h+1,i}^k(s'))+\frac{H}{n_h^k(s)}\cdot\mathrm{BReg}(n_h^k(s))}\tag{Lemma \ref{lemma:tabular concentration 1}}\\
    \geq&\proj_{[0,H+1-h]}\Sp{\max_{a_i\in\A_i}\frac{1}{n_h^k(s)}\sum_{j=1}^{n_h^k(s)}\E_{\a_{-i}\sim\pi_{h,-i}^{k,t_h^k(j;s)}(\cdot\mid s)}(r_{h,i}(s,\a)+\oV_{h+1,i}^k(s'))}\tag{Lemma \ref{lemma:no regret bandit}}\\
    \geq&\proj_{[0,H+1-h]}\Sp{\max_{a_i\in\A_i}\frac{1}{n_h^k(s)}\sum_{j=1}^{n_h^k(s)}\E_{\a_{-i}\sim\pi_{h,-i}^{k,t_h^k(j;s)}(\cdot\mid s)}(r_{h,i}(s,\a)+V_{h+1,i}^{\dagger,\pi_{-i}^k}(s'))}\tag{Induction basis}\\
    =&\proj_{[0,H+1-h]}\Sp{\max_{a_i\in\A_i}\E_{\a_{-i}\sim\pi_{h,-i}^{k}(\cdot\mid s)}(r_{h,i}(s,\a)+V_{h+1,i}^{\dagger,\pi_{-i}^k}(s'))}\\
    \geq&V_{h,i}^{\dagger,\pi_{-i}^k}(s). 
\end{align*}
\end{proof}

\begin{lemma}\label{lemma:pessimism tabular}
Under the good event $\G$, for all $k\in[K]$, $h\in[H]$, $i\in[m]$, $s\in\S$, we have
$$\uV_{h,i}^k(s)\leq V_{h,i}^{\pi^k}(s).$$
\end{lemma}

\begin{proof}
Under the good event $\G$, for all $k\in[K]$, $h\in[H]$, $i\in[m]$, $s\in\S$, we have
\begin{align*}
    \uV_{h,i}^k(s)=&\proj_{[0,H+1-h]}\Sp{\frac{1}{n_h^k(s)}\sum_{j=1}^{n_h^k(s)}(r_{h,i}^{k,t_h^k(j;s)}+\uV_{h+1,i}^k(s_{h+1}^{k,t_h^k(j;s)}))-\beta_{n_h^k(s)}}\\
    \leq&\proj_{[0,H+1-h]}\Sp{\frac{1}{n_h^k(s)}\sum_{j=1}^{n_h^k(s)}\E_{\a\sim\pi_h^{k,t_h^k(j;s)}(\cdot\mid s)}(r_{h,i}(s,\a)+\uV_{h+1,i}^k(s'))}\tag{Lemma \ref{lemma:tabular concentration 1}}\\
    \leq&\proj_{[0,H+1-h]}\Sp{\frac{1}{n_h^k(s)}\sum_{j=1}^{n_h^k(s)}\E_{\a\sim\pi_h^{k,t_h^k(j;s)}(\cdot\mid s)}(r_{h,i}(s,\a)+V_{h+1,i}^{\pi^k}(s'))}\tag{Induction basis}\\
    =&\proj_{[0,H+1-h]}\Sp{\E_{\a\sim\pi_{h,-i}^{k}(\cdot\mid s)}(r_{h,i}(s,\a)+V_{h+1,i}^{\dagger,\pi_{-i}^k}(s'))}\\
    \leq&V_{h,i}^{\pi^k}(s). 
\end{align*}
\end{proof}

\begin{lemma}\label{lemma:tabular CCE gap}
Under the good event $\G$, for all $k\in[K]$, $i\in[m]$, we have
$$\oV_{1,i}^k(s_1)-\uV_{1,i}^k(s_1)\leq\widetilde{O}\Sp{\E_{\pi^{k}}\Mp{\sum_{h=1}^H\sqrt{\frac{H^2A_iT_\Trig}{n_h^k(s_h)\vee T_\Trig}}}}. $$
\end{lemma}

\begin{proof}
We bound $\oV_{1,i}^k(s_1)-V_{1,i}^{\pi^k}(s_1)$ and $V_{1,i}^{\pi^k}(s_1)-\uV_{1,i}^k(s_1)$ separately. 
\begin{align*}
    &\oV_{1,i}^k(s_1)-V_{1,i}^{\pi^k}(s_1)\\
    =&\proj_{[0,H+1-h]}\Sp{\frac{1}{n_1^k(s_1)}\sum_{j=1}^{n_1^k(s_1)}(r_{1,i}^{k,t_h^k(j;s_1)}+\oV_{2,i}^k(s_{2}^{k,t_h^k(j;s_1)}))+\frac{H}{n_1^k(s_1)}\cdot\mathrm{BReg}(n_1^k(s_1))+\beta_{n_1^k(s_1)}}-V_{1,i}^{\pi^k}(s_1)\\
    \leq& \frac{1}{n_1^k(s_1)}\sum_{t=1}^{n_1^k(s)}(r_{1,i}^{k,t_h^k(j;s_1)}+\oV_{2,i}^k(s_{2}^{k,t_h^k(j;s_1)}))+\frac{H}{n_1^k(s_1)}\cdot\mathrm{BReg}(n_1^k(s_1))+\beta_{n_1^k(s_1)}-V_{1,i}^{\pi^k}(s_1)\\
    \leq&\frac{1}{n_1^k(s_1)}\sum_{t=1}^{n_1^k(s_1)}\E_{\a_1\sim\pi_1^{k,t_h^k(j;s_1)}(\cdot\mid s)}(r_{1,i}(s_1,\a_1)+\oV_{2,i}^k(s_2))+\frac{H}{n_1^k(s_1)}\cdot\mathrm{BReg}(n_1^k(s_1))+2\beta_{n_1^k(s_1)}-V_{1,i}^{\pi^k}(s_1)\tag{Lemma \ref{lemma:tabular concentration 1}}\\
    =&\E_{\a_1\sim\pi_1^{k}(\cdot\mid s)}(r_{1,i}(s_1,\a_1)+\oV_{2,i}^k(s_2))+\frac{H}{n_1^k(s_1)}\cdot\mathrm{BReg}(n_1^k(s_1))+2\beta_{n_1^k(s_1)}-V_{1,i}^{\pi^k}(s_1)\\
    =&\E_{\pi_1^{k}}\Mp{\oV_{2,i}^k(s_2)-V_{2,i}^{\pi^k}(s_2)}+\frac{H}{n_1^k(s_1)}\cdot\mathrm{BReg}(n_1^k(s_1))+2\beta_{n_1^k(s_1)}\\
    =&\E_{\pi^{k}}\Mp{\sum_{h=1}^H\frac{H}{n_h^k(s_h)}\cdot\mathrm{BReg}(n_h^k(s_h))+2\beta_{n_h^k(s_h)}},
    %\leq&\E_{\pi^{k}}\Mp{\sum_{h=1}^H\sqrt{\frac{128H^2A_iT_\Trig\log(mA_iKHS/\delta)}{n_h^k(s_h)\vee T_\Trig}}}
\end{align*}
where the first inequality is from
$$\frac{1}{n_1^k(s_1)}\sum_{t=1}^{n_1^k(s)}(r_{1,i}^{k,t_h^k(j;s_1)}+\oV_{2,i}^k(s_{2}^{k,t_h^k(j;s_1)}))+\frac{H}{T}\cdot\mathrm{BReg}(n_1^k(s_1))+\beta_{n_1^k(s_1)}\geq0. $$
In addition, we have
\begin{align*}
    &V_{1,i}^{\pi^k}(s_1)-\uV_{1,i}^k(s_1)\\
    =&V_{1,i}^{\pi^k}(s_1)-\proj_{[0,H+1-h]}\Sp{\frac{1}{n_1^k(s)}\sum_{j=1}^{n_1^k(s_1)}(r_{1,i}^{k,t_1^k(j;s)}+\uV_{2,i}^k(s_{2}^{k,t_1^k(j;s)}))-\beta_{n_1^k(s_1)}}\\
    \leq& V_{1,i}^{\pi^k}(s_1)-\frac{1}{n_1^k(s_1)}\sum_{j=1}^{n_1^k(s_1)}(r_{1,i}^{k,t_1^k(j;s)}+\uV_{2,i}^k(s_{2}^{k,t_1^k(j;s)}))+\beta_{n_1^k(s_1)}\\
    \leq& V_{1,i}^{\pi^k}(s_1)-\frac{1}{n_1^k(s_1)}\sum_{j=1}^{n_1^k(s_1)}\Sp{\E_{\a_1\sim\pi_1^{k,t_h^k(j;s_1)(\cdot\mid s_1)}}(r_{1,i}(s_1,\a_1)+\uV_{2,i}^k(s_2))}+2\beta_{n_1^k(s_1)}\tag{Lemma \ref{lemma:tabular concentration 1}}\\
    =& V_{1,i}^{\pi^k}(s_1)-\E_{\a_1\sim\pi_1^{k}(\cdot\mid s_1)}(r_{1,i}(s_1,\a_1)+\uV_{2,i}^k(s_2))+2\beta_{n_1^k(s_1)}\\
    =& \E_{\a_1\sim\pi_1^{k}}(V_{2,i}^{\pi^k}(s_2)-\uV_{2,i}^k(s_2))+2\beta_{n_1^k(s_1)}\\
    \leq& \E_{\pi^{k}}\Mp{\sum_{h=1}^H2\beta_{n_h^k(s_h)}},\\
\end{align*}
where the first inequality is from 
$$\frac{1}{n_1^k(s)}\sum_{j=1}^{n_1^k(s_1)}(r_{1,i}^{k,t_1^k(j;s)}+\uV_{2,i}^k(s_{2}^{k,t_1^k(j;s)}))-\beta_{n_1^k(s_1)}\leq H+1-h. $$
Then we have
\begin{align*}
    \oV_{1,i}^{\pi^k}(s_1)-\uV_{1,i}^k(s_1)\leq&\E_{\pi^{k}}\Mp{\sum_{h=1}^H\frac{H}{n_h^k(s_h)}\cdot\mathrm{BReg}(n_h^k(s_h))+4\beta_{n_h^k(s_h)}}\\
    \leq&\widetilde{O}\Sp{\E_{\pi^{k}}\Mp{\sum_{h=1}^H\sqrt{\frac{H^2A_i}{n_h^k(s_h)\vee1}}+\sqrt{\frac{H^2T_\Trig}{n_h^k(s_h)\vee T_\Trig}}}}\\
    \leq&\widetilde{O}\Sp{\E_{\pi^{k}}\Mp{\sum_{h=1}^H\sqrt{\frac{H^2A_iT_\Trig}{n_h^k(s_h)\vee T_\Trig}}}}. \\
\end{align*}
\end{proof}

\begin{lemma}\label{lemma:tabular CCE regret}
Under the good event $\G$, for all $i\in[m]$, we have
$$\sum_{k=1}^Kn^k\max_{i\in[m]}\Sp{\oV_{1,i}^k(s_1)-\uV_{1,i}^{\pi^k}(s_1)}\leq \widetilde{O}\Sp{H^2\sqrt{SA_{\max}T_\Trig N}}. $$
\end{lemma}

\begin{proof}
Under the good event $\G$, for all $i\in[m]$, we have
\begin{align*}
    &\sum_{k=1}^Kn^k\E_{\pi^k}\sqrt{\frac{1}{n_h^k(s_h)\vee T_\Trig}}\\
    =&\sum_{k=1}^Kn^k\sum_{s\in\S}d^{\pi^k}_h(s)\sqrt{\frac{1}{n_h^k(s)\vee T_\Trig}}\\
    \leq&\sum_{s\in\S}\sum_{k=1}^Kn^kd^{\pi^k}_h(s)\sqrt{\frac{2}{(\sum_{l=1}^{k-1}n^ld_h^{\pi^l}(s))\vee T_\Trig}}\tag{Lemma \ref{lemma:tabular concentration 2}}\\
    \leq&\sum_{s\in\S}\sqrt{32\sum_{k=1}^Kn^kd^{\pi^k}_h(s)}\tag{Lemma \ref{lemma:tabular concentration 2} and Lemma \ref{lemma:sum of uncertainty tabular}}\\
    \leq&\sqrt{32S\sum_{k=1}^Kn^k}\tag{$\sum_{s\in\S}\sum_{k=1}^Kn^kd^{\pi^k}_h(s)=S\sum_{k=1}^Kn^k$}. 
\end{align*}
Plugging it into Lemma \ref{lemma:tabular CCE gap},   we can prove the lemma. 
\end{proof}

\begin{lemma}\label{lemma:episode bound tabular}
Under the good event $\G$, we have
$$K\leq 9HS\log(N_{\max}),$$
which means $K<K_{\max}$ and Algorithm \ref{algo:tabular} ends due to either   Line \ref{line:terminate tabular} ($\max_{i\in[m]}\oV_{1,i}^k(s_1)-\uV_{1,i}^k(s_1)\leq\epsilon$) or Line \ref{line:trigger tabular} ($n^{\mathrm{tot}}=N_{\max}$). 
\end{lemma}

\begin{proof}
By Lemma \ref{lemma:tabular concentration 2}, for any $h\in[H]$ and $s\in\S$, whenever $T_h^k(s)=n_h^k(s)\vee T_\Trig$ is triggered, we have
$$n^kd_h^{\pi^k}(s)\geq\frac{1}{2}(n^k_h(s)\vee T_\Trig)\geq\frac{1}{4}\Sp{\sum_{l=1}^{k-1}n^ld_h^{\pi^l}(s)}. $$

\begin{comment}
On the other hand, by the multiplicative Chernoff bound, we have
$$\P\Mp{n_h^k(s)<\frac{1}{2}\Sp{\sum_{l=1}^{k-1}n^ld_h^{\pi^l}(s)}}\leq\exp\Sp{-\Sp{\sum_{l=1}^{k-1}n^ld_h^{\pi^l}(s)}/8}. $$
If $\sum_{l=1}^{k-1}n^ld_h^{\pi^l}(s)\geq 8\log(1-\delta)$, we have
$$\P\Mp{n_h^k(s)<\frac{1}{2}\Sp{\sum_{l=1}^{k-1}n^ld_h^{\pi^l}(s)}}\leq\delta.$$
If $\sum_{l=1}^{k-1}n^ld_h^{\pi^l}(s)< 8\log(1-\delta)$, we have
$$n_h^k(s)\vee T_\Trig\geq\frac{1}{2}\Sp{\sum_{l=1}^{k-1}n^ld_h^{\pi^l}(s)}. $$
Thus we have
$$\P\Mp{n_h^k(s)\vee T_\Trig<\frac{1}{2}\Sp{\sum_{l=1}^{k-1}n^ld_h^{\pi^l}(s)}}\leq\delta.$$
\end{comment}

Thus for any $h\in[H]$ and $s\in\S$, whenever $T_h^k(s)=n_h^k(s)\vee T_\Trig$ is triggered, we have
$$\sum_{l=1}^{k}n^ld_h^{\pi^l}(s)\geq\frac{5}{4}\Sp{\sum_{l=1}^{k-1}n^ld_h^{\pi^l}(s)}. $$
In addition, for any $h\in[H]$ and $s\in\S$, for the first time $T_h^k(s)=n_h^k(s)\vee T_\Trig$ is triggered, we have
$$\sum_{l=1}^{k}n^ld_h^{\pi^l}(s)\geq n^kd_h^{\pi^k}(s)\geq\frac{1}{2}(n^k_h(s)\vee T_\Trig)\geq T_{\Trig}. $$
As $\sum_{l=1}^{k}n^ld_h^{\pi^l}(s)$ is non-decreasing and upper bounded by $N_{\max}$, the number of triggering for any $h\in[H]$ and $s\in\S$ is bounded by $\log(N_{\max}/T_\Trig)/\log(5/4)\leq 8\log(N_{\max})$, and the total number of triggering is bounded by $8HS\log(N_{\max})+1$, where $1$ is from the last triggering $n^\tot=N_{\max}$. 
\end{proof}

\tabularCCE*

\begin{proof}
Suppose under the good event $\G$, the algorithm does not end with Line \ref{line:terminate tabular} ($\max_{i\in[m]}\oV_{1,i}^k(s_1)-\uV_{1,i}^k(s_1)\leq\epsilon$). Then by Lemma \ref{lemma:episode bound tabular}, the algorithm ends by $N=N_{\max}$. By Lemma \ref{lemma:tabular CCE regret}, under the good event $\G$, we have
\begin{align*}
    \min_{k\in[K]}\max_{i\in[m]}\Sp{\oV_{1,i}^k(s_1)-\uV_{1,i}^k(s_1)}\leq&\frac{1}{N}\sum_{k=1}^Kn^k\max_{i\in[m]}\Sp{\oV_{1,i}^k(s_1)-\uV_{1,i}^k(s_1)}\\
    \leq&\widetilde{O}\Sp{H^2\sqrt{SA_{\max}T_\Trig /N_{\max}}}. 
\end{align*}
Let $N_{\max}=\widetilde{O}(H^4SA_{\max}\epsilon^{-2})$ we can have
$$\min_{k\in[K]}\max_{i\in[m]}\Sp{\oV_{1,i}^k(s_1)-\uV_{1,i}^k(s_1)}\leq\epsilon,$$
which contradicts with Line \ref{line:terminate tabular}. 
Thus Algorithm \ref{algo:tabular} will end at episode $k$ such that 
$$\max_{i\in[m]}\Sp{\oV_{1,i}^k(s_1)-\uV_{1,i}^k(s_1)}\leq\epsilon. $$
By Lemma \ref{lemma:optimism tabular} and Lemma \ref{lemma:pessimism tabular}, we have
$$\max_{i\in[m]}\Sp{V_{1,i}^{\dagger,\pi^k_{-i}}(s_1)-V_{1,i}^{\pi^k}(s_1)}\leq \max_{i\in[m]}\Sp{\oV_{1,i}^k(s_1)-\uV_{1,i}^k(s_1)}\leq\epsilon, $$
completing the proof. 
\end{proof}

\subsection{Proofs for Learning Markov CE with Algorithm \ref{algo:tabular}}

\begin{lemma}\label{lemma:optimism CE tabular}
Under the good event $\G$, for all $k\in[K]$, $h\in[H]$, $i\in[m]$, $s\in\S$, we have
$$\oV_{h,i}^k(s)\geq \max_{\psi_i}V_{h,i}^{\psi_i\diamond\pi^k}(s).$$
\end{lemma}

\begin{proof}
We prove the lemma by mathematical induction on $h$. The argument holds for $h=H+1$ as both sides are 0. Suppose the argument holds for $h+1$. By the update rule of $\oV_{h,i}^k(s)$, we have 
\begin{align*}
    \oV_{h,i}^k(s)=&\proj_{[0,H+1-h]}\Sp{\frac{1}{n_h^k(s)}\sum_{j=1}^{n_h^k(s)}(r_{h,i}^{k,t_h^k(j;s)}+\oV_{h+1,i}^k(s_{h+1}^{k,t_h^k(j;s)}))+\frac{H}{n_h^k(s)}\mathrm{BSwapReg}(n_h^k(s))+\beta_{n_h^k(s)}}\\
    \geq&\proj_{[0,H+1-h]}\Sp{\frac{1}{n_h^k(s)}\sum_{j=1}^{n_h^k(s)}\E_{\a\sim\pi_h^{k,t_h^k(j;s)}(\cdot\mid s)}(r_{h,i}(s,\a)+\oV_{h+1,i}^k(s'))+\frac{H}{n_h^k(s)}\mathrm{BSwapReg}(n_h^k(s))}\tag{Lemma \ref{lemma:tabular concentration 1}}\\
    \geq&\proj_{[0,H+1-h]}\Sp{\max_{\psi_{h,i}}\frac{1}{n_h^k(s)}\sum_{j=1}^{n_h^k(s)}\E_{\a\sim\psi_{h,i}\diamond\pi_{h}^{k,t_h^k(j;s)}(\cdot\mid s)}(r_{h,i}(s,\a)+\oV_{h+1,i}^k(s'))}\tag{Lemma \ref{lemma:no swap regret bandit}}\\
    =&\proj_{[0,H+1-h]}\Sp{\max_{\psi_{h,i}}\E_{\a\sim\psi_{h,i}\diamond\pi_{h}^{k}(\cdot\mid s)}(r_{h,i}(s,\a)+\oV_{h+1,i}^k(s'))}\\
    \geq&\proj_{[0,H+1-h]}\Sp{\max_{\psi_{h,i}}\E_{\a\sim\psi_{h,i}\diamond\pi_{h}^{k}(\cdot\mid s)}(r_{h,i}(s,\a)+\max_{\psi_i}V_{h+1,i}^{\psi_i\diamond\pi^k}(s'))}\tag{Induction basis}\\
    \geq&\max_{\psi_i}V_{h,i}^{\psi_i\diamond\pi^k}(s). 
\end{align*} 
\end{proof}

\begin{lemma}\label{lemma:tabular CE gap}
Under the good event $\G$, for all $k\in[K]$, $i\in[m]$, we have
$$\oV_{1,i}^k(s_1)-V_{1,i}^{\pi^k}(s_1)\leq\widetilde{O}\Sp{\E_{\pi^{k}}\Mp{\sum_{h=1}^H\sqrt{\frac{H^2A_i^2T_\Trig}{n_h^k(s_h)\vee T_\Trig}}}}. $$
\end{lemma}

\begin{proof}
The proof is the same as the proof of Lemma \ref{lemma:tabular CCE gap} and we replace $\mathrm{BReg}$ with $\mathrm{BSwapReg}$. 
\end{proof}

\begin{lemma}\label{lemma:tabular CE regret}
Under the good event $\G$, for all $k\in[K]$, $h\in[H]$, $i\in[m]$, $s\in\S$, we have
$$\sum_{k=1}^Kn^k\max_{i\in[m]}\Sp{\oV_{1,i}^k(s_1)-\uV_{1,i}^{\pi^k}(s_1)}\leq \widetilde{O}\Sp{H^2\sqrt{SA_{\max}^2T_\Trig N}}. $$
\end{lemma}

\begin{proof}
The proof is the same as the proof of Lemma \ref{lemma:tabular CCE regret} and we replace Lemma \ref{lemma:tabular CCE gap} with Lemma \ref{lemma:tabular CE gap} in the proof.
\end{proof}

\tabularCE*

\begin{proof}
Suppose under the good event $\G$, the algorithm does not end with Line \ref{line:terminate tabular} ($\max_{i\in[m]}\oV_{1,i}^k(s_1)-\uV_{1,i}^k(s_1)\leq\epsilon$). Then by Lemma \ref{lemma:episode bound tabular}, the algorithm ends by $N=N_{\max}$. By Lemma \ref{lemma:tabular CE regret}, under the good event $\G$, we have
\begin{align*}
    \min_{k\in[K]}\max_{i\in[m]}\Sp{\oV_{1,i}^k(s_1)-\uV_{1,i}^k(s_1)}\leq&\frac{1}{N}\sum_{k=1}^Kn^k\max_{i\in[m]}\Sp{\oV_{1,i}^k(s_1)-\uV_{1,i}^k(s_1)}\\
    \leq&\widetilde{O}\Sp{H^2\sqrt{SA_{\max}^2T_\Trig /N_{\max}}}. 
\end{align*}
Let $N_{\max}=\widetilde{O}(H^4SA_{\max}^2\epsilon^{-2})$ we can have
$$\min_{k\in[K]}\max_{i\in[m]}\Sp{\oV_{1,i}^k(s_1)-\uV_{1,i}^k(s_1)}\leq\epsilon,$$
which contradicts with Line \ref{line:terminate tabular}. 
Thus Algorithm \ref{algo:tabular} will end at episode $k$ such that 
$$\max_{i\in[m]}\Sp{\oV_{1,i}^k(s_1)-\uV_{1,i}^k(s_1)}\leq\epsilon. $$
By Lemma \ref{lemma:optimism CE tabular} and Lemma \ref{lemma:pessimism tabular}, we have 
$$\max_{i\in[m]}\Sp{\max_{\psi_i}V_{1,i}^{\psi_i\diamond\pi^k}(s_1)-V_{1,i}^{\pi^k}(s_1)}\leq \max_{i\in[m]}\Sp{\oV_{1,i}^k(s_1)-\uV_{1,i}^k(s_1)}\leq\epsilon. $$
\end{proof}

\section{Technical Tools}

\begin{lemma}\label{lemma:empirical bernstein}
(Theorem 4 in \cite{maurer2009empirical})
For $n\geq2$, let $X_1,\cdots,X_n$ be i.i.d. random variables with values in $[0,1]$ and let $\delta>0$. Define $\widehat{X}=\frac{1}{n}\sum_{i=1}^nX_i$ and $\widehat{\sigma}=\frac{1}{n-1}\sum_{i=1}^n(X_i-\widehat{X})$. Then we have
$$\P\Mp{\abs{\widehat{X}-\E[X]}>\sqrt{\frac{2\widehat{\sigma}\log(4/\delta)}{n}}+\frac{7\log(4/\delta)}{3(n-1)}}\leq\delta. $$
\end{lemma}

\begin{lemma}\label{lemma:trig concentration}
Consider i.i.d. random variables $X_1,X_2,\dots$ with support in $[0,1]$ and $\widehat{S}_n=\frac{1}{n}\sum_{i=1}^nX_i$. Suppose $\overline{n}=\min_{n}\{n:\sum_{i=1}^nX_i\geq T_\Trig\}$ with $T_\Trig\geq64\log(4n_{\max}/\delta)$. Then if $\overline{n}\leq n_{\max}$, with probability at least $1-\delta$, we have
$$\frac{1}{2}\widehat{S}_{\overline{n}}\leq\E[X]\leq\frac{3}{2}\widehat{S}_{\overline{n}},$$
and in addition, for $n\leq\min\{\overline{n},n_{\max}\}$, we have
$$\E[X]\leq\frac{2T_\Trig}{n}. $$
\end{lemma}

\begin{proof}
Define the empirical variance to be 
$$\widehat{\sigma}_n=\frac{1}{n-1}\sum_{i=1}^n(X_i-\widehat{S}_n)^2. $$
By Lemma \ref{lemma:empirical bernstein}, we have that for any fixed $n\geq2$, 
% with probability at least $1-\frac{\delta}{n_{\max}}$, 
$$\P\Mp{\abs{\widehat{S}_n-\E[X]}\leq\sqrt{\frac{2\log(4n_{\max}/\delta)\widehat{\sigma}_n}{n}}+\frac{7\log(4n_{\max}/\delta)}{3(n-1)}}\geq 1-\frac{\delta}{n_{\max}}. $$
Thus we have
\begin{equation}\label{eq:empirical bernstein}
    \P\Mp{\abs{\widehat{S}_n-\E[X]}\leq\sqrt{\frac{2\log(4n_{\max}/\delta)\widehat{\sigma}_n}{n}}+\frac{7\log(4n_{\max}/\delta)}{3(n-1)},\forall 2\leq n\leq n_{\max}}\geq 1-\sum_{n=2}^{n_{\max}}\frac{\delta}{n_{\max}}\geq1-\delta.
\end{equation}
The empirical variance can be bounded by
$$\widehat{\sigma}_n=\frac{1}{n-1}\sum_{i=1}^n(X_i-\widehat{S}_n)^2=\frac{1}{n-1}\Sp{\sum_{i=1}^nX_i^2-n\widehat{X}^2}\leq\frac{1}{n-1}\sum_{i=1}^nX_i\leq 2\widehat{S}_n. $$
Thus for $T_\Trig\geq64\log(4n_{\max}/\delta)$, we have $\overline{n}\widehat{S}_{\overline{n}}\geq T_\Trig\geq 64\log(4n_{\max}/\delta)$ and
$$\sqrt{\frac{2\log(4\overline{n}/\delta)\widehat{\sigma}_{\overline{n}}}{\overline{n}}}+\frac{7\log(4\overline{n}/\delta)}{3(\overline{n}-1)}\leq\sqrt{\frac{4\log(4\overline{n}/\delta)\widehat{S}_{\overline{n}}}{\overline{n}}}+\frac{7\log(4\overline{n}/\delta)}{3(\overline{n}-1)}\leq\frac{\widehat{S}_{\overline{n}}}{2}. $$
Plugging it into  \eqref{eq:empirical bernstein},   we can prove the first argument.

For $n\leq\min\{\overline{n},N\}$, we have $\sum_{i=1}^n X_i\leq T_\Trig+1\leq 2T_\Trig$, which means
$$\widehat{\sigma}_n\leq 2\widehat{S}_n\leq\frac{4T_\Trig}{n}. $$
Plugging it into  \eqref{eq:empirical bernstein},  and with $T_\Trig\geq64\log(4n_{\max}/\delta)$, we can prove the second argument.
\end{proof}

\begin{lemma}\label{lemma:clipped bernoulli}
Suppose $X_1,X_2,\cdots,X_n$ are i.i.d. Bernoulli random variables with $\E[X]=p$ and $N=\sum_{i=1}^nX_i$. For any $a\geq 12\log(2/\delta)$ with probability at least $1-\delta$, we have
$$\frac{1}{2}\Sp{N\vee a}\leq np\vee a\leq 2\Sp{N\vee a}. $$
\end{lemma}

\begin{proof}
By the multiplicative Chernoff bound, we have
$$\P\Mp{\abs{N-np}\geq \frac{1}{2}np}\leq 2\exp\Sp{-\frac{np}{12}}. $$
Thus if $np\geq 12\log(2/\delta)$, we have
$$\P\Mp{\frac{1}{2}np\leq N\leq 2np}\leq \delta. $$
If $np<12\log(2/\delta)$, by Bernstein inequality, with probability $1-\delta$ we have
$$\P\Mp{N-np>t}\leq\exp\Sp{-\frac{t^2/2}{np+t/3}}. $$
Let $t=a\geq np$ and we have
$$\P\Mp{N>2a}\leq\exp\Sp{-\frac{a^2/2}{np+a/3}}\leq\exp(-3a/8)\leq\delta. $$
Note that if $N\leq 2a$, we directly have $$\frac{1}{2}\Sp{N\vee a}\leq np\vee a\leq 2\Sp{N\vee a}. $$
\end{proof}

\begin{lemma}\label{lemma:covering ball}
(Lemma 20.1 in \cite{lattimore2020bandit}) The Euclidean sphere $S^{d-1}=\{x\in\R^d:\|x\|_2=1\}$. There exists a set $\mathcal{C}_\epsilon\subset\R^d$ with $|\mathcal{C_\epsilon}|\leq(3/\epsilon)^d$ such that for all $x\in S^{d-1}$ there exists $y\in C_\epsilon$ with $\Norm{x-y}_2\leq\epsilon$. 
 
\end{lemma}

\begin{lemma}\label{lemma:logdet}
Let $\Sigma\succeq \lambda I$ be a positive definite matrix and $M$ be a positive semidefinite matrix with eigenvalue upper-bounded by $1$. Let $\Sigma'=\Sigma+M$. Then we have
$$\log\det(\Sigma')\geq\log\det(\Sigma)+\mathrm{Tr}(\Sigma^{-1}M).$$
\end{lemma}

\begin{proof}
\begin{align*}
    \det(\Sigma')=&\det(\Sigma+M)\\
    =&\det(\Sigma)\det(I+\Sigma^{-1/2}M\Sigma^{-1/2}). 
\end{align*}
Denote $\lambda_1,\dots,\lambda_d$ as the eigenvalues of $\Sigma^{-1/2}M\Sigma^{-1/2}$. Then we have
$$x^\top\Sigma^{-1/2}M\Sigma^{-1/2}x\leq\Norm{\Sigma^{-1/2}x}_2^2=x^\top\Sigma^{-1}x\leq\lambda^{-1},$$
which means $\lambda_i\in[0,\lambda^{-1}]$ for all $i\in[d]$. Thus, we have 
$$\log\det(\Sigma')=\log\det(\Sigma)+\sum_{i=1}^d\log(1+\lambda_i)\geq\log\det(\Sigma)+\sum_{i=1}^d\frac{\lambda}{\lambda+1}\lambda_i=\log\det(\Sigma)+\frac{\lambda}{\lambda+1}\mathrm{Tr}(\Sigma^{-1}M),$$
completing the proof. 
\end{proof}

\begin{lemma}\label{lemma:information gain}(Lemma 11 in \cite{zanette2022stabilizing})
For any random vector $\phi\in\R^d$, scalar $\alpha>0$ and positive definite matrix $\Sigma$, we have
$$\frac{\alpha}{L}\E\Norm{\phi}_{\Sigma^{-1}}^2\leq\log\frac{\det(\Sigma+\alpha\E[\phi\phi^\top])}{\det(\Sigma)}\leq\alpha\E\Norm{\phi}_{\Sigma^{-1}}^2,$$
whenever $\alpha\E\Norm{\phi}_{\Sigma^{-1}}^2\leq L$ for some $L\geq e-1$. 
\end{lemma}

\begin{lemma}\label{lemma:sum of uncertainty tabular}
Let $b>0$ and $a_1,a_2,\cdots,a_n>0$ such that $a_{n+1}\leq c\cdot \Sp{\sum_{l=1}^{n-1}a_l\vee b}$ for all $n\geq 1$ and some constant $c$. Then we have
$$\sum_{i=1}^\infty a_i\sqrt{\frac{1}{(\sum_{l=1}^{i-1}a_l)\vee b}}\leq 2\sqrt{(c+1)\sum_{l=1}^{n}a_l}.$$
\end{lemma}

\begin{proof}
Note that for any $i\geq1$ we have
$$\sqrt{\frac{1}{(\sum_{l=1}^{i-1}a_l)\vee b}}\leq \sqrt{\frac{c+1}{(\sum_{l=1}^{i}a_l)\vee b}}. $$
Let $f(x)=\sqrt{\frac{c+1}{x\vee b}}$ for $x\geq0$ and immediately we have $f(x)$ is non-increasing. Then we have
\begin{align*}
    \sum_{i=1}^n a_i\sqrt{\frac{1}{(\sum_{l=1}^{i-1}a_l)\vee b}}\leq& \sum_{i=1}^\infty a_i\sqrt{\frac{c+1}{(\sum_{l=1}^{i}a_l)\vee b}}\\
    =&\sum_{i=1}^n a_if(\sum_{l=1}^{i}a_l)\\
    \leq&\int_0^{\sum_{l=1}^{n}a_l} f(x)\\
    \leq& 2\sqrt{(c+1)\sum_{l=1}^{n}a_l}. 
\end{align*}
\end{proof}

\begin{lemma}\label{lemma:constrained LS}
(Lemma 4 in \cite{zanette2022stabilizing})
Let $X\in\R^d$ be a random vector and $Y$ be a random variable such that $\Norm{X}_2\leq 1$, $|Y|\leq Y_{\max}$, $(X,Y)\sim\P$ for some distribution $\P$. Let $\{(x_i,y_i)\}_{i=1}^n$ be $n$ i.i.d. samples from $\P$. Then we define
$$\beta^*:=\argmin_{\Norm{\beta}_2\leq W}\E_{(X,Y)\sim\P}(Y-\inner{X,\beta})^2,$$
$$\widehat{\beta}:=\argmin_{\Norm{\beta}_2\leq W}\frac{1}{n}\sum_{i=1}^n(y_i-\inner{x_i,\beta})^2. $$
Then with probability at least $1-\delta$, we have
$$\Norm{\beta^*-\widehat{\beta}}_{n\E[XX^\top]+\lambda I}\leq 8(W+Y_{\max})\sqrt{d\log(32Wn(W+Y_{\max}))+\log(1/\delta)+\lambda}.$$
\end{lemma} 

\begin{lemma}\label{lemma:cov concentration}
(Covariance Concentration) (Proposition 1 in \cite{zanette2022stabilizing})
Suppose $\{Z_k\}_{k=1^K}$ is a sequence of independent, symmetric and positive definite random matrices of dimension $d$ such that
$$0\leq\lambda_{\min}(Z_k)\leq\lambda_{\max}(Z_k)\leq 1,\forall k\in[K]. $$
Let $\widehat{\Sigma}=\lambda I+\sum_{k=1}^KZ_k$ and $\Sigma=\E[\widehat{\Sigma}]$ for some $\lambda\geq0$. 
For any $\delta\in(0,1)$ and $\lambda>2\frac{\log(2d/\delta)}{\log(36/35)}$, with probability at least $1-\delta$ we have
$$\frac{1}{2}\widehat{\Sigma}\preceq \Sigma\preceq \frac{3}{2}\widehat{\Sigma}. $$
\end{lemma}

\end{document}